\documentclass[10pt]{IEEEtran}
\usepackage{amsmath,amssymb}
\usepackage{amsthm,bm,framed,url}
\usepackage[ruled,lined,linesnumbered]{algorithm2e}
\usepackage{graphicx}
\usepackage[usenames,dvipsnames]{color}
\usepackage{xcolor}
\usepackage{psfrag}
\usepackage{epstopdf}
\usepackage{wrapfig}
\usepackage{caption,subcaption}
\usepackage{cite}
\usepackage{wrapfig}
\usepackage{stmaryrd}
\definecolor{myred}{rgb}{1,0.2,0}
\definecolor{mygreen}{rgb}{0,1,0}
\definecolor{myblue}{rgb}{0,0.2,1}

\newtheorem{lemma}{Lemma}

\newtheorem{property}{Property}
\newtheorem{remark}{Remark}
\newtheorem{proposition}{Proposition}
\newtheorem{theorem}{Theorem}
\newtheorem{definition}{Definition}

\newcommand{\Xt}{{\bf X}}
\newcommand{\Gt}{{\bf G}}
\newcommand{\Yt}{{\bf Y}}
\newcommand{\Nt}{{\bf N}}

\newcommand{\Wt}{{\bf W}}
\newcommand{\Rt}{{\bf R}}
\newcommand{\Pt}{{\bf P}}
\newcommand{\Gammat}{{\bm \Gamma}}

\def\Real{{\mathbb R}}
\def\Complex{{\mathbb C}}
\def\Natural{{\mathbb N}}

\def\bPi{\mbox{\boldmath $\Pi$}}
\def\bLambda{\mbox{\boldmath $\Lambda$}}

\def\bSigma{\mbox{\boldmath $\Sigma$}}
\def\bGamma{\mbox{\boldmath $\Gamma$}}
\def\bgamma{\mbox{\boldmath $\gamma$}}

\newcommand{\diag}[1]{\text{Diag}\left\{#1\right\}}
\newcommand{\trace}[1]{\text{Tr}\left\{#1\right\}}
\newcommand{\Null}[1]{\text{null}\left\{#1\right\}}
\newcommand{\range}[1]{\text{range}\left\{#1\right\}}

\newcommand{\inv}[1]{\left(#1\right)^{-1}}
\newcommand{\invtwo}[1]{\left(#1\right)^{-2}}

\DeclareMathOperator*{\CD}{\ast}
\newcommand{\hada}[1]{\CD_{\substack{j=1\\j \neq #1}}^N}

\newcommand{\Y}{{\bf Y}}
\newcommand{\W}{{\bf W}}
\renewcommand{\H}{{\bf H}}
\newcommand{\A}{{\bf A}}
\newcommand{\B}{{\bf B}}
\newcommand{\C}{{\bf C}}
\newcommand{\D}{{\bf D}}
\newcommand{\I}{{\bf I}}
\newcommand{\G}{{\bf G}}
\newcommand{\F}{{\bf F}}

\newcommand{\Q}{{\bf Q}}

\renewcommand{\L}{{\bf L}}
\newcommand{\M}{{\bf M}}
\newcommand{\N}{{\bf N}}
\newcommand{\T}{{\bf T}}
\newcommand{\K}{{\bf K}}
\newcommand{\Sb}{{\bf S}}

\newcommand{\eye}[1]{{\bf I}_{#1}}
\newcommand{\one}{{\bf 1}}
\newcommand{\zero}{{\bf 0}}

\newcommand{\y}{{\bf y}}
\newcommand{\z}{{\bf z}}
\newcommand{\h}{{\bf h}}
\newcommand{\w}{{\bf w}}

\newcommand{\e}{{\bf e}}

\newcommand{\crb}{Cram{\'e}r-Rao}

\renewcommand{\th}{\bm{\theta}}
\newcommand{\ph}{\bm{\varphi}}
\newcommand{\fim}{\bm{\Phi}}
\newcommand{\ufim}{\bm{\Psi}}
\newcommand{\fullfim}{\bm{\Omega}}
\newcommand{\E}[1]{\text{E}\left\{#1\right\}}
\newcommand{\cov}[1]{\text{cov}\{#1\}}
\renewcommand{\vec}[1]{\text{vec}\left(#1\right)}

\newcommand{\score}{\nabla_{\th} \log p(\y;\th)}

\newcommand{\jacob}{\mathcal{D}}
\usepackage{textcomp}
\newcommand{\ktensor}[1]{\text{\textlbrackdbl}#1\text{\textrbrackdbl}}

\begin{document}

\title{Tensor Decomposition for Signal Processing and Machine Learning}

\author{Nicholas D. Sidiropoulos,~\IEEEmembership{Fellow,~IEEE}, Lieven De Lathauwer,~\IEEEmembership{Fellow,~IEEE}, Xiao Fu,~\IEEEmembership{Member,~IEEE}, Kejun Huang,~\IEEEmembership{Student~Member,~IEEE}, Evangelos E. Papalexakis, and Christos Faloutsos
\thanks{N.D. Sidiropoulos, X. Fu, and K. Huang are with the ECE Department, University of Minnesota, Minneapolis, USA; e-mail: {\tt (nikos,xfu,huang663)@umn.edu}. Supported in part by NSF IIS-1247632, IIS-1447788.}
\thanks{Lieven De Lathauwer is with KU Leuven, Belgium; e-mail: {\tt Lieven.DeLathauwer@kuleuven.be}. Supported by (1) KU Leuven Research Council: CoE EF/05/006 Optimization in Engineering (OPTEC), C1 project C16/15/059-nD; (2) F.W.O.: project G.0830.14N, G.0881.14N; (3) Belgian Federal Science Policy Office: IUAP P7 (DYSCO II, Dynamical systems, control and optimization, 2012–2017); (4) EU: The research leading to these results has received funding from the European Research Council under the European Union's Seventh Framework Programme (FP7/2007-2013) / ERC Advanced Grant: BIOTENSORS (no. 339804). This paper reflects only the authors' views and the EU is not liable for any use that may be made of the contained information.}
\thanks{E.E. Papalexakis and C. Faloutsos are with the CS Department, Carnegie Mellon University, USA; e-mail {\tt (epapalex,christos)@cs.cmu.edu}. Supported in part by NSF IIS-1247489.
}
}

\maketitle

\begin{abstract}
Tensors or {\em multi-way arrays} are functions of three or more indices $(i,j,k,\cdots)$ -- similar to matrices (two-way arrays), which are functions of two indices $(r,c)$ for (row,column). Tensors have a rich history, stretching over almost a century, and touching upon numerous disciplines; but they have only recently become ubiquitous in signal and data analytics at the confluence of signal processing, statistics, data mining and machine learning. This overview article aims to provide a good starting point for researchers and practitioners interested in learning about and working with tensors. As such, it focuses on fundamentals and motivation (using various application examples), aiming to strike an appropriate balance of breadth {\em and depth} that will enable someone having taken first graduate courses in matrix algebra and probability to get started doing research and/or developing tensor algorithms and software. Some background in applied optimization is useful but not strictly required. The material covered includes tensor rank and rank decomposition; basic tensor factorization models and their relationships and properties (including fairly good coverage of identifiability); broad coverage of algorithms ranging from alternating optimization to stochastic gradient; statistical performance analysis; and applications ranging from source separation to collaborative filtering, mixture and topic modeling, classification, and multilinear subspace learning.
\end{abstract}
\begin{IEEEkeywords}
Tensor decomposition, tensor factorization, rank, canonical polyadic decomposition (CPD), parallel factor analysis (PARAFAC), Tucker model, higher-order singular value decomposition (HOSVD), multilinear singular value decomposition (MLSVD), uniqueness, NP-hard problems, alternating optimization, alternating direction method of multipliers, gradient descent, Gauss-Newton, stochastic gradient, Cram\'er-Rao bound, communications, source separation, harmonic retrieval, speech separation, collaborative filtering, mixture modeling, topic modeling, classification, subspace learning.
\end{IEEEkeywords}

\section{Introduction}



Tensors\footnote{The term has different meaning in Physics, however it has been widely adopted across various disciplines in recent years to refer to what was previously known as a {\em multi-way array}.} (of order higher than two) are arrays indexed by three or more indices, say $(i,j,k,\cdots)$ -- a generalization of matrices, which are indexed by two indices, say $(r,c)$ for (row, column). Matrices are two-way arrays, and there are three- and higher-way arrays (or {\em higher-order}) tensors.

Tensor algebra has many similarities but also many striking differences with matrix algebra -- e.g., low-rank tensor factorization is essentially unique under mild conditions; determining tensor rank is NP-hard, on the other hand, and the best low-rank approximation of a higher rank tensor may not even exist. Despite such apparent paradoxes and the 
learning curve needed to digest tensor algebra notation and data manipulation, tensors have already found many applications in signal processing (speech, audio, communications, radar, biomedical), machine learning (clustering, dimensionality reduction, latent factor models, subspace learning), and well beyond. Psychometrics (loosely defined as mathematical methods for the analysis of personality data) and later Chemometrics (likewise, for chemical data) have historically been two important application areas driving theoretical and algorithmic developments. Signal processing followed, in the 90's, but the real spark that popularized tensors came when the computer science community (notably those in machine learning, data mining, computing) discovered the power of tensor decompositions, roughly a decade ago \cite{1565685,Acar2005,4538221}. There are nowadays many hundreds, perhaps thousands of papers published each year on tensor-related topics. Signal processing applications include, e.g., unsupervised separation of unknown mixtures of speech signals \cite{NioMokSidPot08} and code-division communication signals without knowledge of their codes \cite{SidGiaBro00}; and emitter localization for radar, passive sensing, and communication applications \cite{NionSidRadar,SidBroGia00}. There are many more applications of tensor techniques that are not immediately recognized as such, e.g., the analytical constant modulus algorithm \cite{ACMA,kofidis2001tensor}. Machine learning applications include face recognition, mining musical scores, and detecting cliques in social networks -- see \cite{Vasilescu2002,psychovgaltis,PapFalSid2012} and references therein. More recently, there has been considerable work on tensor decompositions for learning latent variable models, particularly topic models \cite{Anandkumar:2014:TDL:2627435.2697055}, and connections between orthogonal tensor decomposition and the method of moments for computing the Latent Dirichlet Allocation (LDA -- a widely used topic model).

After two decades of research on tensor decompositions and applications, the senior co-authors still couldn't point their new graduate students to a single ``point of entry'' to begin research in this area. This article has been designed to address this need: to provide a fairly comprehensive {\em and} deep overview of tensor decompositions that will enable someone having taken first graduate courses in matrix algebra and probability to get started doing research and/or developing related algorithms and software. While no single reference fits this bill, there are several very worthy tutorials and overviews that offer different points of view in certain aspects, and we would like to acknowledge them here. Among them, the highly-cited and clearly-written tutorial \cite{Kolda09tensordecompositions} that appeared 7 years ago in {\it SIAM Review} is perhaps the one closest to this article. It covers the basic models and algorithms (as of that time) well, but it does not go deep into uniqueness, advanced algorithmic, or estimation-theoretic aspects. The target audience of \cite{Kolda09tensordecompositions} is applied mathematics (SIAM). The recent tutorial \cite{psychovgaltis} offers an accessible introduction, with many figures that help ease the reader into three-way thinking. It covers most of the bases and includes many motivating  applications, but it also covers a lot more beyond the basics and thus stays at a high level. The reader gets a good roadmap of the area, without delving into it enough to prepare for research. Another recent tutorial on tensors is \cite{ComonTensorsBriefIntro}, which adopts a more abstract point of view of tensors as mappings from a linear space to another, whose coordinates transform multilinearly under a change of bases. This article is more suited for people interested in tensors as a mathematical concept, rather than how to use tensors in science and engineering. It includes a nice review of tensor rank results and a brief account of uniqueness aspects, but nothing in the way of algorithms or tensor computations. An overview of tensor techniques for large-scale numerical computations is given in \cite{grasedyck2013literature, hackbusch2012tensor}, geared towards a scientific computing audience; see \cite{NV-OD-LS-LDL} for a more accessible introduction. A gentle introduction to tensor decompositions can be found in the highly cited Chemometrics tutorial \cite{Bro1997a} -- a bit outdated but still useful for its clarity -- and the more recent book \cite{SmiBroGel}. Finally, \cite{PapFalSidKol:TIST2016} is an upcoming tutorial with emphasis on scalability and data fusion applications -- it does not go deep into tensor rank, identifiability, decomposition under constraints, or statistical performance benchmarking.

None of the above offers a comprehensive overview that is sufficiently deep to allow one to appreciate the underlying mathematics, the rapidly expanding and diversifying toolbox of tensor decomposition algorithms, and the basic ways in which tensor decompositions are used in signal processing and machine learning -- and they are quite different. Our aim in this paper is to give the reader a tour that goes `under the hood' on the technical side, and, at the same time, serve as a bridge between the two areas.  Whereas we cannot include detailed proofs of some of the deepest results, we do provide insightful derivations of simpler results and {\em sketch} the line of argument behind more general ones. For example, we include a one-page self-contained proof of Kruskal's condition when one factor matrix is full column rank, which illuminates the role of Kruskal-rank in proving uniqueness. We also `translate' between the signal processing (SP) and machine learning (ML) points of view. In the context of the canonical polyadic decomposition (CPD), also known as parallel factor analysis (PARAFAC), SP researchers (and Chemists) typically focus on the columns of the factor matrices ${\bf A}$, ${\bf B}$, ${\bf C}$ and the associated rank-1 factors ${\bf a}_f \circledcirc  {\bf b}_f \circledcirc  {\bf c}_f$ of the decomposition (where $\circledcirc$ denotes the outer product, see section \ref{section_prods}), because they are interested in {\em separation}. ML researchers often focus on the rows of ${\bf A}$, ${\bf B}$, ${\bf C}$, because they think of them as parsimonious latent space representations. For a user $\times$ item $\times$ context ratings tensor, for example, a row of ${\bf A}$ is a representation of the corresponding user in latent space, and likewise a row of ${\bf B}$ (${\bf C}$) is a representation of the corresponding item (context) in the same latent space. The inner product of these three vectors is used to predict that user's rating of the given item in the given context. This is one reason why ML researchers tend to use inner (instead of outer) product notation. SP researchers are interested in model identifiability because it guarantees separability; ML researchers are interested in identifiability to be able to interpret the dimensions of the latent space. In co-clustering applications, on the other hand, the rank-1 tensors ${\bf a}_f \circledcirc  {\bf b}_f \circledcirc  {\bf c}_f$ capture {\em latent concepts} that the analyst seeks to learn from the data (e.g., cliques of users buying certain types of items in certain contexts). SP researchers are trained to seek {\em optimal} solutions, which is conceivable for small to moderate data; they tend to use computationally heavier algorithms. ML researchers are nowadays trained to think about scalability from day one, and thus tend to choose much more lightweight algorithms to begin with. There are many differences, but also many similarities and opportunities for cross-fertilization. Being conversant in both communities allows us to bridge the ground between and help SP and ML researchers better understand each other.

\subsection{Roadmap} The rest of this article is structured as follows. We begin with some matrix preliminaries, including matrix rank and low-rank approximation, and a review of some useful matrix products and their properties. We then move to rank and rank decomposition for tensors. We briefly review bounds on tensor rank, multilinear (mode-) ranks, and relationship between tensor rank and multilinear rank. We also explain the notions of typical, generic, and border rank, and discuss why low-rank tensor approximation may not be well-posed in general. Tensors can be viewed as data or as multi-linear operators, and while we are mostly concerned with the former viewpoint in this article, we also give a few important examples of the latter as well. Next, we provide a fairly comprehensive account of uniqueness of low-rank tensor decomposition. This is the most advantageous difference when one goes from matrices to tensors, and therefore understanding uniqueness is important in order to make the most out of the tensor toolbox. Our exposition includes two stepping-stone proofs: one based on eigendecomposition, the other bearing Kruskal's mark (``down-converted to baseband'' in terms of difficulty). The Tucker model and multilinear SVD come next, along with a discussion of their properties and connections with rank decomposition. A thorough discussion of algorithmic aspects follows, including a detailed discussion of how different types of constraints can be handled, how to exploit data sparsity, scalability, how to handle missing values, and different loss functions. In addition to basic alternating optimization strategies, a host of other solutions are reviewed, including gradient descent, line search, Gauss-Newton, alternating direction method of multipliers, and stochastic gradient approaches. The next topic is statistical performance analysis, focusing on the widely-used Cram\'er-Rao bound and its efficient numerical computation. This section contains novel results and derivations that are of interest well beyond our present context -- e.g., can also be used to characterize estimation performance for a broad range of constrained matrix factorization problems. The final main section of the article presents motivating applications in signal processing (communication and speech signal separation, multidimensional harmonic retrieval) and machine learning (collaborative filtering, mixture and topic modeling, classification, and multilinear subspace learning). We conclude with some pointers to online resources (toolboxes, software, demos), conferences, and some historical notes.

\section{Preliminaries}

\subsection{Rank and rank decomposition for matrices}

Consider an $I \times J$ matrix ${\bf X}$, and let $\text{colrank}({\bf X})$ $:=$ the number of linearly independent columns of ${\bf X}$, i.e., the dimension of the range space of ${\bf X}$, $\text{dim}(\text{range}({\bf X}))$. $\text{colrank}({\bf X})$ is the minimum $k \in \Natural$ such that ${\bf X} = {\bf A} {\bf B}^T$, where ${\bf A}$ is an $I \times k$ basis of $\text{range}({\bf X})$, and ${\bf B}^T$ is $k \times J$ and holds the corresponding coefficients. This is because if we can generate all columns of ${\bf X}$, by linearity we can generate anything in $\text{range}({\bf X})$, and vice-versa. We can similarly define $\text{rowrank}({\bf X})$ $:=$ the number of linearly independent rows of ${\bf X}$ $=$ $\text{dim}(\text{range}({\bf X}^T))$, which is the minimum $\ell \in \Natural$ such that ${\bf X}^T = {\bf B} {\bf A}^T$ $\Longleftrightarrow$ ${\bf X} = {\bf A} {\bf B}^T$, where ${\bf B}$ is $J \times \ell$ and ${\bf A}^T$ is $\ell \times I$. Noting that
\[
{\bf X} = {\bf A} {\bf B}^T = {\bf A}(:,1) ({\bf B}(:,1))^T + \cdots + {\bf A}(:,\ell) ({\bf B}(:,\ell))^T,
\]
where ${\bf A}(:,\ell)$ stands for the $\ell$-th column of ${\bf A}$, we have
\[
{\bf X} = {\bf a}_1 {\bf b}_1^T + \cdots + {\bf a}_\ell {\bf b}_\ell^T,
\]
where ${\bf A} = \left[{\bf a}_1,\cdots,{\bf a}_{\ell}\right]$ and ${\bf B} = \left[{\bf b}_1,\cdots,{\bf b}_{\ell}\right]$.
It follows that
$\text{colrank}({\bf X}) = \text{rowrank}({\bf X}) = \text{rank}({\bf X})$,
and $\text{rank}({\bf X}) = \text{minimum}~m~\text{such that}~{\bf X} = \sum_{n=1}^m {\bf a}_n {\bf b}_n^T$,
so the three definitions actually coincide -- but only in the matrix (two-way tensor) case, as we will see later. Note that, per the definition above, ${\bf a} {\bf b}^T$ is a rank-1 matrix that is `simple' in the sense that every column (or row) is proportional to any other column (row, respectively). In this sense, rank can be thought of as a measure of complexity. Note also that $\text{rank}({\bf X}) \leq \min(I,J)$, because obviously ${\bf X} = {\bf X} {\bf I}$, where ${\bf I}$ is the identity matrix.

\subsection{Low-rank matrix approximation}

In practice ${\bf X}$ is usually full-rank, e.g., due to measurement noise, and we observe ${\bf X} = {\bf L} + {\bf N}$, where ${\bf L}={\bf A} {\bf B}^T$ is low-rank and ${\bf N}$ represents noise and `unmodeled dynamics'. If the elements of ${\bf N}$ are sampled from a jointly continuous distribution, then ${\bf N}$ will be full rank almost surely -- for the determinant of any square submatrix of ${\bf N}$ is a polynomial in the matrix entries, and a polynomial that is nonzero at one point is nonzero at every point except for a set of measure zero. In such cases, we are interested in approximating ${\bf X}$ with a low-rank matrix, i.e., in
\[
\min_{{\bf L} ~|~ \text{rank}({\bf L})=\ell} ||{\bf X} - {\bf L}||_F^2 \Longleftrightarrow
\min_{{\bf A} \in \Real^{I \times \ell},~ {\bf B} \in \Real^{J \times \ell}} ||{\bf X} - {\bf A}{\bf B}^T||_F^2.
\]
The solution is provided by the truncated SVD of ${\bf X}$, i.e., with ${\bf X} = {\bf U} \mathbf{\Sigma} {\bf V}^T$, set ${\bf A} = {\bf U}(:,1:\ell) \mathbf{\Sigma}(1:\ell,1:\ell)$, ${\bf B} = {\bf V}(:,1:\ell)$ or
${\bf L} = {\bf U}(:,1:\ell) \mathbf{\Sigma}(1:\ell,1:\ell) ({\bf V}(:,1:\ell))^T$, where ${\bf U}(:,1:\ell)$ denotes the matrix containing columns $1$ to $\ell$ of ${\bf U}$. However, this factorization is non-unique because ${\bf A} {\bf B}^T = {\bf A} {\bf M} {\bf M}^{-1} {\bf B}^T$ $=$ $({\bf A} {\bf M})({\bf B} {\bf M}^{-T})^T$, for any nonsingular $\ell \times \ell$ matrix ${\bf M}$, where ${\bf M}^{-T} = ({\bf M}^{-1})^{T}$. In other words: the factorization of the approximation is highly non-unique (when $\ell=1$, there is only scaling ambiguity, which is usually inconsequential). As a special case, when ${\bf X}={\bf L}$ (noise-free) so $\text{rank}({\bf X})=\ell$, low-rank decomposition of ${\bf X}$ is non-unique.

\subsection{Some useful products and their properties}
\label{section_prods}

In this section we review some useful matrix products and their properties, as they pertain to tensor computations.

\noindent {\em Kronecker product}: The Kronecker product of ${\bf A}$ ($I \times K$) and ${\bf B}$ ($J \times L$) is the $IJ \times KL$ matrix
\[
{\bf A} \otimes {\bf B} := \left[
                             \begin{array}{cccc}
                               {\bf B} {\bf A}(1,1) & {\bf B} {\bf A}(1,2) & \cdots & {\bf B} {\bf A}(1,K) \\
                               {\bf B} {\bf A}(2,1) & {\bf B} {\bf A}(2,2) & \cdots & {\bf B} {\bf A}(2,K) \\
                               \vdots & \vdots & \cdots & \vdots \\
                               {\bf B} {\bf A}(I,1) & {\bf B} {\bf A}(I,2) & \cdots & {\bf B} {\bf A}(I,K) \\
                             \end{array}
                           \right]
\]
The Kronecker product has many useful properties. From its definition, it follows that ${\bf b}^T \otimes {\bf a} = {\bf a} {\bf b}^T$. For an $I \times J$ matrix ${\bf X}$, define
\[
\text{vec}({\bf X}) :=  \left[
                             \begin{array}{c}
                               {\bf X}(:,1)\\
                               {\bf X}(:,2)\\
                               \vdots\\
                               {\bf X}(:,J)\\
                             \end{array}
                           \right],
\]
i.e., the $IJ \times 1$ vector obtained by vertically stacking the columns of ${\bf X}$. By definition of $\text{vec}(\cdot)$ it follows that $\text{vec}({\bf a} {\bf b}^T) = {\bf b} \otimes {\bf a}$. 

Consider the product ${\bf A} {\bf M} {\bf B}^T$, where ${\bf A}$ is $I \times K$, ${\bf M}$ is $K \times L$, and ${\bf B}$ is $J \times L$. Note that
\begin{eqnarray*}
{\bf A} {\bf M} {\bf B}^T &=& \left( \sum_{k=1}^K {\bf A}(:,k) {\bf M}(k,:) \right) {\bf B}^T\\
                          &=& \sum_{k=1}^K \sum_{\ell=1}^{L} {\bf A}(:,k) {\bf M}(k,\ell) ({\bf B}(:,\ell))^T.\\
\end{eqnarray*}
Therefore, using $\text{vec}({\bf a} {\bf b}^T) = {\bf b} \otimes {\bf a}$ and linearity of the $\text{vec}(\cdot)$ operator
\begin{eqnarray*}
\text{vec}\left({\bf A} {\bf M} {\bf B}^T\right) &=& \sum_{k=1}^K \sum_{\ell=1}^{L} {\bf M}(k,\ell) {\bf B}(:,\ell) \otimes {\bf A}(:,k)\\
&=& \left({\bf B} \otimes {\bf A}\right) \text{vec}({\bf M}).
\end{eqnarray*}

This is useful when dealing with linear least squares problems of the following form
\[
\min_{{\bf M}} ||{\bf X} - {\bf A} {\bf M} {\bf B}^T||_F^2
\Longleftrightarrow \min_{{\bf m}} ||\text{vec}({\bf X}) - ({\bf B} \otimes {\bf A}) {\bf m}||_2^2,
\]
where ${\bf m} := \text{vec}({\bf M})$.

\noindent {\em Khatri--Rao product:} Another useful product is the Khatri--Rao (column-wise Kronecker) product of two matrices {\em with the same number of columns} (see \cite[p. 14]{SmiBroGel} for a generalization). That is, with ${\bf A}=\left[{\bf a}_1,\cdots,{\bf a}_{\ell}\right]$ and ${\bf B}=\left[{\bf b}_1,\cdots,{\bf b}_{\ell}\right]$, the Khatri--Rao product of ${\bf A}$ and ${\bf B}$ is ${\bf A} \odot {\bf B}$ $:=$ $\left[{\bf a}_1 \otimes {\bf b}_1, \cdots {\bf a}_{\ell} \otimes {\bf b}_{\ell}\right]$. It is easy to see that, with ${\bf D}$ being a diagonal matrix with vector ${\bf d}$ on its diagonal (we will write ${\bf D} = \text{Diag}({\bf d})$, and ${\bf d} = \text{diag}({\bf D})$, where we have implicitly defined operators $\text{Diag}(\cdot)$ and $\text{diag}(\cdot)$ to convert one to the other), the following property holds
\[
\text{vec}\left({\bf A} {\bf D} {\bf B}^T\right) = \left({\bf B} \odot {\bf A}\right) {\bf d},
\]
which is useful when dealing with linear least squares problems of the following form
\[
\min_{{\bf D}=\text{Diag}({\bf d})} ||{\bf X} - {\bf A} {\bf D} {\bf B}^T||_F^2
\Longleftrightarrow \min_{{\bf d}} ||\text{vec}({\bf X}) - ({\bf B} \odot {\bf A}) {\bf d}||_2^2.
\]
It should now be clear that the Khatri--Rao product ${\bf B} \odot {\bf A}$ is a subset of columns from ${\bf B} \otimes {\bf A}$. Whereas ${\bf B} \otimes {\bf A}$ contains the `interaction' (Kronecker product) of {\em any} column of ${\bf A}$ with {\em any} column of ${\bf B}$, ${\bf B} \odot {\bf A}$ contains the Kronecker product of {\em any} column of ${\bf A}$ with {\em only the corresponding} column of ${\bf B}$.

\noindent {\em Additional properties:}
\begin{itemize}
\item $({\bf A} \otimes {\bf B}) \otimes {\bf C} = {\bf A} \otimes ({\bf B} \otimes {\bf C})$ (associative); so we may simply write as ${\bf A} \otimes {\bf B} \otimes {\bf C}$. Note though that ${\bf A} \otimes {\bf B} \neq {\bf B} \otimes {\bf A}$, so the Kronecker product is non-commutative.
\item $({\bf A} \otimes {\bf B})^T = {\bf A}^T \otimes {\bf B}^T$ (note order, unlike $({\bf A B})^T = {\bf B}^T {\bf A}^T$).
\item $({\bf A} \otimes {\bf B})^* = {\bf A}^* \otimes {\bf B}^*$ $\Longrightarrow$ $({\bf A} \otimes {\bf B})^H = {\bf A}^H \otimes {\bf B}^H$, where $^*$, $^H$ stand for conjugation and Hermitian (conjugate) transposition, respectively.
\item $({\bf A} \otimes {\bf B})({\bf E} \otimes {\bf F}) = ({\bf A E} \otimes {\bf B F})$ (the {\em mixed product rule}). This is very useful -- as a corollary, if ${\bf A}$ and ${\bf B}$ are square nonsingular, then it follows that $({\bf A} \otimes {\bf B})^{-1} = {\bf A}^{-1} \otimes {\bf B}^{-1}$, and likewise for the pseudo-inverse. More generally, if ${\bf A} = {\bf U}_1 \mathbf{\Sigma}_1 {\bf V}^T_1$ is the SVD of ${\bf A}$, and ${\bf B} = {\bf U}_2 \mathbf{\Sigma}_2 {\bf V}^T_2$ is the SVD of ${\bf B}$, then it follows from the mixed product rule that ${\bf A} \otimes {\bf B}$ $=$ $({\bf U}_1 \mathbf{\Sigma}_1 {\bf V}^T_1) \otimes ({\bf U}_2 \mathbf{\Sigma}_2 {\bf V}^T_2)$ $=$ $({\bf U}_1 \otimes {\bf U}_2)  (\mathbf{\Sigma}_1 \otimes \mathbf{\Sigma}_2) ({\bf V}_1 \otimes {\bf V}_2)^T_2$ is the SVD of ${\bf A} \otimes {\bf B}$. It follows that
\item $\text{rank}({\bf A} \otimes {\bf B})=\text{rank}({\bf A})\text{rank}({\bf B})$.
\item $\text{tr}({\bf A} \otimes {\bf B})=\text{tr}({\bf A})\text{tr}({\bf B})$, for square ${\bf A}$, ${\bf B}$.
\item $\text{det}({\bf A} \otimes {\bf B})=\text{det}({\bf A})\text{det}({\bf B})$, for square ${\bf A}$, ${\bf B}$.
\end{itemize}
The Khatri--Rao product has the following properties, among others:
\begin{itemize}
\item $({\bf A} \odot {\bf B}) \odot {\bf C} = {\bf A} \odot ({\bf B} \odot {\bf C})$ (associative); so we may simply write as ${\bf A} \odot {\bf B} \odot {\bf C}$. Note though that ${\bf A} \odot {\bf B} \neq {\bf B} \odot {\bf A}$, so the Khatri--Rao product is non-commutative.
\item $({\bf A} \otimes {\bf B})({\bf E} \odot {\bf F}) = ({\bf A E}) \odot ({\bf B F})$ ({\em mixed product rule}).
\end{itemize}

\noindent {\em Tensor (outer) product:} The {\em tensor product} or {\em outer product} of vectors ${\bf a}~(I \times 1)$ and ${\bf b}~(J \times 1)$ is defined as the $I \times J$ matrix ${\bf a} \circledcirc  {\bf b}$ with elements $({\bf a} \circledcirc  {\bf b})(i,j)={\bf a}(i) {\bf b}(j)$, $\forall i, j$. Note that
${\bf a} \circledcirc  {\bf b} = {\bf a} {\bf b}^T$. Introducing a third vector ${\bf c}~(K \times 1)$, we can generalize to the outer product of three vectors, which is an $I \times J \times K$ {\em three-way array} or {\em third-order tensor} ${\bf a} \circledcirc  {\bf b} \circledcirc  {\bf c}$ with elements $({\bf a} \circledcirc  {\bf b} \circledcirc  {\bf c})(i,j,k)={\bf a}(i) {\bf b}(j) {\bf c}(k)$. Note that the element-wise definition of the outer product naturally generalizes to three- and higher-way cases involving more vectors, but one loses the `transposition' representation that is familiar in the two-way (matrix) case.

\section{Rank and rank decomposition for tensors: CPD / PARAFAC}
\label{sect:rank}

We know that the outer product of two vectors is a `simple' rank-1 matrix -- in fact we may define matrix rank as the minimum number of rank-1 matrices (outer products of two vectors) needed to synthesize a given matrix. We can express this in different ways: $\text{rank}({\bf X})=F$ if and only if (iff) $F$ is the smallest integer such that ${\bf X} = {\bf A} {\bf B}^T$ for some ${\bf A}=\left[{\bf a}_1,\cdots,{\bf a}_F\right]$ and ${\bf B}=\left[{\bf b}_1,\cdots,{\bf b}_F\right]$, or, equivalently,
${\bf X}(i,j)=\sum_{f=1}^F {\bf A}(i,f) {\bf B}(j,f)$ $=$ $\sum_{f=1}^F {\bf a}_f(i) {\bf b}_f(j)$, $\forall i, j$ $\Longleftrightarrow$ ${\bf X} = \sum_{f=1}^F {\bf a}_f \circledcirc  {\bf b}_f$ $=$ $\sum_{f=1}^F {\bf a}_f {\bf b}_f^T$.

\begin{figure}[t!]
\vspace{-30pt}
\includegraphics[height=1.8in]{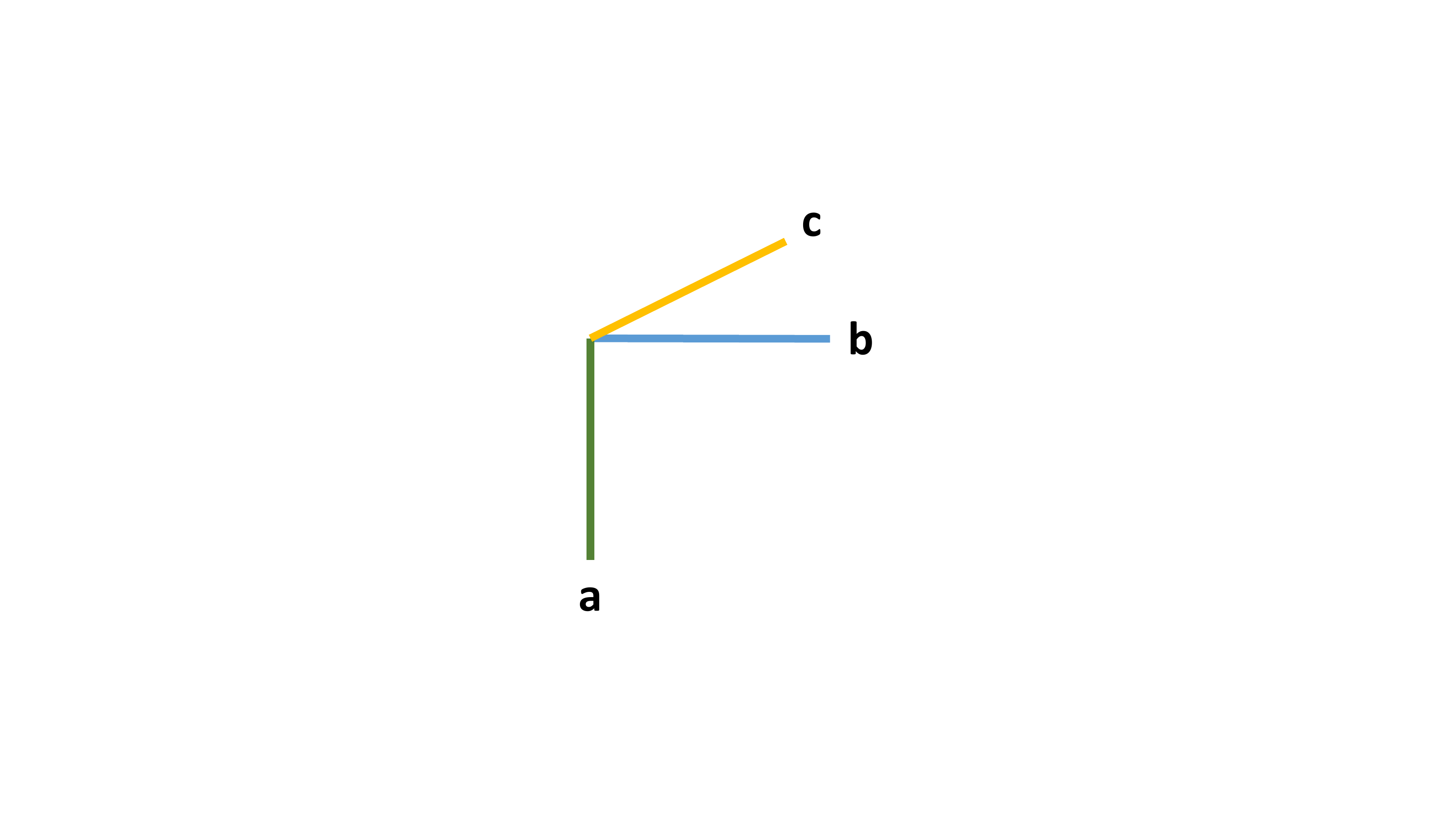}
\vspace{-35pt}
\caption{Schematic of a rank-1 tensor.}\label{fig-nikos1}
\vspace{-13pt}
\end{figure}

A {\em rank-1} third-order tensor $\Xt$ of size $I \times J \times K$ is an outer product of three vectors: $\Xt(i,j,k)={\bf a}(i) {\bf b}(j) {\bf c}(k)$, $\forall i \in \left\{1,\cdots,I\right\}$, $j \in \left\{1,\cdots,J\right\}$, and $k \in \left\{1,\cdots,K\right\}$; i.e., $\Xt={\bf a} \circledcirc  {\bf b} \circledcirc  {\bf c}$ -- see Fig. \ref{fig-nikos1}. A rank-1 $N$-th order tensor $\Xt$ is likewise an outer product of $N$ vectors: $\Xt(i_1,\cdots,i_N)={\bf a}_1(i_1) \cdots {\bf a}_N(i_N)$, $\forall i_n \in \left\{1,\cdots,I_n\right\}$, $\forall n \in \left\{1,\cdots,N\right\}$; i.e., $\Xt={\bf a}_1 \circledcirc  \cdots \circledcirc  {\bf a}_N$. In the sequel we mostly focus on third-order tensors for brevity; everything naturally generalizes to higher-order tensors, and we will occasionally comment on such generalization, where appropriate.

The {\em rank} of tensor $\Xt$ is the minimum number of rank-1 tensors needed to produce $\Xt$ as their sum -- see Fig. \ref{fig-nikos2} for a tensor of rank three. Therefore, a tensor of rank at most $F$ can be written as
\[
\Xt = \sum_{f=1}^F {\bf a}_f \circledcirc  {\bf b}_f \circledcirc  {\bf c}_f \Longleftrightarrow \Xt(i,j,k) = \sum_{f=1}^F {\bf a}_f(i) {\bf b}_f(j) {\bf c}_f(k)
\]
\[
= \sum_{f=1}^F {\bf A}(i,f) {\bf B}(j,f) {\bf C}(k,f),~ \left\{ \forall
\begin{array}{c}
i \in \left\{1,\cdots,I\right\}\\
j \in \left\{1,\cdots,J\right\}\\
k \in \left\{1,\cdots,K\right\}\\
\end{array}
\right.
\]
where ${\bf A}:=\left[{\bf a}_1,\cdots,{\bf a}_F\right]$, ${\bf B}:=\left[{\bf b}_1,\cdots,{\bf b}_F\right]$, and
${\bf C}:=\left[{\bf c}_1,\cdots,{\bf c}_F\right]$. It is also customary to use $a_{i,f}:={\bf A}(i,f)$, so $\Xt(i,j,k) = \sum_{f=1}^F a_{i,f} b_{j,f} c_{k,f}$. For brevity, we sometimes also use the notation $\Xt = \left\llbracket{\bf A},{\bf B},{\bf C}\right\rrbracket$ to denote the relationship $\Xt = \sum_{f=1}^F {\bf a}_f \circledcirc  {\bf b}_f \circledcirc  {\bf c}_f$.

\begin{figure}[t!]
	\vspace{-20pt}
	\includegraphics[width=.45\textwidth]{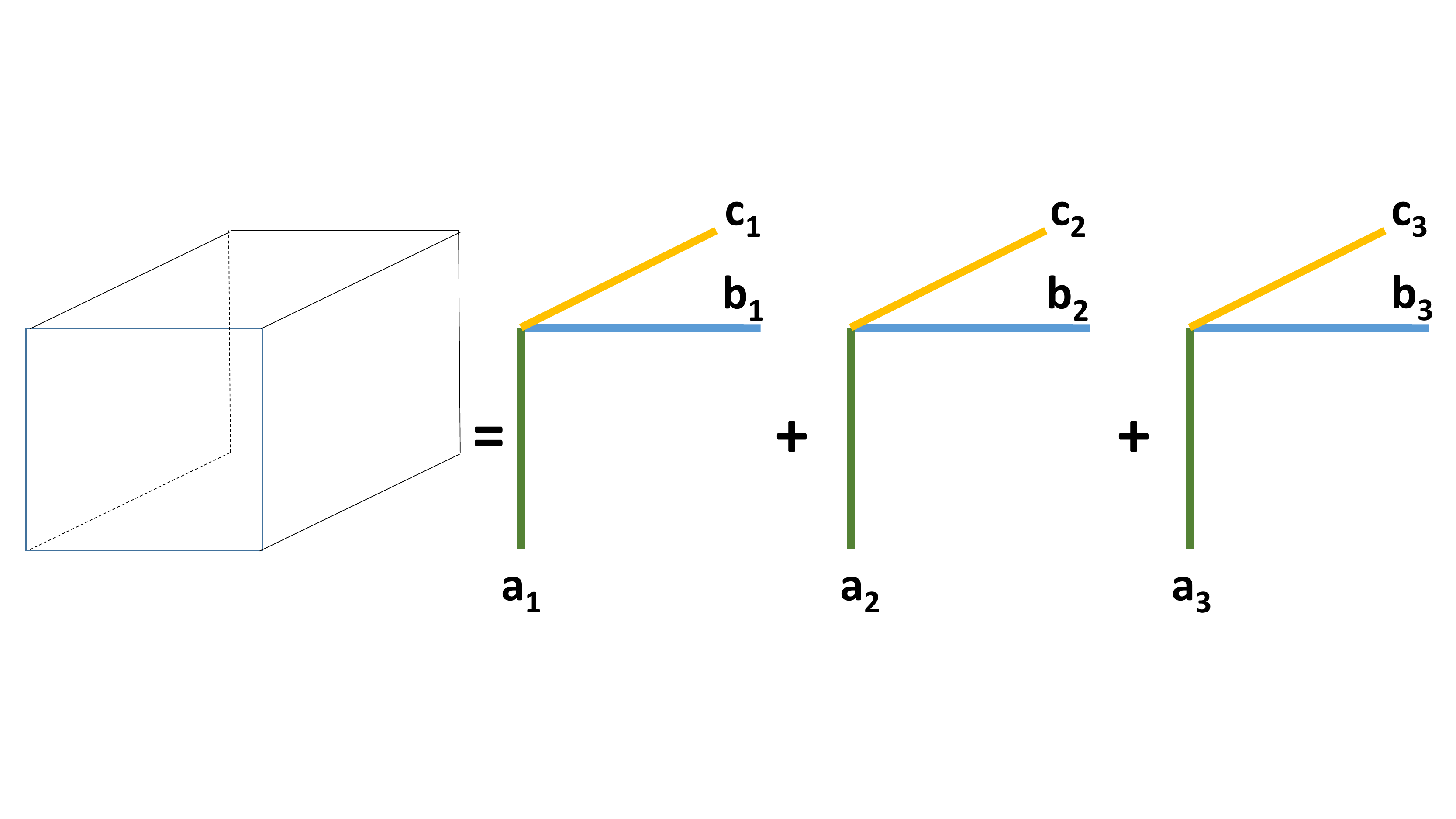}
	\vspace{-30pt}
	\caption{Schematic of tensor of rank three.}\label{fig-nikos2}
	\vspace{-10pt}
\end{figure}

Let us now fix $k=1$ and look at the {\em frontal slab} ${\bf X}(:,:,1)$ of $\Xt$. Its elements can be written as
\[
{\bf X}(i,j,1) = \sum_{f=1}^{F} {\bf a}_f(i) {\bf b}_f(j) {\bf c}_f(1)
\]
\[
\Longrightarrow {\bf X}(:,:,1) = \sum_{f=1}^{F} {\bf a}_f {\bf b}_f^T {\bf c}_f(1) =
\]
\[
{\bf A} \text{Diag}([{\bf c}_1(1),{\bf c}_2(1),\cdots,{\bf c}_F(1)]) {\bf B}^T = {\bf A} \text{Diag}({\bf C}(1,:)) {\bf B}^T,
\]
where we note that the elements of the first row of ${\bf C}$ weigh the rank-1 factors (outer products of corresponding columns of ${\bf A}$ and ${\bf B}$). We will denote $\text{D}_k({\bf C}):= \text{Diag}({\bf C}(k,:))$ for brevity. Hence, for any $k$,
\[
{\bf X}(:,:,k) = {\bf A} \text{D}_k({\bf C}) {\bf B}^T.
\]
Applying the vectorization property of $\odot$ it now follows that
\[
\text{vec}({\bf X}(:,:,k))=({\bf B} \odot {\bf A})({\bf C}(k,:))^T,
\]
and by parallel stacking, we obtain the matrix {\em unfolding} (or, matrix {\em view})
\[
{\bf X}_3 := \left[\text{vec}({\bf X}(:,:,1)), \text{vec}({\bf X}(:,:,2)), \cdots, \text{vec}({\bf X}(:,:,K))\right] \rightarrow
\]
\begin{equation}
\label{mu3}
{\bf X}_3 = ({\bf B} \odot {\bf A}){\bf C}^T, ~~~ (IJ \times K).
\end{equation}
We see that, when cast as a matrix, a third-order tensor of rank $F$ admits factorization in two matrix factors, {\em one of which is specially structured} -- being the Khatri--Rao product of two smaller matrices. One more application of the vectorization property of $\odot$ yields the $IJK \times 1$ vector
\[
{\bf x}_3 = \left( {\bf C} \odot ({\bf B} \odot {\bf A})\right) {\bf 1} = \left( {\bf C} \odot {\bf B} \odot {\bf A}\right) {\bf 1},
\]
where ${\bf 1}$ is an $F \times 1$ vector of all 1's. Hence, when converted to a long vector, a tensor of rank $F$ is a sum of $F$ {\em structured} vectors, each being the Khatri--Rao / Kronecker product of three vectors (in the three-way case; or more vectors in higher-way cases).

In the same vain, we may consider lateral or horizontal slabs\footnote{A warning for Matlab aficionados: due to the way that Matlab stores and handles tensors, one needs to use the `squeeze' operator, i.e., $\text{squeeze}({\bf X}(:,j,:)) = {\bf A} \text{D}_j({\bf B}) {\bf C}^T$, and $\text{vec}(\text{squeeze}({\bf X}(:,j,:))) = ({\bf C} \odot {\bf A}) ({\bf B}(j,:))^T$.}, e.g.,
\[
{\bf X}(:,j,:) = {\bf A} \text{D}_j({\bf B}) {\bf C}^T \rightarrow \text{vec}({\bf X}(:,j,:)) = ({\bf C} \odot {\bf A}) ({\bf B}(j,:))^T.
\]
Hence
\[
{\bf X}_2 := \left[\text{vec}({\bf X}(:,1,:)), \text{vec}({\bf X}(:,2,:)), \cdots, \text{vec}({\bf X}(:,J,:))\right] \rightarrow
\]
\begin{equation}
\label{mu2}
{\bf X}_2 = ({\bf C} \odot {\bf A}){\bf B}^T, ~~~ (IK \times J),
\end{equation}
and similarly\footnote{One needs to use the `squeeze' operator here as well.} ${\bf X}(i,:,:) = {\bf B} \text{D}_i({\bf A}) {\bf C}^T$, so
\[
{\bf X}_1 := \left[\text{vec}({\bf X}(1,:,:)), \text{vec}({\bf X}(2,:,:)), \cdots, \text{vec}({\bf X}(I,:,:))\right] \rightarrow
\]
\vspace{-20pt}
\begin{equation}
\label{mu1}
{\bf X}_1 = ({\bf C} \odot {\bf B}){\bf A}^T, ~~~ (KJ \times I).
\end{equation}

\subsection{Low-rank tensor approximation}

We are in fact ready to get a first glimpse on how we can go about estimating ${\bf A}$, ${\bf B}$, ${\bf C}$ from (possibly noisy) data $\Xt$. Adopting a least squares criterion, the problem is
\[
\min_{{\bf A}, {\bf B}, {\bf C}} ||\Xt - \sum_{f=1}^F {\bf a}_f \circledcirc  {\bf b}_f \circledcirc  {\bf c}_f||_F^2,
\]
where $||\Xt||_F^2$ is the sum of squares of all elements of $\Xt$ (the subscript $F$ in $||\cdot||_F$ stands for {\em Frobenius} (norm), and it should not be confused with the number of {\em factors} $F$ in the rank decomposition -- the difference will always be clear from context). Equivalently, we may consider
\[
\min_{{\bf A}, {\bf B}, {\bf C}} ||{\bf X}_1 - ({\bf C} \odot {\bf B}){\bf A}^T||_F^2.
\]
Note that the above model is nonconvex (in fact {\em trilinear}) in ${\bf A}$, ${\bf B}$, ${\bf C}$; but fixing
${\bf B}$ and ${\bf C}$, it becomes (conditionally) linear in ${\bf A}$, so that we may update
\[
{\bf A} \leftarrow \text{arg} \min_{{\bf A}} ||{\bf X}_1 - ({\bf C} \odot {\bf B}){\bf A}^T||_F^2,
\]
and, using the other two matrix representations of the tensor, update
\[
{\bf B} \leftarrow \text{arg} \min_{{\bf B}} ||{\bf X}_2 - ({\bf C} \odot {\bf A}){\bf B}^T||_F^2,
\]
and
\[
{\bf C} \leftarrow \text{arg} \min_{{\bf C}} ||{\bf X}_3 - ({\bf B} \odot {\bf A}){\bf C}^T||_F^2,
\]
until convergence. The above algorithm, widely known as {\em Alternating Least Squares} (ALS) is a popular way of computing approximate low-rank models of tensor data. We will discuss algorithmic issues in depth at a later stage, but it is important to note that ALS is very easy to program, and we encourage the reader to do so -- this exercise helps a lot in terms of developing the ability to `think three-way'.

\subsection{Bounds on tensor rank}

For an $I \times J$ matrix ${\bf X}$, we know that $\text{rank}({\bf X}) \leq \min(I,J)$,
and $\text{rank}({\bf X}) = \min(I,J)$ almost surely, meaning that rank-deficient real (complex) matrices are a set of Lebesgue measure zero in $\Real^{I \times J}$  $(\Complex^{I \times J})$. What can we say about $I \times J \times K$ tensors $\Xt$? Before we get to this, a retrospective on the matrix case is useful. Considering ${\bf X} = {\bf A} {\bf B}^T$ where ${\bf A}$ is $I \times F$ and ${\bf B}$ is $J \times F$, the size of such parametrization (the {\em number of unknowns}, or {\em degrees of freedom} (DoF) in the model) of ${\bf X}$ is\footnote{Note that we have taken away $F$ DoF due to the scaling / counter-scaling ambiguity, i.e., we may always multiply a column of ${\bf A}$ and divide the corresponding column of ${\bf B}$ with any nonzero number without changing ${\bf A} {\bf B}^T$.} $(I+J-1)F$. The number of equations in ${\bf X} = {\bf A} {\bf B}^T$ is $IJ$, and equations-versus-unknowns considerations suggest that $F$ of order $\min(I,J)$ may be needed -- and this turns out being sufficient as well.

For third-order tensors, the DoF in the low-rank parametrization $\Xt = \sum_{f=1}^F {\bf a}_f \circledcirc  {\bf b}_f \circledcirc  {\bf c}_f$ is\footnote{Note that here we can scale, e.g., ${\bf a}_f$ and ${\bf b}_f$ at will, and counter-scale ${\bf c}_f$, which explains the $( \ldots -2)F$.} $(I+J+K-2)F$, whereas the number of equations is $IJK$. This suggests that
$F \geq \lceil \frac{IJK}{I+J+K-2} \rceil$ may be needed to describe an {\em arbitrary} tensor $\Xt$ of size $I \times J \times K$, i.e., that third-order tensor rank can potentially be as high as $\min(IJ,JK,IK)$. In fact this turns out being sufficient as well. One way to see this is as follows: any frontal slab $\Xt(:,:,k)$ can always be written as $\Xt(:,:,k)={\bf A}_k {\bf B}_k^T$, with ${\bf A}_k$ and ${\bf B}_k$ having at most $\min(I,J)$ columns. Upon defining ${\bf A} := \left[{\bf A}_1, \cdots, {\bf A}_K\right]$, ${\bf B} := \left[{\bf B}_1, \cdots, {\bf B}_K\right]$, and ${\bf C} := {\bf I}_{K \times K} \otimes {\bf 1}_{1 \times \min(I,J)}$ (where ${\bf I}_{K \times K}$ is an identity matrix of size $K \times K$, and ${\bf 1}_{1 \times \min(I,J)}$ is a vector of all 1's of size $1 \times  \min(I,J)$), we can synthesize $\Xt$ as $\Xt = \left\llbracket{\bf A},{\bf B},{\bf C}\right\rrbracket$. Noting that ${\bf A}_k$ and ${\bf B}_k$ have at most $\min(I,J)$ columns, it follows that we need at most $\min(IK,JK)$ columns in ${\bf A}$, ${\bf B}$, ${\bf C}$. Using `role symmetry' (switching the names of the `ways' or `modes'), it follows that we in fact need at most $\min(IJ,JK,IK)$ columns in ${\bf A}$, ${\bf B}$, ${\bf C}$, and thus the rank of any $I \times J \times K$ three-way array $\Xt$ is bounded above by $\min(IJ,JK,IK)$. Another (cleaner but perhaps less intuitive) way of arriving at this result is as follows. Looking at the $IJ \times K$ matrix unfolding
\[
{\bf X}_3 := \left[\text{vec}({\bf X}(:,:,1)), \cdots, \text{vec}({\bf X}(:,:,K))\right] = ({\bf B} \odot {\bf A}){\bf C}^T,
\]
and noting that $({\bf B} \odot {\bf A})$ is $IJ \times F$ and ${\bf C}^T$ is $F \times K$, the issue is what is the maximum inner dimension $F$ that we need to be able to express an {\em arbitrary} $IJ \times K$ matrix  ${\bf X}_3$ on the left (corresponding to an arbitrary $I \times J \times K$ tensor $\Xt$) as a Khatri--Rao product of two $I \times F$, $J \times F$ matrices, times another $F \times K$ matrix? The answer can be seen as follows:
\[
\text{vec}({\bf X}(:,:,k)) =  \text{vec}({\bf A}_k {\bf B}_k^T) = ({\bf B}_k \odot {\bf A}_k) {\bf 1},
\]
and thus we need at most $\min(I,J)$ columns per column of ${\bf X}_3$, which has $K$ columns -- QED.

This upper bound on tensor rank is important because it spells out that tensor rank is finite, and not much larger than the equations-versus-unknowns bound that we derived earlier. On the other hand, it is also useful to have lower bounds on rank. Towards this end, concatenate the frontal slabs one next to each other
\[
\left[ {\bf X}(:,:,1) \cdots {\bf X}(:,:,K) \right] = {\bf A} \left[ {\bf D}_k({\bf C}) {\bf B}^T \cdots {\bf D}_k({\bf C}) {\bf B}^T \right]
\]
since ${\bf X}(:,:,k) = {\bf A} {\bf D}_k({\bf C}) {\bf B}^T$. Note that ${\bf A}$ is $I \times F$, and it follows that $F$ must be greater than or equal to the dimension of the column span of $\Xt$, i.e., the number of linearly independent columns needed to synthesize any of the $JK$ columns $\Xt(:,j,k)$ of $\Xt$. By role symmetry, and upon defining
\[
R_1({\Xt}) := \text{dim} ~ \text{colspan}({\Xt}) := \text{dim} ~ \text{span} \left\{ \Xt(:,j,k) \right\}_{\forall j, k},
\]
\[
R_2({\Xt}) := \text{dim} ~ \text{rowspan}({\Xt}) := \text{dim} ~ \text{span} \left\{ \Xt(i,:,k) \right\}_{\forall i, k},
\]
\[
R_3({\Xt}) := \text{dim} ~ \text{fiberspan}({\Xt}) := \text{dim} ~ \text{span} \left\{ \Xt(i,j,:) \right\}_{\forall i, j},
\]
we have that $F \geq \max(R_1({\Xt}),R_2({\Xt}),R_3({\Xt}))$. $R_1({\Xt})$ is the {\em mode-1 or mode-A rank} of $\Xt$, and likewise $R_2({\Xt})$ and $R_3({\Xt})$ are the {\em mode-2 or mode-B} and {\em mode-3 or mode-C} ranks of $\Xt$, respectively. $R_1({\Xt})$ is sometimes called the {\em column rank}, $R_2({\Xt})$ the {\em row rank}, and $R_3({\Xt})$ the {\em fiber} or {\em tube rank} of $\Xt$. The triple $(R_1({\Xt}),R_2({\Xt}),R_3({\Xt}))$ is called the {\em multilinear rank} of $\Xt$.

At this point it is worth noting that, for matrices we have that column rank = row rank = rank, i.e., in our current notation, for a matrix ${\bf M}$ (which can be thought of as an $I \times J \times 1$ third-order tensor) it holds that
$R_1({\bf M}) = R_2({\bf M}) = \text{rank}({\bf M})$, but for nontrivial tensors $R_1({\Xt})$, $R_2({\Xt})$, $R_3({\Xt})$ and $\text{rank}({\Xt})$ are in general different, with $\text{rank}({\Xt}) \geq \max(R_1({\Xt}),R_2({\Xt}),R_3({\Xt}))$. Since $R_1({\Xt}) \leq I$, $R_2({\Xt}) \leq J$, $R_3({\Xt}) \leq K$, it follows that
$\text{rank}({\bf M}) \leq \min(I,J)$ for matrices but $\text{rank}({\Xt})$ {\em can be} $>$ $\max(I,J,K)$ for tensors.

Now, going back to the first way of explaining the upper bound we derived on tensor rank, it should be clear that we only need
$\min(R_1({\Xt}),R_2({\Xt}))$ rank-1 factors to describe any given frontal slab of the tensor, and so we can describe all slabs with at most $\min(R_1({\Xt}),R_2({\Xt}))K$ rank-1 factors; with a little more thought, it is apparent that $\min(R_1({\Xt}),R_2({\Xt})) R_3({\Xt})$ is enough. Appealing to role symmetry, it then follows that $F \leq \min(R_1({\Xt}) R_2({\Xt}), R_2({\Xt}) R_3({\Xt}), R_1({\Xt}) R_3({\Xt}))$, where $F:=\text{rank}({\Xt})$. Dropping the explicit dependence on ${\Xt}$ for brevity, we have
\[
\max(R_1,R_2,R_3) \leq F \leq \min(R_1 R_2, R_2 R_3, R_1 R_3).
\]

\subsection{Typical, generic, and border rank of tensors}

Consider a $2 \times 2 \times 2$ tensor ${\Xt}$ whose elements are i.i.d., drawn from the standard normal distribution ${\cal N}(0,1)$ ($\Xt=\text{randn(2,2,2)}$ in Matlab). The rank of ${\Xt}$ {\em over the real field}, i.e., when we consider
\[
\Xt = \sum_{f=1}^F {\bf a}_f \circledcirc  {\bf b}_f \circledcirc  {\bf c}_f,~~~ {\bf a}_f \in \Real^{2\times 1}, {\bf b}_f \in \Real^{2\times 1}, {\bf c}_f \in \Real^{2\times 1}, \forall f
\]
is \cite{Bergqvist2013663}
\[
\text{rank}(\Xt) = \left\{ \begin{array}{ll}
                             2, & \text{with probability}~ \frac{\pi}{4} \\
                             3, &  \text{with probability}~ 1-\frac{\pi}{4}
                           \end{array} \right.
\]
This is very different from the matrix case, where $\text{rank}(\text{randn(2,2)})=2$ with probability 1. To make matters more (or less) curious, the rank of the same $\Xt=\text{randn(2,2,2)}$ is in fact 2 with probability 1 when we instead consider decomposition over the complex field, i.e., using ${\bf a}_f \in \Complex^{2\times 1}, {\bf b}_f \in \Complex^{2\times 1}, {\bf c}_f \in \Complex^{2\times 1}, \forall f$. As another example \cite{Bergqvist2013663}, for $\Xt=\text{randn(3,3,2)}$,
\[
\text{rank}(\Xt) = \left\{ \begin{array}{ll} \begin{array}{ll}
                             3, & \text{with probability}~ \frac{1}{2} \\
                             4, &  \text{with probability}~ \frac{1}{2}
                           \end{array}, & \text{over}~\Real;\\\\
                           \begin{array}{ll}
                             3, & \text{with probability}~ 1
                           \end{array}, & \text{over}~\Complex.
                           \end{array} \right.
\]
To understand this behavior, consider the $2 \times 2 \times 2$ case. We have two $2 \times 2$ slabs, ${\bf S}_1:={\bf X}(:,:,1)$ and ${\bf S}_2:={\bf X}(:,:,2)$. For $\Xt$ to have $\text{rank}(\Xt)=2$, we must be able to express these two slabs as
\[
{\bf S}_1 = {\bf A} {\bf D}_1({\bf C}) {\bf B}^T,~\text{and}~~~{\bf S}_2 = {\bf A} {\bf D}_2({\bf C}) {\bf B}^T,
\]
for some $2 \times 2$ real or complex matrices ${\bf A}$, ${\bf B}$, and ${\bf C}$, depending on whether we decompose over the real or the complex field. Now, if $\Xt=\text{randn(2,2,2)}$, then both ${\bf S}_1$ and ${\bf S}_2$ are nonsingular matrices, almost surely (with probability 1). It follows from the above equations that ${\bf A}$, ${\bf B}$, ${\bf D}_1({\bf C})$, and ${\bf D}_2({\bf C})$ must all be nonsingular too. Denoting $\tilde {\bf A} := {\bf A} {\bf D}_1({\bf C})$, ${\bf D} := ({\bf D}_1({\bf C}))^{-1} {\bf D}_2({\bf C})$, it follows that ${\bf B}^T = (\tilde {\bf A})^{-1} {\bf S}_1$, and substituting in the second equation we obtain ${\bf S}_2 = \tilde {\bf A} {\bf D} (\tilde {\bf A})^{-1} {\bf S}_1$, i.e., we obtain the eigen-problem
\[
{\bf S}_2 {\bf S}_1^{-1} = \tilde {\bf A} {\bf D} (\tilde {\bf A})^{-1}.
\]
It follows that for $\text{rank}(\Xt)=2$ over $\Real$, the matrix ${\bf S}_2 {\bf S}_1^{-1}$ should have two {\em real} eigenvalues; but complex conjugate eigenvalues do arise with positive probability. When they do, we have $\text{rank}(\Xt)=2$ over $\Complex$, but $\text{rank}(\Xt) \geq 3$ over $\Real$ -- and it turns out that $\text{rank}(\Xt) = 3$ over $\Real$ is enough.

We see that the rank of a tensor for decomposition over $\Real$ is a random variable that can take more than one value with positive probability. These values are called {\em typical ranks}. For decomposition over $\Complex$ the situation is different: $\text{rank}(\text{randn(2,2,2)})=2$ with probability 1, so there is only one typical rank. When there is only one typical rank (that occurs with probability 1 then) we call it {\em generic rank}.

All these differences with the usual matrix algebra may be fascinating -- and they don't end here either. Consider
\[
\Xt = {\bf u} \circledcirc  {\bf u} \circledcirc  {\bf v} +  {\bf u} \circledcirc  {\bf v} \circledcirc  {\bf u} + {\bf v} \circledcirc  {\bf u} \circledcirc  {\bf u},
\]
where $||{\bf u}||=||{\bf v}||=1$, with $|<{\bf u},{\bf v}>| \neq 1$, where $<\cdot,\cdot>$ stands for the inner product. This tensor has rank equal to 3, however it can be {\em arbitrarily well} approximated \cite{CEM:CEM584} by the following sequence of rank-two tensors (see also \cite{Kolda09tensordecompositions}):
\[
\Xt_n = n ({\bf u} + \frac{1}{n} {\bf v}) \circledcirc  ({\bf u} + \frac{1}{n} {\bf v}) \circledcirc  ({\bf u} + \frac{1}{n} {\bf v}) - n {\bf u} \circledcirc  {\bf u} \circledcirc  {\bf u}
\]
\[
= {\bf u} \circledcirc  {\bf u} \circledcirc  {\bf v} + {\bf u} \circledcirc  {\bf v} \circledcirc  {\bf u} + {\bf v} \circledcirc  {\bf u} \circledcirc  {\bf u} +
\]
\[
+ \frac{1}{n} {\bf v} \circledcirc  {\bf v} \circledcirc  {\bf u} + + \frac{1}{n} {\bf u} \circledcirc  {\bf v} \circledcirc  {\bf v} + \frac{1}{n^2} {\bf v} \circledcirc  {\bf v} \circledcirc  {\bf v},
\]
so
\[
\Xt_n = \Xt + \text{terms that vanish as~} n \rightarrow \infty.
\]
$\Xt$ has rank equal to 3, but {\em border rank} equal to 2 \cite{ComonTensorsBriefIntro}. It is also worth noting that $\Xt_n$ contains two diverging rank-1 components that progressively cancel each other approximately, leading to ever-improving approximation of $\Xt$. This situation is actually encountered in practice when fitting tensors of border rank lower than their rank. Also note that the above example shows clearly that the low-rank tensor approximation problem
\[
\min_{\left\{ {\bf a}_f, {\bf b}_f, {\bf c}_f \right\}_{f=1}^F} \left|\left| \Xt - \sum_{f=1}^F {\bf a}_f \circledcirc  {\bf b}_f \circledcirc  {\bf c}_f \right| \right|_F^2,
\]
is ill-posed in general, for there is no minimum if we pick $F$ equal to the border rank of $\Xt$ -- the set of tensors of a given rank is not closed. There are many ways to fix this ill-posedness, e.g., by adding constraints such as element-wise non-negativity of ${\bf a}_f, {\bf b}_f, {\bf c}_f$ \cite{CEM:CEM1244,DBLP:journals/corr/QiCL14} in cases where $\Xt$ is element-wise non-negative (and these constraints are physically meaningful), or orthogonality \cite{RePEc:spr:psycho:v:73:y:2008:i:3:p:431-439} -- any application-specific constraint that prevents terms from diverging while approximately canceling each other will do. An alternative is to add norm regularization to the cost function, such as $\lambda \left( ||{\bf A}||_F^2 + ||{\bf B}||_F^2 + ||{\bf C}||_F^2 \right)$. This can be interpreted as coming from a Gaussian prior on the sought parameter matrices; yet, if not properly justified, regularization may produce artificial results and a false sense of security.

Some useful results on maximal and typical rank for decomposition over $\Real$ are summarized in Tables~\ref{tab1}, \ref{tab2}, \ref{tab3} -- see \cite{Kolda09tensordecompositions,landsberg-tensors-2012} for more results of this kind, as well as original references.
Notice that, for a tensor of a given size, there is always one typical rank over $\Complex$, which is therefore generic. For $I_1 \times I_2 \times \cdots \times I_N$ tensors, this generic rank is the value $\lceil \frac{\prod_{n=1}^{N} I_n}{\sum_{n=1}^{N} I_n -N +1} \rceil$ that can be expected from the equations-versus-unknowns reasoning, except for the so-called defective cases (i) $I_1 > \prod_{n=2}^{N} I_n - \sum_{n=2}^{N} (I_n -1)$ (assuming w.l.o.g. that the first dimension $I_1$ is the largest), (ii) the third-order case of dimension $(4,4,3)$, (iii) the third-order cases of dimension $(2p+1,2p+1,3)$, $p \in \mathbb{N}$, and (iv) the fourth-order cases of dimension $(p,p,2,2)$, $p \in \mathbb{N}$, where it is 1 higher \footnote{In fact this has been verified for $R \leq 55$, with the probability that a defective case has been overlooked less than $10^{-55}$, the limitations being a matter of computing power \cite{vannieuwenhoven2014randomized}.}. Also note that the typical rank may change when the tensor is constrained in some way; e.g., when the frontal slabs are symmetric, we have the results in Table~\ref{tab3}, so symmetry may restrict the typical rank. Also, one may be interested in symmetric or asymmetric rank decomposition (i.e., symmetric or asymmetric rank-1 factors) in this case, and therefore symmetric or regular rank. Consider, for example, a fully symmetric tensor, i.e., one such that $\Xt(i,j,k)=\Xt(i,k,j)=\Xt(j,i,k)=\Xt(j,k,i)=\Xt(k,i,j)=\Xt(k,j,i)$, i.e., its value is invariant to any permutation of the three indices (the concept readily generalizes to $N$-way tensors $\Xt(i_i,\cdots,i_N)$). Then the {\em symmetric rank} of $\Xt$ over $\Complex$ is defined as the minimum $R$ such that $\Xt$ can be written as
$\Xt = \sum_{r=1}^R {\bf a}_r \circledcirc  {\bf a}_r \circledcirc  \cdots \circledcirc  {\bf a}_r$, where the outer product involves $N$ copies of vector ${\bf a}_r$, and ${\bf A} := \left[{\bf a}_1, \cdots, {\bf a}_R \right] \in \Complex^{I \times R}$. It has been shown that this symmetric rank equals $\lceil \binom{I+N-1}{N} / I \rceil$ almost surely except in the defective cases $(N,I) = (3,5), (4,3), (4,4), (4,5)$, where it is 1 higher  \cite{alexander1995polynomial}. Taking $N=3$ as a special case, this formula gives $\frac{(I+1)(I+2)}{6}$. We also remark that constraints such as nonnegativity of a factor matrix can strongly affect rank.

\begin{table}
	\centering
	
	\caption{Maximum attainable rank over $\Real$.}
	\label{tab1}
	\vspace{-5pt}
	\begin{tabular}{|l|c|}
		\hline
		Size & Maximum attainable rank over $\Real$\\
		\hline
		$I \times J \times 2$ & $\min(I,J)+\min(I,J,\lfloor\max(I,J)/2\rfloor)$ \\
		\hline
		$2 \times 2 \times 2$ & $3$\\
		\hline
		$3 \times 3 \times 3$ & $5$\\
		\hline
	\end{tabular}
	
	\vspace{10pt}
	
	\caption{Typical rank over $\Real$}
	\label{tab2}
	\vspace{-5pt}
	\begin{tabular}{|l|l|}
		\hline
		Size & Typical ranks over $\Real$\\
		\hline
		$I \times I \times 2$ & $\left\{I,I+1\right\}$ \\
		\hline
		$I \times J \times 2$, $I > J$ & $\min(I,2J)$ \\
		\hline
		$I \times J \times K$, $I > JK$ & $JK$ \\
		\hline
	\end{tabular}
	
	\vspace{10pt}
	
	\caption{Symmetry may affect typical rank.}
	\label{tab3}
	\vspace{-5pt}
	\begin{tabular}{|l|l|l|}
		\hline
		Size & Typical ranks, $\Real$ & Typical ranks, $\Real$\\
		& partial symmetry & no symmetry\\
		\hline
		$I \times I \times 2$ & $\left\{I,I+1\right\}$ & $\left\{I,I+1\right\}$\\
		\hline
		$9 \times 3 \times 3$ & $6$ & $9$ \\
		\hline
	\end{tabular}
	
	\vspace{-10pt}
\end{table}

Given a particular tensor $\Xt$, determining $\text{rank}(\Xt)$ is NP-hard \cite{Hastad:1990:TRN:96134.95990}. There is a well-known example of a $9 \times 9 \times 9$ tensor\footnote{See the insert entitled {\em Tensors as bilinear operators}.} whose rank (border rank) has been bounded between $19$ and $23$ ($14$ and $21$, resp.), but has not been pinned down yet. At this point, the reader may rightfully wonder whether this is an issue in practical applications of tensor decomposition, or merely a mathematical curiosity? The answer is not black-and-white, but rather nuanced: In most applications, one is really interested in fitting a model that has the ``essential'' or ``meaningful'' number of components that we usually call the (useful signal) rank, which is usually much less than the actual rank of the tensor that we observe, due to noise and other imperfections. Determining this rank is challenging, even in the matrix case. There exist heuristics and a few more disciplined approaches that can help, but, at the end of the day, the process generally involves some trial-and-error.

An exception to the above is certain applications where the tensor actually models a mathematical object (e.g., a multilinear map) rather than ``data''. A good example of this is Strassen's matrix multiplication tensor -- see the insert entitled {\em Tensors as bilinear operators}. A vector-valued (multiple-output) bilinear map can be represented as a third-order tensor, a vector-valued trilinear map as a fourth-order tensor, etc. When working with tensors that represent such maps, one is usually interested in exact factorization, and thus the mathematical rank of the tensor. The border rank is also of interest in this context, when the objective is to obtain a very accurate approximation (say, to within machine precision) of the given map. There are other applications (such as {\em factorization machines}, to be discussed later) where one is forced to approximate a general multilinear map in a possibly crude way, but then the number of components is determined by other means, not directly related to notions of rank. 

Consider again the three matrix views of a given tensor $\Xt$ in \eqref{mu1}, \eqref{mu2}, \eqref{mu3}. Looking at ${\bf X}_1$ in \eqref{mu3}, note that if $({\bf C} \odot {\bf B})$ is full column rank and so is ${\bf A}$, then $\text{rank}({\bf X}_1)=F=\text{rank}(\Xt)$. Hence this matrix view of $\Xt$ is rank-revealing. For this to happen it is {\em necessary} (but not sufficient) that $JK \geq F$, and $I \geq F$, so $F$ has to be small: $F \leq \min(I,JK)$. Appealing to role symmetry of the three modes, it follows that $F \leq \max(\min(I,JK),\min(J,IK),\min(K,IJ))$ is necessary to have a rank-revealing matricization of the tensor. However, we know that the (perhaps unattainable) upper bound on $F=\text{rank}(\Xt)$ is $F \leq \min(IJ,JK,IK)$, hence for matricization to reveal rank, it must be that the rank is really small relative to the upper bound. More generally, what holds for sure, as we have seen, is that $F=\text{rank}(\Xt) \geq \max(\text{rank}({\bf X}_1), \text{rank}({\bf X}_2), \text{rank}({\bf X}_3))$.

\begin{table}[ht]
\colorbox{lightgray}{\begin{minipage}{.48\textwidth}\normalsize
{\bf Tensors as bilinear operators:} When multiplying two $2 \times 2$ matrices ${\bf M}_1$, ${\bf M}_2$, every element of the $2 \times 2$ result ${\bf P} = {\bf M}_1 {\bf M}_2$ is a bilinear form $\text{vec}({\bf M}_1)^T {\bf X}_k \text{vec}({\bf M}_2)$, where ${\bf X}_k$ is $4 \times 4$, holding the coefficients that produce the $k$-th element of $\text{vec}({\bf P})$, $k \in \left\{1,2,3,4\right\}$. Collecting the slabs $\left\{{\bf X}_k\right\}_{k=1}^4$ into a $4 \times 4 \times 4$ tensor $\Xt$, matrix multiplication can be implemented by means of evaluating $4$ bilinear forms involving the $4$ frontal slabs of $\Xt$. Now suppose that $\Xt$ admits a rank decomposition involving matrices ${\bf A}$, ${\bf B}$, ${\bf C}$ (all $4 \times F$ in this case). Then any element of ${\bf P}$ can be written as $\text{vec}({\bf M}_1)^T {\bf A} {\bf D}_k({\bf C}) {\bf B}^T \text{vec}({\bf M}_2)$. Notice that ${\bf B}^T \text{vec}({\bf M}_2)$ can be computed using $F$ inner products, and the same is true for $\text{vec}({\bf M}_1)^T {\bf A}$. If the elements of ${\bf A}$, ${\bf B}$, ${\bf C}$ take values in $\left\{0,\pm1\right\}$ (as it turns out, this is true for the ``naive'' as well as the minimal decomposition of $\Xt$), then these inner products require no multiplication -- only selection, addition, subtraction. Letting ${\bf u}^T := \text{vec}({\bf M}_1)^T {\bf A}$ and ${\bf v} := {\bf B}^T \text{vec}({\bf M}_2)$, it remains to compute ${\bf u}^T  {\bf D}_k({\bf C}) {\bf v} = \sum_{f=1}^F {\bf u}(f) {\bf v}(f) {\bf C}(k,f)$, $\forall k \in \left\{1,2,3,4\right\}$. This entails $F$ multiplications to compute the products $\left\{{\bf u}(f) {\bf v}(f)\right\}_{f=1}^F$ -- the rest is all selections, additions, subtractions if ${\bf C}$ takes values in $\left\{0,\pm1\right\}$. Thus $F$ multiplications suffice to multiply two $2 \times 2$ matrices -- and it so happens, that the rank of Strassen's $4 \times 4 \times 4$ tensor is $7$, so $F=7$ suffices. Contrast this to the ``naive'' approach which entails $F=8$ multiplications (or, a ``naive'' decomposition of Strassen's tensor involving ${\bf A}$, ${\bf B}$, ${\bf C}$ all of size $4 \times 8$).
\end{minipage}}
\vspace{-8pt}
\end{table}

Before we move on, let us extend what we have done so far to the case of $N$-way tensors. Let us start with $4$-way tensors, whose rank decomposition can be written as
\[
\Xt(i,j,k,\ell) = \sum_{f=1}^{F} {\bf a}_f(i) {\bf b}_f(j) {\bf c}_f(k) {\bf e}_f(\ell), \forall \left\{ \begin{array}{l}
i \in \left\{1, \cdots, I \right\}\\
j \in \left\{1, \cdots, J \right\}\\
k \in \left\{1, \cdots, K \right\}\\
\ell \in \left\{1, \cdots, L \right\}\\
\end{array}
\right.
\]
\[
\text{or, equivalently~} \Xt = \sum_{f=1}^{F} {\bf a}_f \circledcirc  {\bf b}_f \circledcirc  {\bf c}_f \circledcirc  {\bf e}_f.
\]
Upon defining ${\bf A} :=\left[{\bf a}_1,\cdots,{\bf a}_F\right]$, ${\bf B} :=\left[{\bf b}_1,\cdots,{\bf b}_F\right]$, ${\bf C} :=\left[{\bf c}_1,\cdots,{\bf c}_F\right]$, ${\bf E} :=\left[{\bf e}_1,\cdots,{\bf e}_F\right]$, we may also write
\[
\Xt(i,j,k,\ell) = \sum_{f=1}^{F} {\bf A}(i,f) {\bf B}(j,f) {\bf C}(k,f) {\bf E}(\ell,f),
\]
and we sometimes also use $\Xt(i,j,k,\ell) = \sum_{f=1}^{F} a_{i,f} b_{j,f} c_{k,f} e_{\ell,f}$. Now consider $\Xt(:,:,:,1)$, which is a third-order tensor. Its elements are given by
\[
\Xt(i,j,k,1) = \sum_{f=1}^{F} a_{i,f} b_{j,f} c_{k,f} e_{1,f},
\]
where we notice that the `weight' $e_{1,f}$ is independent of $i,j,k$, it only depends on $f$, so we would normally absorb it in, say, $a_{i,f}$, if we only had to deal with $\Xt(:,:,:,1)$ -- but here we don't, because we want to model $\Xt$ as a whole. Towards this end, let us vectorize $\Xt(:,:,:,1)$ into an $IJK \times 1$ vector
\[
\text{vec}\left( \text{vec}\left(\Xt(:,:,:,1)\right)\right) = ({\bf C} \odot {\bf B} \odot {\bf A}) ({\bf E}(1,:))^T,
\]
where the result on the right should be contrasted with $({\bf C} \odot {\bf B} \odot {\bf A}) {\bf 1}$, which would have been the result had we absorbed $e_{1,f}$ in $a_{i,f}$. Stacking one next to each other the vectors corresponding to $\Xt(:,:,:,1)$, $\Xt(:,:,:,2)$, $\cdots$, $\Xt(:,:,:,L)$, we obtain $({\bf C} \odot {\bf B} \odot {\bf A}) {\bf E}^T$; and after one more $\text{vec}(\cdot)$ we get $({\bf E} \odot {\bf C} \odot {\bf B} \odot {\bf A}) {\bf 1}$.

\begin{table}[ht]
\colorbox{lightgray}{\begin{minipage}{.48\textwidth}\normalsize
{\bf Multiplying two complex numbers:} Another interesting example involves the multiplication of two complex numbers -- each represented as a $2 \times 1$ vector comprising its real and imaginary part. Let $j := \sqrt{-1}$, $x=x_r+jx_i \leftrightarrow {\bf x}:=[x_r ~ x_i]^T$, $y=y_r+jy_i \leftrightarrow {\bf y}:=[y_r ~ y_i]^T$. Then $xy = (x_r y_r - x_i y_i) + j (x_r y_i + x_r y_r) =: z_r + j z_i$. It appears that $4$ real multiplications are needed to compute the result; but in fact $3$ are enough. To see this, note that the $2 \times 2 \times 2$ multiplication tensor in this case has frontal slabs
\[
\Xt(:,:,1)=
\left[
\begin{array}{lr}
1 & 0 \\
0 & -1
\end{array}
\right],~~~
\Xt(:,:,2)=
\left[
\begin{array}{lr}
0 & 1 \\
1 & 0
\end{array}
\right],
\]
whose rank is at most $3$, because
\[
\left[
\begin{array}{lr}
1 & 0 \\
0 & -1
\end{array}
\right] = \left[
\begin{array}{c}
1\\
0
\end{array}
\right] \left[
\begin{array}{c}
1\\
0
\end{array}
\right]^T - \left[
\begin{array}{c}
0\\
1
\end{array}
\right] \left[
\begin{array}{c}
0\\
1
\end{array}
\right]^T,
\]
and
\[
\left[
\begin{array}{lr}
0 & 1 \\
1 & 0
\end{array}
\right] = \left[
\begin{array}{c}
1\\
1
\end{array}
\right] \left[
\begin{array}{c}
1\\
1
\end{array}
\right]^T - \left[
\begin{array}{c}
1\\
0
\end{array}
\right] \left[
\begin{array}{c}
1\\
0
\end{array}
\right]^T -
\left[
\begin{array}{c}
0\\
1
\end{array}
\right] \left[
\begin{array}{c}
0\\
1
\end{array}
\right]^T,
\]
Thus taking
\[
{\bf A} = {\bf B} = \left[
                      \begin{array}{ccc}
                        1 & 0 & 1 \\
                        0 & 1 & 1 \\
                      \end{array}
                    \right], ~~ {\bf C} = \left[
                      \begin{array}{ccc}
                        1 & -1 & 0 \\
                        -1 & -1 & 1 \\
                      \end{array}
                    \right],
\]
we only need to compute $p_1 = x_r y_r$, $p_2 = x_i y_i$, $p_3 = (x_r + x_i)(y_r + y_i)$, and then $z_r = p_1 - p_2$, $z_i = p_3 - p_1 - p_2$. Of course, we did not need tensors to invent these computation schedules -- but tensors can provide a way of obtaining them.
\end{minipage}}
\vspace{-15pt}
\end{table}

It is also easy to see that, if we fix the last two indices and vary the first two, we get
\[
\Xt(:,:,k,\ell) = {\bf A} {\bf D}_k({\bf C}) {\bf D}_{\ell}({\bf E}) {\bf B}^T,
\]
so that
\[
\text{vec}\left(\Xt(:,:,k,\ell)\right)=({\bf B} \odot {\bf A})({\bf C}(k,:) * {\bf E}(\ell,:))^T,
\]
where $*$ stands for the Hadamard (element-wise) matrix product. If we now stack these vectors one next to each other, we obtain the following ``balanced'' matricization\footnote{An alternative way to obtain this is to start from $({\bf E} \odot {\bf C} \odot {\bf B} \odot {\bf A}) {\bf 1}$ $=$ $(({\bf E} \odot {\bf C}) \odot ({\bf B} \odot {\bf A})) {\bf 1}$ $=$ vectorization of $({\bf B} \odot {\bf A}) ({\bf E} \odot {\bf X})^T$, by the vectorization property of $\odot$.} of the $4$-th order tensor $\Xt$:
\[
{\bf X}_b = ({\bf B} \odot {\bf A}) ({\bf E} \odot {\bf C})^T.
\]
This is interesting because the inner dimension is $F$, so if ${\bf B} \odot {\bf A}$ and ${\bf E} \odot {\bf C}$ are both full column rank, then $F=\text{rank}({\bf X}_b)$, i.e., the matricization ${\bf X}_b$ is {\em rank-revealing} in this case. Note that full column rank of ${\bf B} \odot {\bf A}$ and ${\bf E} \odot {\bf C}$ requires $F \leq \min(IJ,KL)$, which seems to be a more relaxed condition than in the three-way case. The catch is that, for $4$-way tensors, the corresponding upper bound on tensor rank (obtained in the same manner as for third-order tensors) is $F \leq \min(IJK,IJL,IKL,JKL)$ -- so the upper bound on tensor rank increases as well. Note that the boundary where matricization can reveal tensor rank remains off by one order of magnitude relative to the upper bound on rank, when $I=J=K=L$. In short: matricization can generally reveal the tensor rank in low-rank cases only.

Note that once we have understood what happens with $3$-way and $4$-way tensors, generalizing to $N$-way tensors for any integer $N \geq 3$ is easy. For a general $N$-way tensor, we can write it in scalar form as
\[
\Xt(i_1,\cdots,i_N)=\sum_{f=1}^F {\bf a}_f^{(1)}(i_1) \cdots {\bf a}_f^{(N)}(i_N) =
\sum_{f=1}^F a_{i_1,f}^{(1)} \cdots a_{i_N,f}^{(N)},
\]
and in (combinatorially!) many different ways, including
\[
{\bf X}_N = ({\bf A}_{N-1} \odot \cdots \odot {\bf A}_1) {\bf A}_N^T \rightarrow \text{vec}({\bf X}_N) = ({\bf A}_N \odot \cdots \odot {\bf A}_1) {\bf 1}.
\]
We sometimes also use the shorthand $\text{vec}({\bf X}_N) = \left(\odot_{n=N}^1 {\bf A}_n\right) {\bf 1}$, 
where $\text{vec}(\cdot)$ is now a compound operator, and the order of vectorization only affects the ordering of the factor matrices in the Khatri--Rao product.

\section{Uniqueness, demystified} \label{sect:uniq}

We have already emphasized what is perhaps the most significant advantage of low-rank decomposition of third- and higher-order tensors versus low-rank decomposition of matrices (second-order tensors): namely, the former is essentially unique under mild conditions, whereas the latter is never essentially unique, unless the rank is equal to one, or else we impose additional constraints on the factor matrices. The reason why uniqueness happens for tensors but not for matrices may seem like a bit of a mystery at the beginning. The purpose of this section is to shed light in this direction, by assuming more stringent conditions than necessary to enable simple and insightful proofs. First, a concise definition of essential uniqueness.

\begin{definition}
	Given a tensor $\Xt$ of rank $F$, we say that its CPD is {\em essentially unique} if the $F$ rank-1 terms in its decomposition (the outer products or ``chicken feet'') in Fig. \ref{fig-nikos2} are unique, i.e., there is no other way to decompose $\Xt$ for the given number of terms. Note that we can of course permute these terms without changing their sum, hence there exists an inherently unresolvable permutation ambiguity in the rank-1 tensors. If $\Xt = \left\llbracket{\bf A},{\bf B},{\bf C}\right\rrbracket$, with ${{\bf A}}:~I \times F$, ${{\bf B}}:~J \times F$, and ${{\bf C}}:~K \times F$, then essential uniqueness means that ${\bf A}$, ${\bf B}$, and ${\bf C}$ are unique up to a common permutation and scaling / counter-scaling of columns, meaning that if $\Xt = \left\llbracket {\bar {\bf A}},{\bar {\bf B}},{\bar {\bf C}}\right\rrbracket$, for some ${\bar {\bf A}}:~I \times F$, ${\bar {\bf B}}:~J \times F$, and ${\bar {\bf C}}:~K \times F$, then there exists a permutation
	matrix $\bPi$ and diagonal scaling matrices ${\bLambda}_{1},{\bLambda}_{2},{\bLambda}_{3}$ such that
	\[
	{\bar {\bf A}} = {\bf A} {\bPi} {\bLambda}_{1},~ {\bar {\bf B}} =
	{\bf B} {\bPi} {\bLambda}_{2},~ {\bar {\bf C}} = {\bf C} {\bPi}
	{\bLambda}_{3},~ {\bLambda}_{1} {\bLambda}_{2} {\bLambda}_{3} = {\bf I}.
	\]
\end{definition}
\begin{remark}
	Note that if we under-estimate the true rank $F=\text{rank}(\Xt)$, it is impossible to fully decompose the given tensor using $R < F$ terms by definition. If we use $R > F$, uniqueness cannot hold unless we place conditions on ${\bf A}$, ${\bf B}$, ${\bf C}$. In particular, for uniqueness it is necessary that each of the matrices ${\bf A} \odot {\bf B}$, ${\bf B} \odot {\bf C}$ and ${\bf C} \odot {\bf A}$ is full column rank. Indeed, if for instance ${\bf a}_R \otimes {\bf b}_R = \sum_{r=1}^{R-1} d_r {\bf a}_r \otimes {\bf b}_r$, then $\Xt = \left\llbracket{\bf A}(:,1:R-1),{\bf B}(:,1:R-1),{\bf C}(:,1:R-1) + {\bf c}_R {\bf d}^T \right\rrbracket$, with ${\bf d} = \left[d_1, \cdots, d_{R-1}\right]^T$, is an alternative decomposition that involves only $R-1$ rank-1 terms, i.e. the number of rank-1 terms has been overestimated.
\end{remark}
We begin with the simplest possible line of argument. Consider an $I \times J \times 2$ tensor $\Xt$ of rank $F \leq \min(I,J)$. We know that the maximal rank of an $I \times J \times 2$ tensor over $\Real$ is $\min(I,J)+\min(I,J,\lfloor\max(I,J)/2\rfloor)$, and typical rank is $\min(I,2J)$ when $I > J$, or $\left\{I,I+1\right\}$ when $I=J$ (see Tables \ref{tab1}, \ref{tab2}) -- so here we purposefully restrict ourselves to low-rank tensors (over $\Complex$ the argument is more general).

Let us look at the two frontal slabs of $\Xt$. Since $\text{rank}(\Xt)=F$, it follows that
\[
{\bf X}^{(1)} := \Xt(:,:,1)={\bf A} {\bf D}_1({\bf C}) {\bf B}^T,
\]
\[
{\bf X}^{(2)} := \Xt(:,:,2)={\bf A} {\bf D}_2({\bf C}) {\bf B}^T,
\]
where ${\bf A}$, ${\bf B}$, ${\bf C}$ are $I \times F$, $J \times F$, and $2 \times F$, respectively. Let us assume that the multilinear rank of $\Xt$ is $(F,F,2)$, which implies that
$\text{rank}\left({\bf A}\right)=\text{rank}\left({\bf B}\right)=F$.
Now define the pseudo-inverse ${\bf E} := ({\bf B}^T)^\dagger$. It is clear that the columns of ${\bf E}$ are generalized eigenvectors of the matrix pencil $({\bf X}^{(1)}, {\bf X}^{(2)})$:
\[
{\bf X}^{(1)} {\bf e}_f = c_{1,f} {\bf a}_f, ~~~
{\bf X}^{(2)} {\bf e}_f = c_{2,f} {\bf a}_f.
\]
(In the case $I=J$ and assuming that ${\bf X}^{(2)}$ is full rank, the Generalized EVD (GEVD) is algebraically equivalent with the basic EVD ${\bf X}^{(2)^{-1}} {\bf X}^{(1)} = {\bf B}^{-T} {\bf D} {\bf B}^{T}$ where ${\bf D} := \mbox{diag}(c_{1,1}/c_{2,1}, \cdots, c_{1,F}/c_{2,F})$; however, there are numerical differences.) For the moment we assume that the generalized eigenvalues are distinct, i.e. no two columns of ${\bf C}$ are proportional.
There is freedom to scale the generalized eigenvectors (they remain generalized eigenvectors), and obviously one cannot recover the order of the columns of ${\bf E}$. This means that there is permutation and scaling ambiguity in recovering ${\bf E}$. That is, what we do recover is actually ${\widetilde {\bf E}} = {\bf E} \bPi \bLambda$, where $\bPi$ is a permutation matrix and $\bLambda$ is a nonsingular diagonal scaling matrix. If we use ${\widetilde {\bf E}}$ to recover ${\bf B}$, we will in fact recover $({\widetilde {\bf E}}^T)^\dagger = {\bf B} \bPi \bLambda^{-1}$ -- that is, ${\bf B}$ up to the same column permutation and scaling. It is now easy to see that we can recover ${\bf A}$ and ${\bf C}$ by going back to the original equations for ${\bf X}^{(1)}$ and ${\bf X}^{(2)}$ and multiplying from the right by ${\widetilde {\bf E}} = \left[\tilde{{\bf e}}_{1}, \cdots, \tilde{{\bf e}}_{F}\right]$. Indeed, since $\tilde{{\bf e}}_{\tilde{f}} = \lambda_{\tilde{f},\tilde{f}} {\bf e}_f$ for some $f$, we obtain per column a rank-1 matrix $\left[{\bf X}^{(1)} \tilde{{\bf e}}_{\tilde{f}}, {\bf X}^{(2)} \tilde{{\bf e}}_{\tilde{f}} \right] = \lambda_{\tilde{f},\tilde{f}} {\bf a}_f {\bf c}_f^T$, from which the corresponding column of ${\bf A}$ and ${\bf C}$ can be recovered.

The basic idea behind this type of EVD-based uniqueness proof has been rediscovered many times under different disguises and application areas. We refer the reader to Harshman (who also credits Jenkins) \cite{Har70,Har72}. The main idea is similar to a well-known parameter estimation technique in signal processing, known as ESPRIT \cite{ESPRIT}. A detailed and streamlined EVD proof that also works when $I \neq J$ and $F < \min(I,J)$ and is constructive (suitable for implementation) can be found in the supplementary material. That proof owes much to ten Berge \cite{CEM:CEM1204} for the random slab mixing argument.

\begin{remark}
	Note that if we start by assuming that $\text{rank}(\Xt)=F$ over $\Real$, then, by definition, all the matrices involved will be real, and the eigenvalues in ${\bf D}$ will also be real. If $\text{rank}(\Xt)=F$ over $\Complex$, then whether ${\bf D}$ is real or complex is not an issue.
\end{remark}

Note that there are $F$ linearly independent eigenvectors {\em by construction} under our working assumptions. Next, if two or more of the generalized eigenvalues are identical, then linear combinations of the corresponding eigenvectors are also eigenvectors, corresponding to the same generalized eigenvalue. Hence distinct generalized eigenvalues are {\em necessary} for uniqueness.\footnote{Do note however that, even in this case, uniqueness breaks down only partially, as eigenvectors corresponding to other, distinct eigenvalues are still unique up to scaling.}
The generalized eigenvalues are distinct
if and only if {\em any two} columns of ${\bf C}$ are linearly independent -- in which case we say that ${\bf C}$ has {\em Kruskal rank} $\geq 2$. The definition of Kruskal rank is as follows.
\begin{definition}
	The {\em Kruskal rank} $k_{\bf A}$ of an $I \times F$ matrix ${\bf A}$ is the largest integer $k$ such that {\em any} $k$ columns of ${\bf A}$ are linearly independent. Clearly, $k_{\bf A} \leq r_{\bf A} := \text{rank}({\bf A}) \leq \min(I,F)$. Note that $k_{\bf A} = s_{\bf A} - 1 := \text{spark}({\bf A}) - 1$, where $\text{spark}({\bf A})$ is the minimum number of linearly dependent columns of ${\bf A}$ (when this is $\leq F$). Spark is a familiar notion in the compressed sensing literature, but Kruskal rank was defined earlier.
\end{definition}
We will see that the notion of Kruskal rank plays an important role in  uniqueness results in our context, notably in what is widely known as Kruskal's result (in fact, a ``common denominator'' implied by a number of results that Kruskal has proven in his landmark paper \cite{Kru77}).
Before that, let us summarize the result we have just obtained.

\begin{theorem}
Given $\Xt = \left\llbracket{\bf A},{\bf B},{\bf C}\right\rrbracket$, with ${{\bf A}}:~I \times F$, ${{\bf B}}:~J \times F$, and ${{\bf C}}:~2 \times F$, if $F > 1$ it is {\it necessary} for uniqueness of ${\bf A}$, ${\bf B}$ that $k_{\bf C}=2$. If, in addition $r_{{\bf A}} = r_{{\bf B}} = F$, then $\mbox{rank}(\Xt) = F$ and the decomposition of $\Xt$ is essentially unique.
\label{theor:GEVD}
\end{theorem}

For tensors that consist of $K\geq 2$ slices, one can consider a pencil of two random slice mixtures and infer the following result from Theorem \ref{theor:GEVD}.

\begin{theorem}
Given $\Xt = \left\llbracket{\bf A},{\bf B},{\bf C}\right\rrbracket$, with ${{\bf A}}:~I \times F$, ${{\bf B}}:~J \times F$, and ${{\bf C}}:~K \times F$, if $F > 1$ it is {\it necessary} for uniqueness of ${\bf A}$, ${\bf B}$ that $k_{\bf C} \geq 2$. If, in addition $r_{{\bf A}} = r_{{\bf B}} = F$, then $\mbox{rank}(\Xt) = F$ and the decomposition of $\Xt$ is essentially unique.
\label{theor:uniq2xfcr}
\end{theorem}
A probabilistic version of Theorem \ref{theor:uniq2xfcr} goes as follows.
\begin{theorem}
	Given $\Xt = \left\llbracket{\bf A},{\bf B},{\bf C}\right\rrbracket$, with ${{\bf A}}:~I \times F$, ${{\bf B}}:~J \times F$, and ${{\bf C}}:~K \times F$, if $I \geq F$, $J \geq F$ and $K \geq 2$, then $\mbox{rank}(\Xt) = F$ and the decomposition of $\Xt$ in terms of ${\bf A}$, ${\bf B}$, and ${\bf C}$ is essentially unique, {\em almost surely} (meaning that it is essentially unique for all $\Xt = \left\llbracket{\bf A},{\bf B},{\bf C}\right\rrbracket$ except for a set of measure zero with respect to the Lebesgue measure in $\Real^{(I+J+K-2)F}$ or $\Complex^{(I+J+K-2)F}$).
	\label{theor:uniq2xfcrgen}
\end{theorem}
Now let us relax our assumptions and claim that (at least) one of the loading matrices is full column rank, instead of two. After some reflexion, the matricization ${\bf X}^{(JI \times K)} := \left( {\bf A} \odot {\bf B} \right) {\bf C}^{T}$ yields the following condition, which is both necessary and sufficient.
\begin{theorem} \cite{JiaSid04}
Given $\Xt = \left\llbracket{\bf A},{\bf B},{\bf C}\right\rrbracket$, with ${{\bf A}}:~I \times F$, ${{\bf B}}:~J \times F$, and ${{\bf C}}:~K \times F$, and assuming $r_{\bf C} = F$, it holds that the decomposition
$\Xt=\left\llbracket{\bf A},{\bf B},{\bf C}\right\rrbracket$ is essentially unique $\Longleftrightarrow$ nontrivial linear combinations of columns of ${\bf A} \odot{\bf B}$ cannot be written as $\otimes$ product of two vectors.
\label{theor:uniqnecsuff}
\end{theorem}
Despite its conceptual simplicity and appeal, the above condition is hard to check. In \cite{JiaSid04} it is shown that it is possible to recast this condition as an equivalent criterion on the solutions of a system of quadratic equations -- which is also hard to check, but will serve as a stepping stone to easier conditions and even generalizations of the EVD-based computation.  
Let ${\bf M}_k({\bf A})$ denote the $\binom{I}{k} \times \binom{F}{k}$ $k$-th compound matrix containing all $k \times k$ minors of ${\bf A}$, e.g., for
\[
{\bf A} = \left[
            \begin{array}{cccc}
              a_1 & 1 & 0 & 0 \\
              a_2 & 0 & 1 & 0 \\
              a_3 & 0 & 0 & 1 \\
            \end{array}
          \right]
\]
\[
{\bf M}_2({\bf A}) =
\left[
  \begin{array}{rrrrrr}
  -a_2 & a_1 & 0 & 1 & 0 & 0\\
  -a_3 & 0 & a_1 & 0 & 1 & 0\\
  0 & -a_3 & a_2 & 0 & 0 & 1\\
  \end{array}
\right].
\]
Starting from a vector ${\bf d} = \left[d_{1},\cdots,d_{F} \right]^{T} \in
\Complex^{F}$, let ${\bf v}_k({\bf d})$ consistently denote $\left[d_{1}d_{2} \cdots d_{k}, d_{1}d_{2} \cdots d_{k-1} d_{k+1}, \cdots, d_{F-k+1} d_{F-k+2} \cdots d_{F} \right]^{T} \in
\Complex^{\binom{F}{k}}$.
Theorem \ref{theor:uniqnecsuff} can now be expressed as follows.
\begin{theorem} \cite{JiaSid04}
Given $\Xt = \left\llbracket{\bf A},{\bf B},{\bf C}\right\rrbracket$, with ${{\bf A}}:~I \times F$, ${{\bf B}}:~J \times F$, and ${{\bf C}}:~K \times F$, and assuming $r_{\bf C} = F$, it holds that the decomposition
\begin{eqnarray*}
& \Xt=\left\llbracket{\bf A},{\bf B},{\bf C}\right\rrbracket~\text{is essentially unique}~\Longleftrightarrow~
\\
& \left( {\bf M}_2({\bf B}) \odot {\bf M}_2({\bf A}) \right) {\bf v}_2({\bf d}) = {\bf 0} \\
& \text{implies~that~} {\bf v}_2({\bf d}) = \left[d_{1}d_{2}, d_{1}d_{3}, \cdots, d_{F-1} d_{F} \right]^{T} = {\bf 0}, \\
& \text{i.e., at most one entry of} {\bf d} \text{is nonzero.}
\end{eqnarray*}
\label{theor:uniqnecsuff2}
\end{theorem}
The size of ${\bf M}_2({\bf B}) \odot {\bf M}_2({\bf A})$ is $\binom{I}{2} \binom{J}{2} \times \binom{F}{2}$.
A sufficient condition that can be checked with basic linear algebra is readily obtained by ignoring the structure of ${\bf v}_2({\bf d})$.
\begin{theorem} \cite{JiaSid04,de2006link}
If $r_{\bf C} = F$, and $r_{{\bf M}_2({\bf B}) \odot {\bf M}_2({\bf A})}= \binom{F}{2}$, then $\mbox{rank}(\Xt) = F$ and the decomposition of
$\Xt=\left\llbracket{\bf A},{\bf B},{\bf C}\right\rrbracket$ is essentially unique.
\label{theor:uniqsuffrelax}
\end{theorem}
The generic version of Theorems \ref{theor:uniqnecsuff} and \ref{theor:uniqnecsuff2} has been obtained from an entirely different (algebraic geometry) point of view:
\begin{theorem} \cite{ChiOtt2012,domanov2015generic,strassen1983rank}
Given $\Xt = \left\llbracket{\bf A},{\bf B},{\bf C}\right\rrbracket$, with ${{\bf A}}:~I \times F$, ${{\bf B}}:~J \times F$, and ${{\bf C}}:~K \times F$, let $K \geq F$ and $\min(I,J) \geq 3$. Then $\mbox{rank}(\Xt) = F$ and the decomposition of $\Xt$ is essentially unique, almost surely, if and only if $(I-1)(J-1) \geq F$.
\label{theor:uniqnecsuffalggeom}
\end{theorem}
The next theorem is the generic version of Theorem \ref{theor:uniqsuffrelax}; the second inequality implies that ${\bf M}_2({\bf B}) \odot {\bf M}_2({\bf A})$ does not have more columns than rows.
\begin{theorem} 
Given $\Xt = \left\llbracket{\bf A},{\bf B},{\bf C}\right\rrbracket$, with ${{\bf A}}:~I \times F$, ${{\bf B}}:~J \times F$, and ${{\bf C}}:~K \times F$, if $K \geq F$ and $I(I-1)J(J-1) \geq 2F(F-1)$, then $\mbox{rank}(\Xt) = F$ and the decomposition of $\Xt$ is essentially unique, almost surely.
\label{theor:link2006generic}
\end{theorem}
Note that $(I-1)(J-1) \geq F \Longleftrightarrow IJ-I-J+1 \geq F \Rightarrow IJ \geq F-1+I+J \Rightarrow IJ \geq F-1$, and multiplying the first and the last inequality yields $I(I-1)J(J-1) \geq F(F-1)$. So Theorem \ref{theor:uniqnecsuffalggeom} is at least a factor of 2 more relaxed than Theorem \ref{theor:link2006generic}. Put differently, ignoring the structure of ${\bf v}_2({\bf d})$ makes us lose about a factor of 2 generically.

On the other hand, Theorem \ref{theor:uniqsuffrelax} admits a remarkable constructive interpretation. Consider any rank-revealing decomposition, such as a QR-factorization or an SVD, of ${\bf X}^{(JI \times K)} = {\bf E}{\bf F}^{T}$, involving a $JI \times F$ matrix ${\bf E}$ and a $K \times F$ matrix ${\bf F}$ that both are full column rank. (At this point, recall that full column rank of ${\bf A} \odot {\bf B}$ is necessary for uniqueness, and that ${\bf C}$ is full column rank by assumption.) We are interested in finding an $F \times F$ (invertible) basis transformation matrix ${\bf G}$ such that ${\bf A} \odot {\bf B} = {\bf E} {\bf G}$ and ${\bf C} = {\bf F} {\bf G}^{-T}$. It turns out that, under the conditions in Theorem \ref{theor:uniqsuffrelax} and through the computation of second compound matrices, an $F \times F \times F$ auxiliary tensor ${\bf Y}$ can be derived from the given tensor $\Xt$, admitting the CPD ${\bf Y} = \left\llbracket {\widetilde{\bf G}}, {\widetilde{\bf G}},{\bf H}\right\rrbracket$, in  which ${\widetilde{\bf G}}$ equals ${\bf G}$ up to column-wise scaling and permutation, and in which the $F \times F$ matrix ${\bf H}$ is nonsingular \cite{de2006link}. As the three loading matrices are full column rank, uniqueness of the auxiliary CPD is guaranteed by Theorem \ref{theor:uniq2xfcr}, and it can be computed by means of an EVD. Through a more sophisticated derivation of an auxiliary tensor, \cite{ignatrelaxedLAA} attempts to regain the ``factor of 2'' above and extend the result up to the necessary and sufficient generic bound in Theorem \ref{theor:uniqnecsuffalggeom}; that the latter bound is indeed reached has been verified numerically up to $F = 24$.

Several results have been extended to situations where none of the loading matrices is full column rank, using $m$-th compound matrices ($m > 2$). For instance, the following theorem generalizes Theorem \ref{theor:uniqsuffrelax}:
\begin{theorem} \cite{domanov2013uniquenessII,domanov2014canonical}
Given $\Xt = \left\llbracket{\bf A},{\bf B},{\bf C}\right\rrbracket$, with ${{\bf A}}:~I \times F$, ${{\bf B}}:~J \times F$, and ${{\bf C}}:~K \times F$. Let $m_{\bf C} = F - k_{\bf C} +2$. If
$\max(\min(k_{\bf A}, k_{\bf B} - 1), \min(k_{\bf A} - 1, k_{\bf B})) + k_{\bf C} \geq F + 1$ and
${\bf M}_{m_{\bf C}}({\bf A}) \odot {\bf M}_{m_{\bf C}}({\bf B})$ has full column rank,
then $\mbox{rank}(\Xt) = F$ and the decomposition $\Xt=\left\llbracket{\bf A},{\bf B},{\bf C}\right\rrbracket$ is essentially unique.
\label{theor:uniqdetnonefcr}
\end{theorem}
(To see that Theorem \ref{theor:uniqdetnonefcr} reduces to Theorem \ref{theor:uniqsuffrelax} when $r_{\bf C} = F$,  note that $r_{\bf C} = F$ implies $k_{\bf C} = F$ and recall that  $\min(k_{\bf A},k_{\bf B}) > 1$ is necessary for uniqueness.) Under the conditions in Theorem \ref{theor:uniqdetnonefcr} computation of the CPD can again be reduced to a GEVD \cite{domanov2014canonical}.

It can be shown \cite{domanov2013uniquenessII,domanov2014canonical} that Theorem \ref{theor:uniqdetnonefcr} implies the next theorem, which
is the most well-known result covered by Kruskal; this includes the possibility of reduction to GEVD.
\begin{theorem} \cite{Kru77}
Given $\Xt = \left\llbracket{\bf A},{\bf B},{\bf C}\right\rrbracket$, with ${{\bf A}}:~I \times F$, ${{\bf B}}:~J \times F$, and ${{\bf C}}:~K \times F$, if $k_{\bf A} + k_{\bf B} + k_{\bf C} \geq 2F+2$, then $\mbox{rank}(\Xt) = F$ and the decomposition of $\Xt$ is essentially unique.
\label{theor:kruskal}
\end{theorem}
Note that Theorem \ref{theor:kruskal} is symmetric in ${\bf A}$, ${\bf B}$, ${\bf C}$, while in Theorem \ref{theor:uniqdetnonefcr} the role of ${\bf C}$ is different from that of ${\bf A}$ and ${\bf B}$. Kruskal's condition is sharp, in the sense that there exist decompositions that are not unique as soon as $F$ goes beyond the bound \cite{derksen2013kruskal}. This does not mean that uniqueness is impossible beyond Kruskal's bound -- as indicated, Theorem \ref{theor:uniqdetnonefcr} also covers other cases. (Compare the generic version of Kruskal's condition, $\min(I,F) + \min(J,F) + \min(K,F) \geq 2F +2$, with Theorem \ref{theor:uniqnecsuffalggeom}, for instance.)

Kruskal's original proof is beyond the scope of this overview paper; instead, we refer the reader to \cite{SteSid07} for a compact version that uses only matrix algebra, and to the supplementary material for a relatively simple proof of an intermediate result which still conveys the flavor of Kruskal's derivation.

With respect to generic conditions, one could wonder whether a CPD is not unique almost surely for any value of $F$ strictly less than the generic rank, see the equations-versus-unknowns discussion in Section \ref{sect:rank}. For symmetric decompositions this has indeed been proved, with the exceptions $(N,I;F) = (6,3;9), (4,4;9), (3,6;9)$ where there are two decompositions generically \cite{chiantini2016gensym}. For unsymmetric decompositions it has been verified for tensors up to 15000 entries (larger tensors can be analyzed with a larger computational effort) that the only exceptions are $(I_1, \cdots, I_N; F) = (4,4,3;5)$, $(4,4,4;6)$, $(6,6,3;8)$, $(p,p,2,2;2p-1)$ for $p \in \mathbb{N}$, $(2,2,2,2,2;5)$, and the so-called unbalanced case $I_1 > \alpha$, $F \geq \alpha$, with $\alpha  = \prod_{n=2}^{N} I_n - \sum_{n=2}^{N} (I_n -1)$ \cite{chiantini2014algorithm}.

Note that in the above we assumed that the factor matrices are unconstrained. (Partial) symmetry can be integrated in the deterministic conditions by substituting for instance ${\bf A} = {\bf B}$. (Partial) symmetry does change the generic conditions, as the number of equations / number of parameters ratio is affected, see \cite{domanov2015generic} and references therein for variants. For the partial Hermitian symmetry ${\bf A} = {\bf B}^\ast$ we can do better by constructing the extended $I \times I \times 2K$ tensor $\Xt^{\rm (ext)}$ via $x_{i,j,k}^{\rm (ext)} = x_{i,j,k}$ for $k \leq K$ and $x_{i,j,k}^{\rm (ext)} = x_{j,i,k}^\ast$ for $K+1 \leq k \leq 2K$. We have $\Xt^{\rm (ext)} = \left\llbracket{\bf A},{\bf A}^\ast,{\bf C}^{\rm (ext)}\right\rrbracket$, with ${\bf C}^{\rm (ext)} = \left[{\bf C}^T, {\bf C}^H \right]^T$. Since obviously $r_{{\bf C}^{\rm (ext)}} \geq r_{\bf C}$ and $k_{{\bf C}^{\rm (ext)}} \geq k_{\bf C}$, uniqueness is easier to establish for $\Xt^{\rm (ext)}$ than for $\Xt$ \cite{sorensen2015new}. By exploiting orthogonality, some deterministic conditions can be relaxed as well \cite{MS-LDL-PC-SI-LD}. For a thorough study of implications of nonnegativity, we refer to \cite{DBLP:journals/corr/QiCL14}.

Summarizing, there exist several types of uniqueness conditions. First, there are probabilistic conditions that indicate whether it is reasonable to expect uniqueness for a certain number of terms, given the size of the tensor. Second, there are deterministic conditions that allow one to establish uniqueness for a particular decomposition -- this is useful for an a posteriori analysis of the uniqueness of results obtained by a decomposition algorithm. There  also exist deterministic conditions under which the decomposition can actually be computed using only conventional linear algebra (EVD or GEVD), at least under noise-free conditions. In the case of (mildly) noisy data, such algebraic algorithms can provide a good starting value for optimization-based algorithms (which will be discussed in Section \ref{sect:alg}), i.e. the algebraic solution is refined in an optimization step. Further, the conditions can be affected by constraints. While in the matrix case constraints can make a rank decomposition unique that otherwise is not unique, for tensors the situation is rather that constraints affect the range of values of $F$ for which uniqueness holds.

There exist many more uniqueness results that we didn't touch upon in this overview, but the ones that we did present give a good sense of what is available and what one can expect. In closing this section, we note that many (but not all) of the above results have been extended to the case of higher-order (order $N > 3$) tensors. For example, the following result generalizes Kruskal's theorem to tensors of arbitrary order:
\begin{theorem} \cite{SidBro00} \label{theor:SidBro00}
Given $\Xt = \left\llbracket {\bf A}_1, \ldots, {\bf A}_N \right\rrbracket$, with ${\bf A}_n:~I_n \times F$, if $\sum_{n=1}^N k_{{\bf A}_n} \geq 2F+N-1$, then the decomposition of $\Xt$ in terms of $\left\{{\bf A}_n\right\}_{n=1}^N$ is essentially unique.
\end{theorem}
This condition is sharp in the same sense as the $N=3$ version is sharp \cite{derksen2013kruskal}.
The starting point for proving Theorem \ref{theor:SidBro00} is that a fourth-order tensor of rank $F$ can be written in third-order form as ${\bf X}_{[1,2;3;4]} = \left\llbracket {\bf A}_1 \odot {\bf A}_2, {\bf A}_3, {\bf A}_4 \right\rrbracket$ $=$ $\left\llbracket {\bf A}_{[1,2]}, {\bf A}_3, {\bf A}_4 \right\rrbracket$ -- i.e., can be viewed as a third-order tensor with a specially structured mode loading matrix ${\bf A}_{[1,2]} := {\bf A}_1 \odot {\bf A}_2$. Therefore, Kruskal's third-order result can be applied, and what matters is the k-rank of the Khatri--Rao product ${\bf A}_1 \odot {\bf A}_2$ -- see property \ref{frpkrp} in the supplementary material, 
and \cite{SidBro00} for the full proof.

\section{The Tucker model and Multilinear Singular Value Decomposition} \label{sect:MLSVD}

\subsection{Tucker and CPD}

Any $I \times J$ matrix ${\bf X}$ can be decomposed via SVD as ${\bf X} = {\bf U} \bSigma {\bf V}^T$, where ${\bf U}^T {\bf U} = {\bf I} = {\bf U} {\bf U}^T$, ${\bf V}^T {\bf V} = {\bf I} = {\bf V} {\bf V}^T$, $\bSigma(i,j) \geq 0$, $\bSigma(i,j) > 0$ only when $j=i$ and $i \leq r_{\bf X}$, and $\bSigma(i,i) \geq \bSigma(i+1,i+1)$, $\forall i$. With ${\bf U} := [{\bf u}_1,\cdots, {\bf u}_I]$, ${\bf V} := [{\bf v}_1,\cdots, {\bf v}_J]$, and $\sigma_f := \bSigma(f,f)$, we can thus write ${\bf X} = {\bf U}(:,1:F) \bSigma(1:F,1:F) ({\bf V}(:,1:F))^T = \sum_{f=1}^F \sigma_f {\bf u}_f {\bf v}_f^T$.

The question here is whether we can generalize the SVD to tensors, and if there is a way of doing so that retains the many beautiful properties of matrix SVD. The natural generalization would be to employ another matrix, of size $K \times K$, call it ${\bf W}$, such that ${\bf W}^T {\bf W} = {\bf I} = {\bf W} {\bf W}^T$, and a nonnegative $I \times J \times K$ {\em core tensor} $\bSigma$ such that $\bSigma(i,j,k) > 0$ only when $k=j=i$ -- see the schematic illustration in Fig. \ref{fig-nikos3}.
\begin{figure}[t]
\includegraphics[width=.45\textwidth]{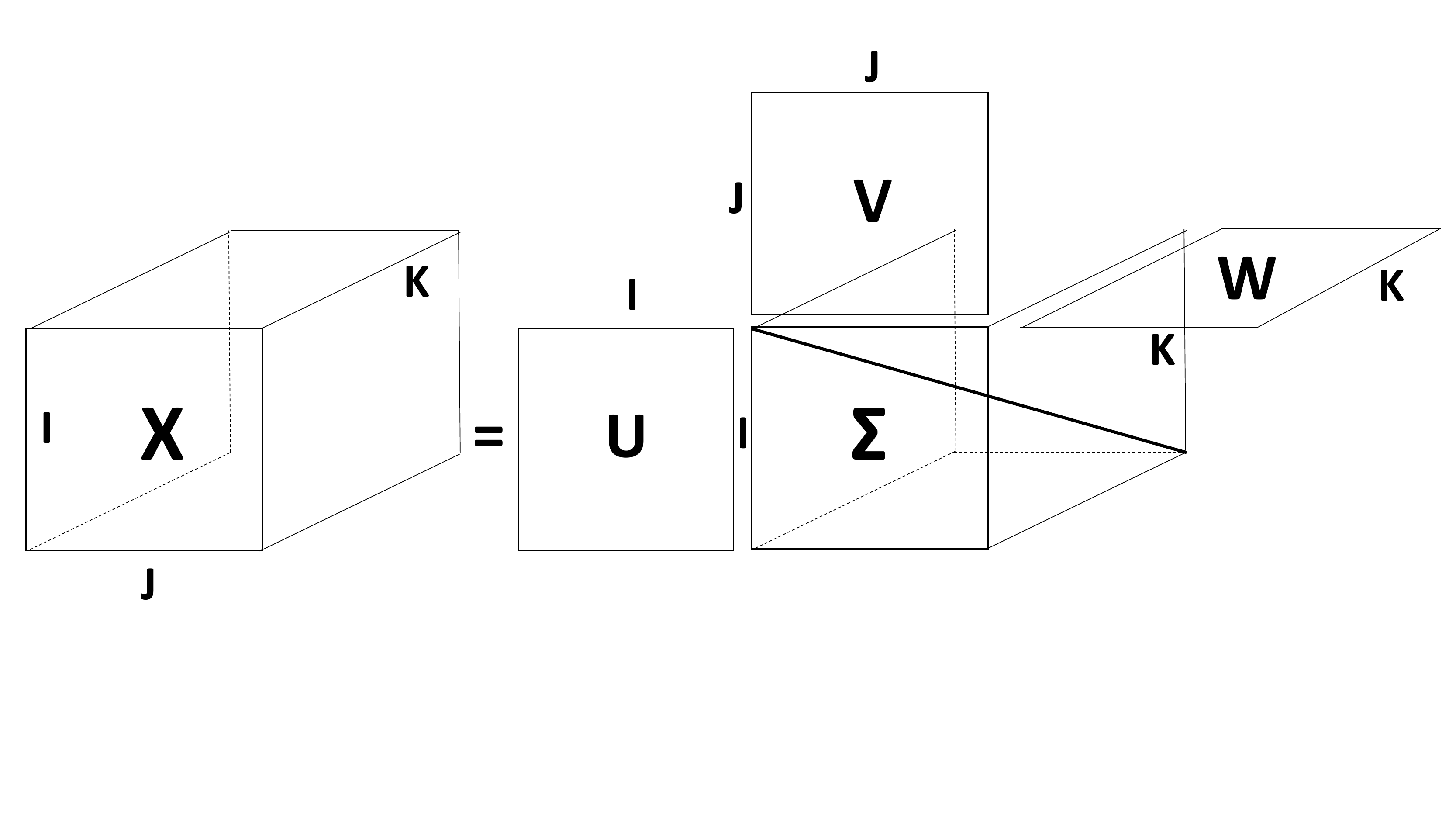}
\vspace{-40pt}
\caption{Diagonal tensor SVD?}\label{fig-nikos3}
\vspace{-10pt}
\end{figure}
Is it possible to decompose an arbitrary tensor in this way? A back-of-the-envelop calculation shows that the answer is no. Even disregarding the orthogonality constraints, the degrees of freedom in such a decomposition would be less\footnote{Since the model exhibits scaling/counter-scaling invariances.} than $I^2+J^2+K^2 + \min(I,J,K)$, which is in general $< IJK$ -- the number of (nonlinear) equality constraints. [Note that, for matrices, $I^2+J^2 + \min(I,J) > I^2+J^2  > IJ$, always.] A more formal way to look at this is that the model depicted in Fig. \ref{fig-nikos3} can be written as
\[
\sigma_1 {\bf u}_1 \circledcirc  {\bf v}_1 \circledcirc  {\bf w}_1 + \sigma_2 {\bf u}_2 \circledcirc  {\bf v}_2 \circledcirc  {\bf w}_2 + \cdots + \sigma_m {\bf u}_m \circledcirc  {\bf v}_m \circledcirc  {\bf w}_m,
\]
where $m := \min(I,J,K)$. The above is a tensor of rank at most $\min(I,J,K)$, but we know that tensor rank can be (much) higher than that. Hence we certainly have to give up diagonality. Consider instead a full (possibly dense, but ideally sparse) core tensor ${\bf G}$, as illustrated in Fig. \ref{fig-nikos4}.
\begin{figure}[t]
\includegraphics[width=.45\textwidth]{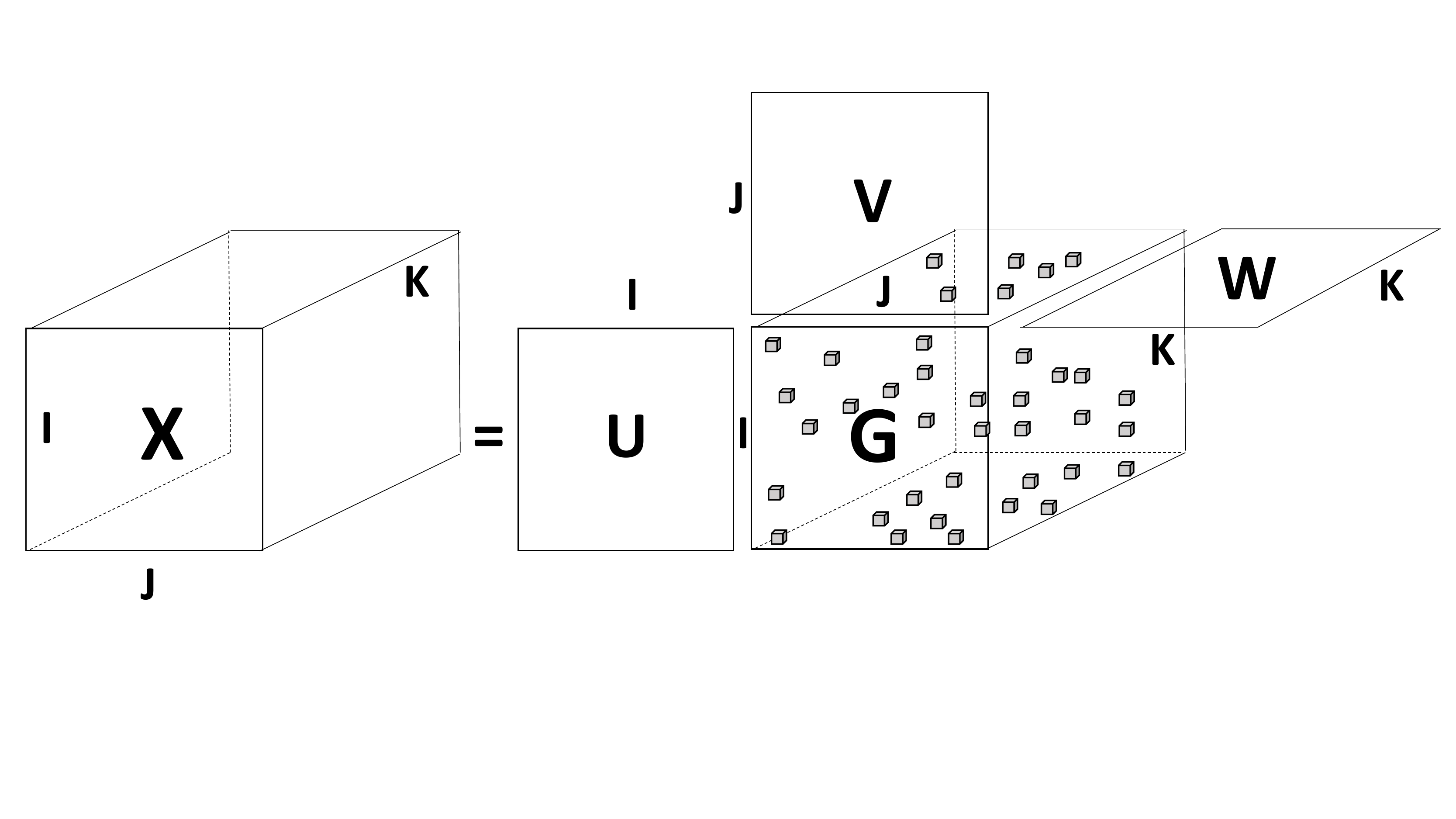}
\vspace{-40pt}
\caption{The Tucker model}\label{fig-nikos4}
\vspace{-10pt}
\end{figure}
\begin{figure}[t]
\includegraphics[width=.45\textwidth]{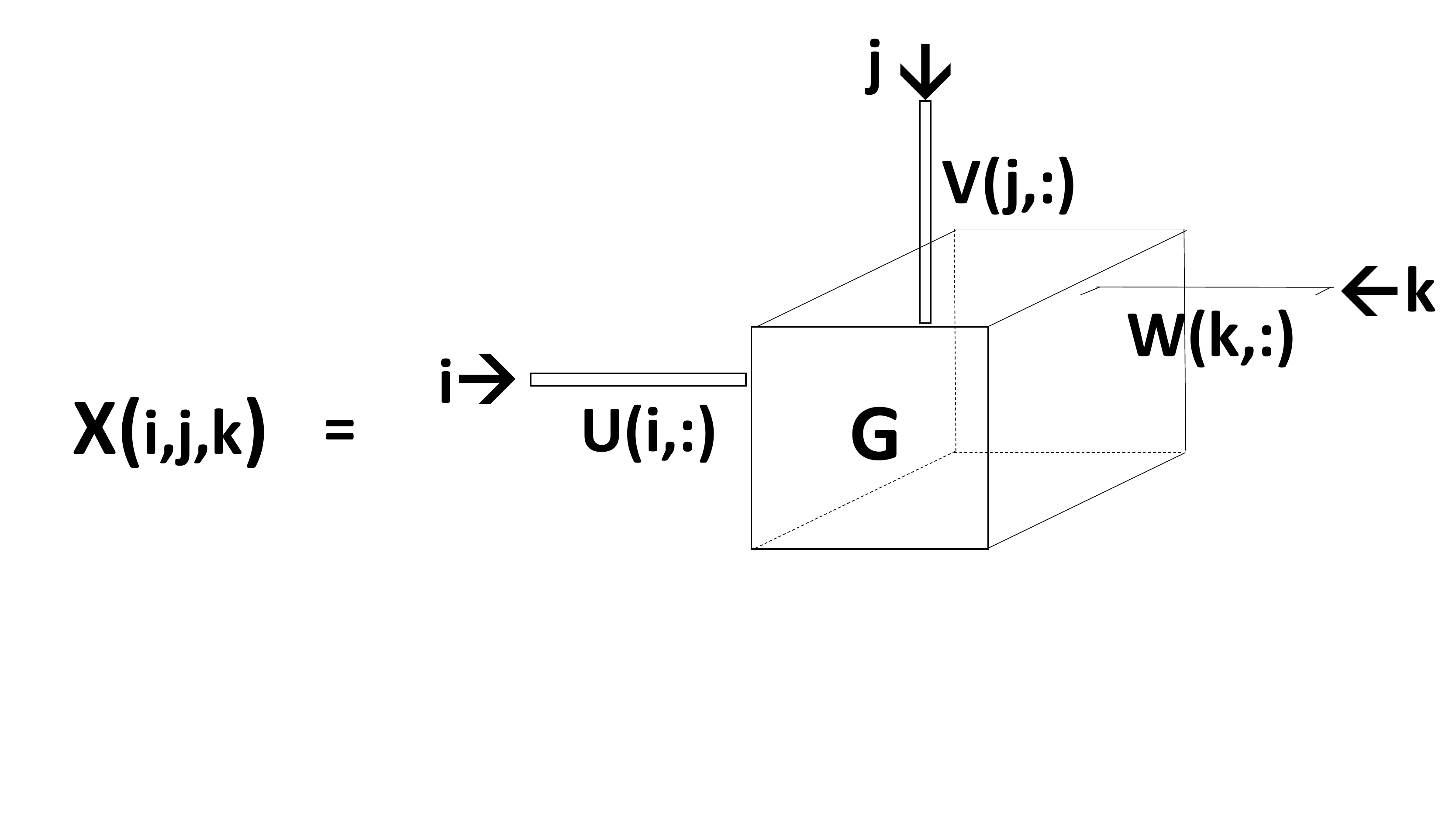}
\vspace{-40pt}
\caption{Element-wise view of the Tucker model}\label{fig-nikos5}
\vspace{-10pt}
\end{figure}
An element-wise interpretation of the decomposition in Fig. \ref{fig-nikos4} is shown in Fig. \ref{fig-nikos5}. From Fig. \ref{fig-nikos5}, we write
\[
\Xt(i,j,k)=\sum_{\ell=1}^I \sum_{m=1}^J \sum_{n=1}^K \Gt(\ell,m,n) {\bf U}(i,\ell) {\bf V}(j,m) {\bf W}(k,n),
\]
or, equivalently,
\[
\Xt = \sum_{\ell=1}^I \sum_{m=1}^J \sum_{n=1}^K \Gt(\ell,m,n) {\bf U}(:,\ell) \circledcirc  {\bf V}(:,m) \circledcirc  {\bf W}(:,n),
\]
\begin{equation}
\label{eqn:Tucker}
\text{or}~\Xt = \sum_{\ell=1}^I \sum_{m=1}^J \sum_{n=1}^K \Gt(\ell,m,n) {\bf u}_{\ell} \circledcirc  {\bf v}_m \circledcirc  {\bf w}_n,
\end{equation}
where ${\bf u}_{\ell} := {\bf U}(:,\ell)$ and likewise for the ${\bf v}_m$, ${\bf w}_n$. Note that each column of ${\bf U}$ interacts with every column of ${\bf V}$ and every column of ${\bf W}$ in this decomposition, and the strength of this interaction is encoded in the corresponding element of $\Gt$. This is different from the rank decomposition model (CPD) we were discussing until this section, which only allows interactions between corresponding columns of ${\bf A}, {\bf B}, {\bf C}$, i.e., the only outer products that can appear in the CPD are of type ${\bf a}_f \circledcirc  {\bf b}_f \circledcirc  {\bf c}_f$. On the other hand, we emphasize that the {\em Tucker model} in \eqref{eqn:Tucker} also allows ``mixed'' products of non-corresponding columns of ${\bf U}$, ${\bf V}$, ${\bf W}$. Note that {\em any} tensor $\Xt$ can be written in Tucker form \eqref{eqn:Tucker}, and a trivial way of doing so is to take ${\bf U}={\bf I}_{I \times I}$, ${\bf V}={\bf I}_{J \times J}$, ${\bf W}={\bf I}_{K \times K}$, and $\Gt=\Xt$. Hence we may seek a possibly sparse $\Gt$, which could help reveal the underlying ``essential'' interactions between triples of columns of ${\bf U}$, ${\bf V}$, ${\bf W}$. This is sometimes useful when one is interested in quasi-CPD models. The main interest in Tucker though is for finding subspaces and for tensor approximation purposes.

From the above discussion, it may appear that CPD is a special case of the Tucker model, which appears when $\Gt(\ell,m,n)=0$ for all $\ell,m,n$ except possibly for $\ell=m=n$. However, when ${\bf U}$, ${\bf V}$, ${\bf W}$ are all square, such a restricted diagonal Tucker form can only model tensors up to rank $\min(I,J,K)$. If we allow ``fat'' (and therefore, clearly, non-orthogonal) ${\bf U}$, ${\bf V}$, ${\bf W}$ in Tucker though, it is possible to think of CPD as a special case of such a ``blown-up'' non-orthogonal Tucker model.

By a similar token, if we allow column repetition in ${\bf A}$, ${\bf B}$, ${\bf C}$ for CPD, i.e., every column of ${\bf A}$ is repeated $JK$ times, and we call the result ${\bf U}$; every column of ${\bf B}$ is repeated $IK$ times, and we call the result ${\bf V}$; and every column of ${\bf C}$ is repeated $IJ$ times, and we call the result ${\bf W}$, then it is possible to think of non-orthogonal Tucker as a special case of CPD -- but notice that, due to column repetitions, this particular CPD model has k-ranks equal to one in all modes, and is therefore highly non-unique.

In a nutshell, both CPD and Tucker are sum-of-outer-products models, and one can argue that the most general form of one contains the other. What distinguishes the two is uniqueness, which is related but not tantamount to model parsimony (``minimality''); and modes of usage, which are quite different for the two models, as we will see.

\subsection{MLSVD and approximation}

By now the reader must have developed some familiarity with vectorization, and it should be clear that the Tucker model can be equivalently written in various useful ways, such as in vector form as
\[
{\bf x} := \text{vec}(\Xt) = \left( {\bf U} \otimes {\bf V} \otimes {\bf W} \right) {\bf g},
\]
where ${\bf g} := \text{vec}(\Gt)$, and the order of vectorization of $\Xt$ only affects the order in which the factor matrices ${\bf U}$, ${\bf V}$, ${\bf W}$ appear in the Kronecker product chain, and of course the corresponding permutation of the elements of ${\bf g}$. From the properties of the Kronecker product, we know that the expression above is the result of vectorization of matrix
\[
{\bf X}_1 = \left( {\bf V} \otimes {\bf W} \right) {\bf G}_1 {\bf U}^T
\]
where the $KJ \times I$ matrix ${\bf X}_1$ contains all rows (mode-1 vectors) of tensor $\Xt$, and the $KJ \times I$ matrix ${\bf G}_1$ is a likewise reshaped form of the core tensor ${\bf G}$. From this expression it is evident that we can linearly transform the columns of ${\bf U}$ and absorb the inverse transformation in ${\bf G}_1$, i.e.,
\[
{\bf G}_1 {\bf U}^T = {\bf G}_1 {\bf M}^{-T} \left({\bf U} {\bf M}\right)^T,
\]
from which it follows immediately that the Tucker model is not unique. Recalling that ${\bf X}_1$ contains all rows of tensor $\Xt$, and letting $r_1$ denote the row-rank (mode-1 rank) of $\Xt$, it is clear that, without loss of generality, we can pick ${\bf U}$ to be an $I \times r_1$ orthonormal basis of the row-span of $\Xt$, and absorb the linear transformation in $\Gt$, which is thereby reduced from $I \times J \times K$ to $r_1 \times J \times K$. Continuing in this fashion with the other two modes, it follows that, without loss of generality, the Tucker model can be written as
\[
{\bf x} := \text{vec}(\Xt) = \left( {\bf U}_{r_1} \otimes {\bf V}_{r_2} \otimes {\bf W}_{r_3} \right) {\bf g},
\]
where ${\bf U}_{r_1}$ is $I \times r_1$, ${\bf V}_{r_2}$ is $J \times r_2$, ${\bf W}_{r_3}$ is $K \times r_3$, and ${\bf g} := \text{vec}(\Gt)$ is $r_1 r_2 r_3 \times 1$ -- the vectorization of the $r_1 \times r_2 \times r_3$ reduced-size core tensor $\Gt$. This compact-size Tucker model is depicted in Fig. \ref{fig-nikos6}.
\begin{figure}
\vspace{-20pt}
\includegraphics[width=.45\textwidth]{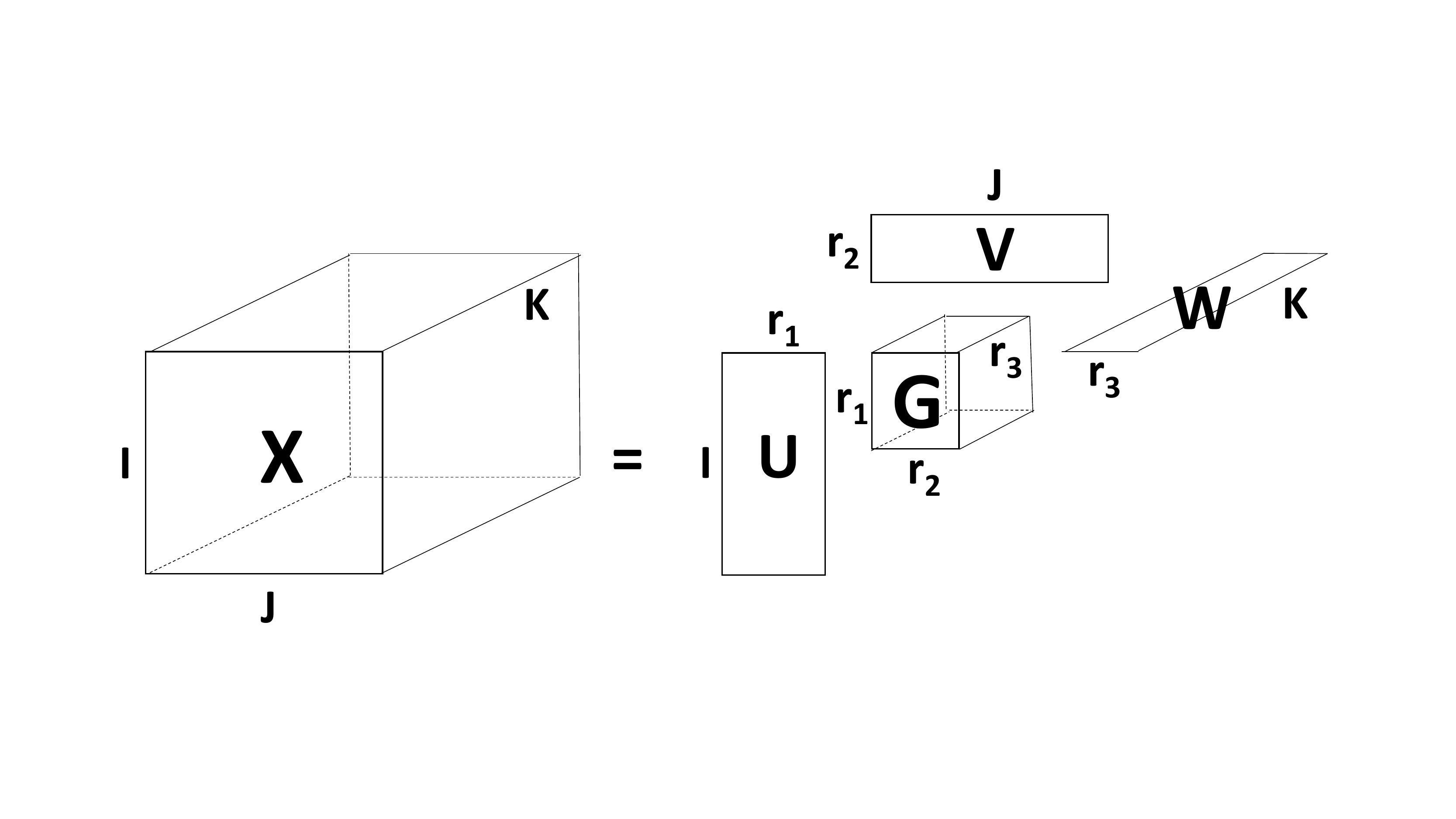}
\vspace{-30pt}
\caption{Compact (reduced) Tucker model: $r_1$, $r_2$, $r_3$ are the mode (row, column, fiber, resp.) ranks of $\Xt$.}\label{fig-nikos6}
\vspace{-10pt}
\end{figure}
Henceforth we drop the subscripts from ${\bf U}_{r_1}$, ${\bf V}_{r_2}$, ${\bf W}_{r_3}$ for brevity -- the meaning will be clear from context. The Tucker model with orthonormal ${\bf U}$, ${\bf V}$, ${\bf W}$ chosen as the right singular vectors of the matrix unfoldings ${\bf X}_1$, ${\bf X}_2$, ${\bf X}_3$, respectively, is also known as the multilinear SVD (MLSVD) (earlier called the higher-order SVD: HOSVD) \cite{doi:10.1137/S0895479896305696}, and it has several interesting and useful properties, as we will soon see.

It is easy to see that orthonormality of the columns of ${\bf U}_{r_1}$, ${\bf V}_{r_2}$, ${\bf W}_{r_3}$ implies orthonormality of the columns of their Kronecker product. This is because $({\bf U}_{r_1} \otimes {\bf V}_{r_2})^T ({\bf U}_{r_1} \otimes {\bf V}_{r_2}) = ({\bf U}_{r_1}^T \otimes {\bf V}_{r_2}^T) ({\bf U}_{r_1} \otimes {\bf V}_{r_2}) = ({\bf U}_{r_1}^T {\bf U}_{r_1}) \otimes ({\bf V}_{r_2}^T {\bf V}_{r_2}) = {\bf I} \otimes {\bf I} = {\bf I}$. Recall that ${\bf x}_1 \perp {\bf x}_2 \Longleftrightarrow {\bf x}_1^T {\bf x}_2 = 0 \Longrightarrow ||{\bf x}_1 + {\bf x}_2||_2^2 = ||{\bf x}_1||_2^2 + ||{\bf x}_2||_2^2$. It follows that
\[
||\Xt||_F^2 := \sum_{\forall ~ i,j,k} |\Xt(i,j,k)|^2 = ||{\bf x}||_2^2 = ||{\bf g}||_2^2 = ||\Gt||_F^2,
\]
where ${\bf x} = \text{vec}(\Xt)$, and ${\bf g} = \text{vec}(\Gt)$. It also follows that, if we drop certain outer products from the decomposition ${\bf x} = \left( {\bf U} \otimes {\bf V} \otimes {\bf W} \right) {\bf g}$, or  equivalently from (\ref{eqn:Tucker}),
i.e., set the corresponding core elements to zero, then, by orthonormality
\[
\left|\left|\Xt - \widehat \Xt\right|\right|_F^2 = \sum_{(\ell,m,n) \in {\cal D}} |\Gt(\ell,m,n)|^2,
\]
where ${\cal D}$ is the set of dropped core element indices. So, if we order the elements of $\Gt$ in order of decreasing magnitude, and discard the ``tail'', then $\widehat \Xt$ will be close to $\Xt$, and we can quantify the error without having to reconstruct $\Xt$, take the difference and evaluate the norm.

In trying to generalize the matrix SVD, we are tempted to consider dropping entire columns of ${\bf U}$, ${\bf V}$, ${\bf W}$. Notice that, for matrix SVD, this corresponds to zeroing out small singular values on the diagonal of matrix $\bSigma$, and per the Eckart--Young theorem, it is optimal in that it yields the best low-rank approximation of the given matrix. Can we do the same for higher-order tensors?

First note that we can permute the slabs of $\Gt$ in any direction, and permute the corresponding columns of ${\bf U}$, ${\bf V}$, ${\bf W}$ accordingly -- this is evident from \eqref{eqn:Tucker}. In this way, we may bring the frontal slab with the highest energy $||\Gt(:,:,n)||_F^2$ up front, then the one with second highest energy, etc. Next, we can likewise order the lateral slabs of the core {\em without changing the energy of the frontal slabs}, and so on -- and in this way, we can compact the energy of the core on its upper-left-front corner. We can then truncate the core, keeping only its upper-left-front dominant part of size $r_1^{'} \times r_2^{'} \times r_3^{'}$, with $r_1^{'} \leq r_1$, $r_2^{'} \leq r_2$, and $r_3^{'} \leq r_3$. The resulting approximation error can be readily bounded as
\begin{eqnarray*}
\left|\left|\Xt - \widehat \Xt\right|\right|_F^2 & \leq & \sum_{\ell=r_1^{'}+1}^{r_1} \left|\left|\Gt(\ell,:,:)\right|\right|_F^2 + \sum_{m=r_2^{'}+1}^{r_2} \left|\left|\Gt(:,m,:)\right|\right|_F^2\\
& & + \sum_{n=r_3^{'}+1}^{r_3} \left|\left|\Gt(:,:,n)\right|\right|_F^2,
\end{eqnarray*}
where we use $\leq$ as opposed to $=$ because dropped elements may be counted up to three times (in particular, the lower-right-back ones). One can of course compute the exact error of such a truncation strategy, but this involves instantiating $\Xt - \widehat \Xt$.

Either way, such truncation in general {\em does not} yield the best approximation of $\Xt$ for the given $(r_1^{'},r_2^{'},r_3^{'})$. That is, there is no exact equivalent of the Eckart--Young theorem for tensors of order higher than two \cite{doi:10.1137/S0895479801394465} -- in fact, as we will see later, the best low multilinear rank approximation problem for tensors is NP-hard. Despite this ``bummer'', much of the beauty of matrix SVD remains in MLSVD, as explained next. In particular, the slabs of the core array $\Gt$ along each mode are orthogonal to each other, i.e., $(\text{vec}(\Gt(\ell,:,:)))^T \text{vec}(\Gt(\ell^{'},:,:))=0$ for $\ell^{'} \neq \ell$, and $\left|\left|\Gt(\ell,:,:)\right|\right|_F$ equals the $\ell$-th singular value of ${\bf X}_1$; and similarly for the other modes (we will actually prove a more general result very soon). These orthogonality and Frobenius norm properties of the Tucker core array generalize a property of matrix SVD: namely, the ``core matrix'' of singular values $\bSigma$ in matrix SVD is diagonal, which implies that its rows are orthogonal to each other, and the same is true for its columns. Diagonality thus implies orthogonality of one-lower-order slabs (sub-tensors of order one less than the original tensor), but the converse is not true, e.g., consider
\[
\left[
  \begin{array}{lr}
    1 & 1 \\
    1 & -1 \\
  \end{array}
\right].
\]
We have seen that diagonality of the core is not possible in general for higher-order tensors, because it severely limits the degrees of freedom; but all-orthogonality of one-lower-order slabs of the core array, and the interpretation of their Frobenius norms as singular values of a certain matrix view of the tensor come without loss of generality (or optimality, as we will see in the proof of the next property). This intuitively pleasing result was pointed out by De Lathauwer \cite{doi:10.1137/S0895479896305696}, and it largely motivates the analogy to matrix SVD -- albeit simply truncating slabs (or elements) of the full core will not give the best low multilinear rank approximation of $\Xt$ in the case of three- and higher-order tensors. The error bound above is actually the proper generalization of the Eckart--Young theorem. In the matrix case, because of diagonality there is only one summation and equality instead of inequality.

Simply truncating the MLSVD at sufficiently high $(r_1^{'},r_2^{'},r_3^{'})$ is often enough to obtain a good approximation in practice -- we may control the error as we wish, so long as we pick high enough $(r_1^{'},r_2^{'},r_3^{'})$. The error $||\Xt - \widehat \Xt||_F^2$ is in fact at most 3 times higher than the minimal error ($N$ times higher in the $N$-th order case) \cite{grasedyck2013literature,hackbusch2012tensor}. If we are interested in the best possible approximation of $\Xt$ with mode ranks $(r_1^{'},r_2^{'},r_3^{'})$, however, then we need to consider the following, after dropping the $^{'}$s for brevity:
\begin{property} \label{pptuc} \cite{doi:10.1137/S0895479896305696,doi:10.1137/S0895479898346995}
Let $\left(\widehat{\bf U}, \widehat{\bf V}, \widehat{\bf W}, \widehat{\bf G}_1\right)$ be a solution to
\begin{eqnarray*}
  \min_{\left({\bf U}, {\bf V}, {\bf W}, {\bf G}_1\right)} && ||{\bf X}_1 - ({\bf V} \otimes {\bf W}) {\bf G}_1 {\bf U}^T||_F^2 \\
  \text{such that:} && {\bf U}: I \times r_1,~r_1 \leq I,~{\bf U}^T{\bf U}={\bf I}\\
  && {\bf V}: J \times r_2,~r_2 \leq J,~{\bf V}^T{\bf V}={\bf I}\\
  && {\bf W}: K \times r_3,~r_3 \leq K,~{\bf W}^T{\bf W}={\bf I}\\
  && {\bf G}_1: r_3 r_2 \times r_1
\end{eqnarray*}
Then
\begin{itemize}
\item $\widehat {\bf G}_1 = (\widehat {\bf V} \otimes \widehat {\bf W})^T {\bf X}_1 \widehat {\bf U}$;
\item Substituting the conditionally optimal ${\bf G}_1$, the problem can be recast in ``concentrated'' form as
\begin{eqnarray*}
  \max_{\left({\bf U}, {\bf V}, {\bf W} \right)} && ||({\bf V} \otimes {\bf W})^T {\bf X}_1 {\bf U}||_F^2 \\
  \text{such that:} && {\bf U}: I \times r_1,~r_1 \leq I,~{\bf U}^T{\bf U}={\bf I}\\
  && {\bf V}: J \times r_2,~r_2 \leq J,~{\bf V}^T{\bf V}={\bf I}\\
  && {\bf W}: K \times r_3,~r_3 \leq K,~{\bf W}^T{\bf W}={\bf I}
\end{eqnarray*}
\item $\widehat {\bf U}$ $=$ dominant $r_1$-dim. right subspace of $(\widehat {\bf V} \otimes \widehat {\bf W})^T {\bf X}_1$;
\item $\widehat {\bf V}$ $=$ dominant $r_2$-dim. right subspace of $(\widehat {\bf U} \otimes \widehat {\bf W})^T {\bf X}_2$;
\item $\widehat {\bf W}$ $=$ dominant $r_3$-dim. right subspace of $(\widehat {\bf U} \otimes \widehat {\bf V})^T {\bf X}_3$;
\item $\widehat{\bf G}_1$ has orthogonal columns; and
\item $\left\{||\widehat {\bf G}_1(:,m)||_2^2\right\}_{m=1}^{r_1}$ are the $r_1$ principal singular values of $(\widehat {\bf V} \otimes \widehat {\bf W})^T {\bf X}_1$. Note that each column of $\widehat{\bf G}_1$ is a vectorized slab of the core array $\widehat \Gt$ obtained by fixing the first reduced dimension index to some value.
\end{itemize}
\end{property}
\begin{proof}
Note that $||\text{vec}({\bf X}_1) - \left( {\bf U} \otimes {\bf V} \otimes {\bf W} \right) \text{vec}({\bf G}_1)||_2^2 = ||{\bf X}_1 - ({\bf V} \otimes {\bf W}) {\bf G}_1 {\bf U}^T||_F^2$, so conditioned on (orthonormal) ${\bf U}$, ${\bf V}$, ${\bf W}$ the optimal $\Gt$ is given by $\text{vec}(\widehat {\bf G}_1) = \left( {\bf U} \otimes {\bf V} \otimes {\bf W} \right)^T \text{vec}({\bf X}_1)$, and therefore $\widehat {\bf G}_1 = ({\bf V} \otimes {\bf W})^T {\bf X}_1 {\bf U}$.

Consider $||{\bf X}_1 - ({\bf V} \otimes {\bf W}) {\bf G}_1 {\bf U}^T||_F^2$, define $\widetilde {\bf X}_1 := ({\bf V} \otimes {\bf W}) {\bf G}_1 {\bf U}^T$, and use that $||{\bf X}_1 - \widetilde {\bf X}_1||_F^2 = \text{Tr}(({\bf X}_1 - \widetilde {\bf X}_1)^T({\bf X}_1 - \widetilde {\bf X}_1)) = ||{\bf X}_1||_F^2 + ||\widetilde {\bf X}_1||_F^2 - 2 \text{Tr}({\bf X}_1^T \widetilde {\bf X}_1)$. By orthonormality of ${\bf U}$, ${\bf V}$, ${\bf W}$, it follows that $||\widetilde {\bf X}_1||_F^2= ||{\bf G}_1||_F^2$. Now, consider
\[
- 2 \text{Tr}({\bf X}_1^T \widetilde {\bf X}_1) = - 2 \text{Tr}({\bf X}_1^T ({\bf V} \otimes {\bf W}) {\bf G}_1 {\bf U}^T),
\]
and substitute ${\bf G}_1 = ({\bf V} \otimes {\bf W})^T {\bf X}_1 {\bf U}$ to obtain
\[
- 2 \text{Tr}({\bf X}_1^T ({\bf V} \otimes {\bf W}) ({\bf V} \otimes {\bf W})^T {\bf X}_1 {\bf U} {\bf U}^T).
\]
Using a property of the trace operator to bring the rightmost matrix to the left, we obtain
\[
- 2 \text{Tr}({\bf U}^T {\bf X}_1^T ({\bf V} \otimes {\bf W}) ({\bf V} \otimes {\bf W})^T {\bf X}_1 {\bf U}) =
\]
\[
- 2 \text{Tr}({\bf G}_1^T {\bf G}_1) = -2 ||{\bf G}_1||_F^2.
\]
It follows that $||{\bf X}_1 - ({\bf V} \otimes {\bf W}) {\bf G}_1 {\bf U}^T||_F^2 = ||{\bf X}_1||_F^2 - ||{\bf G}_1||_F^2$,
so we may equivalently maximize $||{\bf G}_1||_F^2=||({\bf V} \otimes {\bf W})^T {\bf X}_1 {\bf U}||_F^2$. From this, it immediately follows that $\widehat {\bf U}$ is the dominant right subspace of $(\widehat {\bf V} \otimes \widehat {\bf W})^T {\bf X}_1$, so we can take it to be the $r_1$ principal right singular vectors of $(\widehat {\bf V} \otimes \widehat {\bf W})^T {\bf X}_1$. The respective results for $\widehat {\bf V}$ and $\widehat {\bf W}$ are obtained by appealing to role symmetry. Next, we show that $\widehat{\bf G}_1$ has orthogonal columns. To see this, let $\widehat{\bf G}_1 = [\widehat {\bf g}_{1,1},\cdots,\widehat {\bf g}_{1,r_1}]$, and $\widehat{\bf U} = [\widehat {\bf u}_{1},\cdots,\widehat {\bf u}_{r_1}]$. Consider
\[
\widehat {\bf g}_{1,m_1}^T \widehat {\bf g}_{1,m_2} = \widehat {\bf u}_{m_1}^T {\bf X}_1^T (\widehat {\bf V} \otimes \widehat {\bf W}) (\widehat {\bf V} \otimes \widehat {\bf W})^T {\bf X}_1 \widehat {\bf u}_{m_2}.
\]
Let ${\bGamma} \bSigma \tilde {\bf U}^T$ be the SVD of  $(\widehat {\bf V} \otimes \widehat {\bf W})^T {\bf X}_1$. Then $\tilde {\bf U} = [\widehat {\bf U},~\check {\bf U}]$, so
\[
(\widehat {\bf V} \otimes \widehat {\bf W})^T {\bf X}_1 {\bf u}_{m_2} = {\bgamma}_{m_2} \sigma_{m_2},
\]
with obvious notation; and therefore
\[
\widehat {\bf g}_{1,m_1}^T \widehat {\bf g}_{1,m_2} = \sigma_{m_1} \sigma_{m_2} {\bgamma}_{m_1}^T {\bgamma}_{m_2} = \sigma_{m_1} \sigma_{m_2} \delta(m_1-m_2),
\]
by virtue of orthonormality of left singular vectors of $(\widehat {\bf V} \otimes \widehat {\bf W})^T {\bf X}_1$ (here $\delta(\cdot)$ is the Kronecker delta). By role symmetry, it follows that the slabs of $\widehat {\bf G}$ along any mode are likewise orthogonal. It is worth mentioning that, as a byproduct of the last equation, $||\widehat \Gt(:,:,m)||_F^2 = ||\widehat {\bf G}_1(:,m)||_2^2 = ||\widehat {\bf g}_{1,m}||_2^2 = \sigma_m^2$; that is, {\em the Frobenius norms of the lateral core slabs are the $r_1$ principal singular values of} $(\widehat {\bf V} \otimes \widehat {\bf W})^T {\bf X}_1$.
\end{proof}
The best rank-1 tensor approximation problem over $\Real$ is NP-hard \cite[Theorem 1.13]{hillar2013most}, so the best low multilinear rank approximation problem is also NP-hard (the best multilinear rank approximation with $(r_1,r_2,r_3)=(1,1,1)$ is the best rank-1 approximation). This is reflected in a key limitation of the characterization in Property \ref{pptuc}, which gives explicit expressions that relate the sought ${\bf U}$, ${\bf V}$, ${\bf W}$, and $\Gt$, but it does not provide an explicit solution for any of them. On the other hand, Property \ref{pptuc} naturally suggests the following alternating least squares scheme:
\begin{center}
{\bf $\perp$-Tucker ALS}
\end{center}
\begin{enumerate}
\item Initialize:
\begin{itemize}
\item ${\bf U}$ $=$ $r_1$ principal right singular vectors of ${\bf X}_1$;
\item ${\bf V}$ $=$ $r_2$ principal right singular vectors of ${\bf X}_2$;
\item ${\bf W}$ $=$ $r_3$ principal right singular vectors of ${\bf X}_3$;
\end{itemize}
\item repeat:
\begin{itemize}
\item ${\bf U}$ $=$ $r_1$ principal right sing. vec. of $({\bf V} \otimes {\bf W})^T {\bf X}_1$;
\item ${\bf V}$ $=$ $r_2$ principal right sing. vec. of $({\bf U} \otimes {\bf W})^T {\bf X}_2$;
\item ${\bf W}$ $=$ $r_3$ principal right sing. vec. of $({\bf U} \otimes {\bf V})^T {\bf X}_3$;
\item until negligible change in $||({\bf V} \otimes {\bf W})^T {\bf X}_1 {\bf U}||_F^2$.
\end{itemize}
\item ${\bf G}_1 = ({\bf V} \otimes {\bf W})^T {\bf X}_1 {\bf U}$.
\end{enumerate}
The initialization in step 1) [together with step 3)] corresponds to (truncated) MLSVD. It is not necessarily optimal, as previously noted, but it does help as a very good initialization in most cases. The other point worth noting is that each variable update is optimal conditioned on the rest of the variables, so the reward $||({\bf V} \otimes {\bf W})^T {\bf X}_1 {\bf U}||_F^2$ is non-decreasing (equivalently, the cost $||{\bf X}_1 - ({\bf V} \otimes {\bf W}) {\bf G}_1 {\bf U}^T||_F^2$ is non-increasing) and bounded from above (resp. below), thus convergence of the reward (cost) sequence is guaranteed. Note the conceptual similarity of the above algorithm with ALS for CPD, which we discussed earlier. The first variant of Tucker-ALS goes back to the work of Kroonenberg and De Leeuw; see \cite{Kroonenberg} and references therein.

Note that using MLSVD with somewhat higher $(r_1,r_2,r_3)$ can be computationally preferable to ALS. In the case of big data, even the computation of MLSVD may be prohibitive, and randomized projection approaches become more appealing \cite{SidKyrSPL2012,PARACOMP}. A drastically different approach is to draw the columns of ${\bf U}$, $ {\bf V}$, ${\bf W}$ from the columns, rows, fibers of $\Xt$ \cite{oseledets2008tucker,mahoney2008tensor,caiafa2010generalizing}. Although the idea is simple, it has sound algebraic foundations and error bounds are available \cite{goreinov1997theory}. Finally, we mention that for large-scale matrix problems Krylov subspace methods are one of the main classes of algorithms. They only need an implementation of the matrix-vector product to iteratively find subspaces on which to project. See \cite{savas2013krylov} for Tucker-type extensions.

\subsection{Compression as preprocessing}

Consider a tensor $\Xt$ in vectorized form, and corresponding CPD and orthogonal Tucker ($\perp$-Tucker) models
\[
{\bf x} = ({\bf A} \odot {\bf B} \odot {\bf C}) {\bf 1} = ({\bf U} \otimes {\bf V} \otimes {\bf W}) {\bf g}.
\]
Pre-multiplying with $({\bf U} \otimes {\bf V} \otimes {\bf W})^T = ({\bf U}^T \otimes {\bf V}^T \otimes {\bf W}^T)$ and using the mixed-product rule for $\otimes$, $\odot$, we obtain
\[
{\bf g} = \left( ({\bf U}^T {\bf A}) \odot ({\bf V}^T {\bf B}) \odot ({\bf W}^T {\bf C}) \right) {\bf 1},
\]
i.e., the Tucker core array $\Gt$ (shown above in vectorized form ${\bf g}$) admits a CPD decomposition of $\text{rank}(\Gt) \leq \text{rank}(\Xt)$. Let $\left\llbracket\tilde {\bf A}, \tilde {\bf B}, \tilde {\bf C}\right\rrbracket$ be a CPD of $\Gt$, i.e., ${\bf g} = (\tilde {\bf A} \odot \tilde {\bf B} \odot \tilde {\bf C}) {\bf 1}$. Then
\[
{\bf x} = ({\bf U} \otimes {\bf V} \otimes {\bf W}) {\bf g} = ({\bf U} \otimes {\bf V} \otimes {\bf W}) (\tilde {\bf A} \odot \tilde {\bf B} \odot \tilde {\bf C}) {\bf 1} =
\]
\[
= \left( ({\bf U} \tilde {\bf A}) \odot ({\bf V} \tilde {\bf B}) \odot ({\bf W} \tilde {\bf C}) \right) {\bf 1},
\]
by the mixed product rule. Assuming that the CPD of $\Xt$ is essentially unique, it then follows that
\[
{\bf A} = {\bf U} \tilde {\bf A} \bPi \bLambda_a,~{\bf B} = {\bf V} \tilde {\bf B} \bPi \bLambda_b,~{\bf C} = {\bf W} \tilde {\bf C} \bPi \bLambda_c,
\]
where $\bPi$ is a permutation matrix and $~\bLambda_a \bLambda_b \bLambda_c = {\bf I}$. It follows that
\[
{\bf U}^T {\bf A} = \tilde {\bf A} \bPi \bLambda_a,~{\bf V}^T {\bf B} = \tilde {\bf B} \bPi \bLambda_b,~ {\bf W}^T {\bf C} = \tilde {\bf C} \bPi \bLambda_c,
\]
so that the CPD of $\Gt$ is essentially unique, and therefore $\text{rank}(\Gt) = \text{rank}(\Xt)$.

Since the size of $\Gt$ is smaller than or equal to the size of $\Xt$, this suggests that an attractive way to compute the CPD of $\Xt$ is to first compress (using one of the orthogonal schemes in the previous subsection), compute the CPD of $\Gt$, and then ``blow-up'' the resulting factors, since ${\bf A} = {\bf U} \tilde {\bf A}$ (up to column permutation and scaling). It also shows that ${\bf A} = {\bf U} {\bf U}^T {\bf A}$, and likewise for the other two modes. The caveat is that the discussion above assumes {\em exact} CPD and $\perp$-Tucker models, whereas in reality we are interested in low-rank least-squares approximation -- for this, we refer the reader to the {\em Candelinc} theorem of Carroll {\em et al.} \cite{DouglasCarroll1980}; see also Bro \& Andersson \cite{Bro1998105}.

This does not work for a constrained CPD (e.g. one or more factor matrices nonnegative, monotonic, sparse, \ldots) since the orthogonal compression destroys the constraints. In the ALS approach we can still exploit multi-linearity, however, to update ${\bf U}$ by solving a {\em constrained} and/or regularized linear least squares problem, and similarly for ${\bf V}$ and ${\bf W}$, by role symmetry. For $\Gt$, we can use the vectorization property of the Kronecker product to bring it to the right, and then use a constrained or regularized linear least squares solver. By the mixed product rule, this last step entails pseudo-inversion of the ${\bf U}$, ${\bf V}$, ${\bf W}$ matrices, instead of their (much larger) Kronecker product. This type of model is sometimes called {\em oblique} Tucker, to distinguish from {\em orthogonal} Tucker. More generally than in ALS (see the algorithms in Section \ref{sect:alg}), one can fit the constrained CPD in the uncompressed space, but with $\Xt$ replaced by its parameter-efficient factorized representation. The structure of the latter may then be exploited to reduce the per iteration complexity \cite{tensorlab3asilomar}.

\section{Other decompositions}

\subsection{Compression}

In Section \ref{sect:MLSVD} we have emphasized the use of $\perp$-Tucker/MLSVD for tensor approximation and compression. This use was in fact limited to tensors of moderate order. Let us consider the situation at order $N$ and let us assume for simplicity that $r_1 = r_2 = \ldots = r_N = r > 1$. Then the core tensor ${\bf G}$ has $r^N$ entries. The exponential dependence of the number of entries on the tensor order $N$ is called the Curse of Dimensionality: in the case of large $N$ (e.g. $N=100$), $r^N$ is large, even when $r$ is small, and as a result $\perp$-Tucker/MLSVD cannot be used. In such cases one may resort to a Tensor Train (TT) representation or a hierarchical Tucker (hTucker) decomposition instead \cite{oseledets2011tensor,grasedyck2013literature}. A TT of an $N$-th order tensor ${\bf X}$ is of the form
\begin{equation}
{\bf X}(i_1,i_2, \ldots, i_N) = \sum_{r_1 r_2 \ldots r_{N-1}} u_{i_1 r_1}^{(1)} u_{r_1 i_2 r_2}^{(2)} u_{r_2 i_3 r_3}^{(3)} \ldots u_{i_N r_{N-1}}^{(N)},
\end{equation}
in which one can see ${\bf U}^{(1)}$ as the locomotive and the next factors as the carriages.
Note that each carriage ``transports'' one tensor dimension, and that two consecutive carriages are connected through the summation over one common index. Since every index appears at most twice and since there are no index cycles, the TT-format is ``matrix-like'', i.e. a TT approximation can be computed using established techniques from numerical {\em linear} algebra, similarly to MLSVD. Like for MLSVD, fiber sampling schemes have been developed too. On the other hand, the number of entries is now $O(NIr^2)$, so the Curse of Dimensionality has been broken. hTucker is the extension in which the indices are organized in a binary tree.

\subsection{Analysis}

In Section \ref{sect:uniq} we have emphasized the uniqueness of CPD under mild conditions as a profound advantage of tensors over matrices in the context of signal separation and data analysis -- constraints such as orthogonality or triangularity are not necessary per se. An even more profound advantage is the possibility to have a unique decomposition in terms that are not even rank-1. Block Term Decompositions (BTD) write a tensor as a sum of terms that have low multilinear rank, i.e. the terms can be pictured as in Fig. \ref{fig-nikos6} rather than as in Fig. \ref{fig-nikos1} \cite{de2008decompositions,de2011blind}. Note that rank-1 structure of data components is indeed an assumption that needs to be justified.

As in CPD, uniqueness of a BTD is up to a permutation of the terms. The scaling/counterscaling ambiguities within a rank-1 term generalize to the indeterminacies in a Tucker representation. Expanding the block terms into sums of rank-1 terms with repeated vectors as in (\ref{eqn:Tucker}) yields a form that is known as PARALIND \cite{bro2009modeling}; see also more recent results in \cite{doi:10.1137/080743354,doi:10.1137/110847275,doi:10.1137/110825765}.

\subsection{Fusion} \label{subsect:fusion}

Multiple data sets may be jointly analyzed by means of coupled decompositions of several matrices and/or tensors, possibly of different size
\cite{SDF,lahat2015multimodal}. An early variant, in which coupling was imposed through a shared covariance matrix, is Harshman's PARAFAC2 \cite{harshman1972parafac2}. In a coupled setting, particular decompositions may inherit uniqueness from other decompositions; in particular, the decomposition of a data matrix may become unique thanks to coupling \cite{sorensen2015coupled}.

\section{Algorithms}  \label{sect:alg}

\subsection{ALS: Computational aspects}

\subsubsection{CPD}
We now return to the basic ALS algorithm for CPD, to discuss (efficient) computation issues. First note that the pseudo-inverse that comes into play in ALS updates is structured: in updating ${\bf C}$, for example
\[
{\bf C} \leftarrow \text{arg} \min_{{\bf C}} ||{\bf X}_3 - ({\bf B} \odot {\bf A}){\bf C}^T||_F^2,
\]
the pseudo-inverse
\[
({\bf B} \odot {\bf A})^{\dagger} = \left[({\bf B} \odot {\bf A})^T ({\bf B} \odot {\bf A})\right]^{-1} ({\bf B} \odot {\bf A})^T,
\]
can be simplified. In particular,
\[
({\bf B} \odot {\bf A})^T ({\bf B} \odot {\bf A}) =
\left[
  \begin{array}{c}
    {\bf A} {\bf D}_1({\bf B}) \\
    \vdots \\
    {\bf A} {\bf D}_J({\bf B})
  \end{array}
\right]^T
\left[
  \begin{array}{c}
    {\bf A} {\bf D}_1({\bf B}) \\
    \vdots \\
    {\bf A} {\bf D}_J({\bf B})
  \end{array}
\right]
\]
\[
= \sum_{j=1}^J {\bf D}_j({\bf B}) {\bf A}^T {\bf A} {\bf D}_j({\bf B}),
\]
where we note that the result is $F \times F$, and element $(f_1,f_2)$ of the result is element $(f_1,f_2)$ of ${\bf A}^T {\bf A}$ times $\sum_{j=1}^J {\bf B}(j,f_1) {\bf B}(j,f_2)$. The latter is element $(f_1,f_2)$ of ${\bf B}^T {\bf B}$. It follows that
\[
({\bf B} \odot {\bf A})^T ({\bf B} \odot {\bf A}) = ({\bf B}^T {\bf B}) * ({\bf A}^T {\bf A}),
\]
which only involves the Hadamard product of $F \times F$ matrices, and is easy to invert for small ranks $F$ (but note that in the case of big sparse data, small $F$ may not be enough). Thus the update of ${\bf C}$ can be performed as
\[
{\bf C}^T \leftarrow \left(({\bf B}^T {\bf B}) * ({\bf A}^T {\bf A})\right)^{-1} ({\bf B} \odot {\bf A})^T {\bf X}_3.
\]
For small $F$, the bottleneck of this is actually the computation of $({\bf B} \odot {\bf A})^T {\bf X}_3$ -- notice that
${\bf B} \odot {\bf A}$ is $IJ \times F$, and ${\bf X}_3$ is $IJ \times K$. Brute-force computation of $({\bf B} \odot {\bf A})^T {\bf X}_3$ thus demands $IJF$ additional memory and flops to instantiate ${\bf B} \odot {\bf A}$, even though the result is only $F \times K$, and $IJKF$ flops to actually compute the product -- but see \cite{BadKol2007,NV-KM-RV,AHP-PT-AC}. If $\Xt$ (and therefore ${\bf X}_3$) is sparse, having $\text{NNZ}(\Xt)$ nonzero elements stored in a  $[\text{i,j,k,value}]$ list, then every nonzero element multiplies a column of $({\bf B} \odot {\bf A})^T$, and the result should be added to column $k$. The specific column needed can be generated on-the-fly with $F+1$ flops, for an overall complexity of $(2F+1) \text{NNZ}(\Xt)$, without requiring any additional memory (other than that needed to store the running estimates of ${\bf A}$, ${\bf B}$, ${\bf C}$, and the data $\Xt$). When $\Xt$ is dense, the number of flops is inevitably of order $IJKF$, but still no additional memory is needed this way. Furthermore, the computation can be parallelized in several ways -- see \cite{BadKol2007,kang2012gigatensor,niranjay,choi2014dfacto,smith2015splatt} for various resource-efficient algorithms for {\em matricized tensor times Khatri--Rao product} (MTTKRP) computations.

\subsubsection{Tucker}
For $\perp$-Tucker ALS, we need to compute products of type $({\bf V} \otimes {\bf W})^T {\bf X}_1$ (and then compute the principal right singular vectors of the resulting $r_2 r_3 \times I$ matrix). The column-generation idea can be used here as well to avoid intermediate memory explosion and exploit sparsity in $\Xt$ when computing $({\bf V} \otimes {\bf W})^T {\bf X}_1$.

For oblique Tucker ALS we need to compute $\left(({\bf V} \otimes {\bf W}) {\bf G}_1\right)^{\dagger} {\bf X}_1$ for updating ${\bf U}$, and $\left( {\bf U}^{\dagger} \otimes {\bf V}^{\dagger} \otimes {\bf W}^{\dagger} \right) {\bf x}$ for updating ${\bf g} \leftrightarrow \Gt$. The latter requires pseudo-inverses of relatively small matrices, but note that
\[
\left(({\bf V} \otimes {\bf W}) {\bf G}_1\right)^{\dagger} \neq {\bf G}_1^{\dagger} \left({\bf V} \otimes {\bf W} \right)^{\dagger},
\]
in general. Equality holds if ${\bf V} \otimes {\bf W}$ is full column rank {\em and} ${\bf G}_1$ is full row rank, which requires $r_2 r_3 \leq r_1$.

ALS is a special case of {\em block coordinate descent} (BCD), in which the subproblems take the form of linear LS estimation. As the musings in \cite{M} make clear, understanding the convergence properties of ALS is highly nontrivial. ALS monotonically reduces the cost function, but it is not guaranteed to converge to a stationary point. A conceptually easy fix is to choose for the next update the parameter block that decreases the cost function the most -- this {\em maximum block improvement} (MBI) variant is guaranteed to converge under some conditions \cite{LZ-AU-SZ}. However, in the case of third-order CPD MBI doubles the computation time as two possible updates have to be compared. At order $N$, the computation time increases by a factor $N-1$ -- and in practice there is usually little difference between MBI and plain ALS. Another way to ensure convergence of ALS is to include proximal regularization terms and invoke the {\em block successive upper bound minimization} (BSUM) framework of \cite{doi:10.1137/120891009}, which also helps in ill-conditioned cases. In cases where ALS converges, it does so at a local linear rate (under some non-degeneracy condition), which makes it (locally) slower than some derivative-based algorithms \cite{AU,ME-WH-AK}, see further. The same is true for MBI \cite{LZ-AU-SZ}.

\subsection{Gradient descent}
Consider the squared loss
\[
{\cal L}({\bf A}, {\bf B}, {\bf C}) := ||{\bf X}_1 - ({\bf C} \odot {\bf B}) {\bf A}^T||_F^2 =
\]
\[
\text{tr}\left( ({\bf X}_1 - ({\bf C} \odot {\bf B}) {\bf A}^T)^T ({\bf X}_1 - ({\bf C} \odot {\bf B}) {\bf A}^T) \right) =
||{\bf X}_1||_F^2 -
\]
\[
2~\text{tr}\left({\bf X}_1^T ({\bf C} \odot {\bf B}) {\bf A}^T\right) + \text{tr}\left({\bf A} ({\bf C} \odot {\bf B})^T ({\bf C} \odot {\bf B}) {\bf A}^T\right).
\]
Recall that $({\bf C} \odot {\bf B})^T ({\bf C} \odot {\bf B}) = ({\bf C}^T {\bf C}) * ({\bf B}^T {\bf B})$, so we may equivalently take the gradient of $-2~\text{tr}\left({\bf X}_1^T ({\bf C} \odot {\bf B}) {\bf A}^T\right) + \text{tr}\left({\bf A} ({\bf C}^T {\bf C}) * ({\bf B}^T {\bf B}) {\bf A}^T\right)$. Arranging the gradient in the same format\footnote{In some books, $\frac{\partial f({\bf A})}{\partial {\bf A}}$ stands for the transpose of what we denote by $\frac{\partial f({\bf A})}{\partial {\bf A}}$, i.e., for an $F \times I$ matrix instead of $I \times F$ in our case.} as ${\bf A}$, we have
\begin{eqnarray*}
\frac{\partial {\cal L}({\bf A}, {\bf B}, {\bf C})}{\partial {\bf A}} &=& -2 {\bf X}_1^T ({\bf C} \odot {\bf B}) + 2~{\bf A} \left[ ({\bf C}^T {\bf C}) * ({\bf B}^T {\bf B}) \right]\\
&=& -2  \left({\bf X}_1^T - {\bf A} ({\bf C} \odot {\bf B})^T\right) ({\bf C} \odot {\bf B}),
\end{eqnarray*}
Appealing to role symmetry, we likewise obtain
\begin{eqnarray*}
\frac{\partial {\cal L}({\bf A}, {\bf B}, {\bf C})}{\partial {\bf B}} &=& -2 {\bf X}_2^T ({\bf C} \odot {\bf A}) + 2~{\bf B} \left[ ({\bf C}^T {\bf C}) * ({\bf A}^T {\bf A}) \right]\\
&=& -2  \left({\bf X}_2^T - {\bf B} ({\bf C} \odot {\bf A})^T\right) ({\bf C} \odot {\bf A}),\\
\frac{\partial {\cal L}({\bf A}, {\bf B}, {\bf C})}{\partial {\bf C}} &=& -2 {\bf X}_3^T ({\bf B} \odot {\bf A}) + 2~{\bf C} \left[ ({\bf B}^T {\bf B}) * ({\bf A}^T {\bf A}) \right]\\
&=& -2  \left({\bf X}_3^T - {\bf C} ({\bf B} \odot {\bf A})^T\right) ({\bf B} \odot {\bf A}).
\end{eqnarray*}
\begin{remark}
The conditional least squares update for ${\bf A}$ is
\[
{\bf A} \leftarrow \left[ ({\bf C}^T {\bf C}) * ({\bf B}^T {\bf B}) \right]^{-1} ({\bf C} \odot {\bf B})^T {\bf X}_1,
\]
So taking a gradient step or solving the least-squares sub-problem to (conditional) optimality involves computing the same quantities: $({\bf C}^T {\bf C}) * ({\bf B}^T {\bf B})$ and $({\bf C} \odot {\bf B})^T {\bf X}_1$. The only difference is that to take a gradient step you don't need to invert the $F \times F$ matrix $({\bf C}^T {\bf C}) * ({\bf B}^T {\bf B})$. For small $F$, this inversion has negligible cost relative to the computation of the MTTKRP $({\bf C} \odot {\bf B})^T {\bf X}_1$. Efficient algorithms for the MTTKRP can be used for gradient computations as well; but note that, for small $F$, each gradient step is essentially as expensive as an ALS step. Also note that, whereas it appears that keeping three different matricized copies of $\Xt$ is necessary for efficient gradient (and ALS) computations, only one is needed -- see \cite{niranjay,LS-MVB-LDL}.
\end{remark}
With these gradient expressions at hand, we can employ any gradient-based algorithm for model fitting.

\subsection{Quasi-Newton and Nonlinear Least Squares}

The well-known Newton descent algorithm uses a local quadratic approximation of the cost function ${\cal L}({\bf A}, {\bf B}, {\bf C})$ to obtain a new step as the solution of the set of linear equations
\begin{equation}
\mathbf{H} \mathbf{p} = -\mathbf{g},
\label{eq:Newtonstep}
\end{equation}
in which $\mathbf{g}$ and $\mathbf{H}$ are the gradient and Hessian of ${\cal L}$, respectively. As computation of the Hessian may be prohibitively expensive, one may resort to an approximation, leading to quasi-Newton and Nonlinear Least Squares (NLS). Quasi-Newton methods such as Nonlinear Conjugate Gradients (NCG) and (limited memory) BFGS use a diagonal plus low-rank matrix approximation of the Hessian. In combination with line search or trust region globalization strategies for step size selection, quasi-Newton does guarantee convergence to a stationary point, contrary to plain ALS, and its convergence is superlinear \cite{LS-MVB-LDL,nocedal2006numerical}.

NLS methods such as Gauss--Newton and Levenberg--Marquardt start from a local linear approximation of the residual ${\bf X}_1 - ({\bf C} \odot {\bf B}){\bf A}^T$ to approximate the Hessian as $\jacob_{\th}\ph(\th)^T \jacob_{\th}\ph(\th)$, with  $\jacob_{\th}\ph(\th)$ the Jacobian matrix of $\ph(\th)$ (where $\th$ is the parameter vector; see section \ref{sec:CRB} for definitions of $\ph(\th)$, and $\jacob_{\th}\ph(\th)$). The algebraic structure of $\jacob_{\th}\ph(\th)^T \jacob_{\th}\ph(\th)$ can be exploited to obtain a fast inexact NLS algorithm that has several favorable properties \cite{tomasi2006comparison,LS-MVB-LDL}. Briefly, the inexact NLS algorithm uses a  ``parallel version'' of one ALS iteration as a preconditioner for solving the linear system of equations (\ref{eq:Newtonstep}). (In this parallel version the factor matrices are updated all together starting from the estimates in the previous iteration; note that the preconditioning can hence be parallelized.) After preconditioning, (\ref{eq:Newtonstep}) is solved inexactly by a truncated conjugate gradient algorithm. That is, the set of equations is not solved exactly and neither is the matrix $\jacob_{\th}\ph(\th)^T \jacob_{\th}\ph(\th)$ computed or stored. Storage of ${\bf A}$, ${\bf B}$, ${\bf C}$ and ${\bf A}^T {\bf A}$, ${\bf B}^T {\bf B}$, ${\bf C}^T {\bf C}$ suffices for an efficient computation of the product of a vector with $\jacob_{\th}\ph(\th)^T \jacob_{\th}\ph(\th)$, exploiting the structure of the latter, and an approximate solution of (\ref{eq:Newtonstep}) is obtained by a few such matrix-vector products. As a result, the conjugate gradient refinement adds little to the memory and computational cost, while it does yield the nice NLS-type convergence behavior. The algorithm has close to quadratic convergence, especially when the residuals are small.
NLS has been observed to be more robust for difficult decompositions than plain ALS \cite{LS-MVB-LDL,tomasi2006comparison}. The action of $\jacob_{\th}\ph(\th)^T \jacob_{\th}\ph(\th)$ can easily be split into smaller matrix-vector products ($N^2$ in the $N$-th order case), which makes inexact NLS  overall well-suited for parallel implementation. Variants for low multilinear rank approximation are discussed in \cite{ishteva2011best,savas2010quasi} and references therein.

\subsection{Exact line search}

An important issue in numerical optimization is the choice of step-size. One approach that is sometimes used in multi-way analysis is the following \cite{doi:10.1137/06065577}, which exploits the multi-linearity of the cost function. Suppose we have determined an update (``search'') direction, say the negative gradient one. We seek to select the optimal step-size $\mu$ for the update
\[
\left[
  \begin{array}{c}
    {\bf A} \\
    {\bf B} \\
    {\bf C} \\
  \end{array}
\right]
\leftarrow
\left[
  \begin{array}{c}
    {\bf A} \\
    {\bf B} \\
    {\bf C} \\
  \end{array}
\right]
+ \mu \left[
  \begin{array}{c}
    \bm{\Delta}_A \\
    \bm{\Delta}_B \\
    \bm{\Delta}_C \\
  \end{array}
\right],
\]
and the goal is to
\[
\min_{\mu} \left| \left| {\bf X}_1 - \left( ({\bf C} + \mu \bm{\Delta}_C) \odot ({\bf B} + \mu \bm{\Delta}_B) \right) ({\bf A} + \mu \bm{\Delta}_A)^T \right| \right|_F^2.
\]
Note that the above cost function is a polynomial of degree 6 in $\mu$. We can determine the coefficients $c_0, \cdots, c_6$ of this polynomial by evaluating it for $7$ different values of $\mu$ and solving
\[
\left[
  \begin{array}{ccccc}
    1 & \mu_1 & \mu_1^2 & \cdots & \mu_1^6 \\
    1 & \mu_2 & \mu_2^2 & \cdots & \mu_2^6 \\
     &  & \vdots &  &  \\
     1 & \mu_7 & \mu_7^2 & \cdots & \mu_7^6 \\
  \end{array}
\right]
\left[
  \begin{array}{c}
  c_0\\
  c_1\\
  \vdots\\
  c_6\\
  \end{array}
\right]
=
\left[
  \begin{array}{c}
  \ell_1\\
  \ell_2\\
  \vdots\\
  \ell_7\\
  \end{array}
\right],
\]
where $\ell_1, \cdots, \ell_7$ are the corresponding loss values.
Once the coefficients are determined, the derivative is the $5$-th order polynomial $c_1 +2 c_2 \mu + \cdots + 6 c_6 \mu^5$, and we can use numerical root finding to evaluate the loss at its roots and pick the best $\mu$. The drawback of this is that it requires $11$ evaluations of the loss function. We can work out the polynomial coefficients analytically, but this can only save about half of the computation. The bottom line is that optimal line search costs more than gradient computation {\em per se} (which roughly corresponds to 3 evaluations of the loss function, each requiring $IJKF$ flops for dense data). In practice, we typically use a small, or ``good enough'' $\mu$, resorting to exact line search in more challenging cases, for instance where the algorithm encounters ``swamps''. Note that the possibility of exact line search is a profound implication of the multilinearity of the problem. More generally, the optimal update in a search plane (involving two search directions $(\bm{\Delta}_{A,1}^T,\bm{\Delta}_{B,1}^T,\bm{\Delta}_{C,1}^T)^T$ and $(\bm{\Delta}_{A,2}^T,\bm{\Delta}_{B,2}^T,\bm{\Delta}_{C,2}^T)^T$ and even in a three-dimensional search space (additionally involving scaling of $({\bf A}^T, {\bf B}^T,{\bf C}^T)^T$) can be found via polynomial rooting \cite{LS-ID-MVB-LDL}. 

\subsection{Missing values}
Consider the ${\bf C}$-update step in ALS, i.e., $\min_{{\bf C}} ||{\bf X}_3 - ({\bf B} \odot {\bf A}) {\bf C}^T||_F^2$. If there are missing elements in $\Xt$ (and so in ${\bf X}_3$), define the weight tensor
\[
\Wt(i,j,k) = \left\{ \begin{array}{lc}
                       1, & \Xt(i,j,k):~\text{available} \\
                       0, & \text{otherwise}.
                     \end{array} \right.,
\]
and consider $\min_{{\bf C}} ||{\bf W}_3 * ({\bf X}_3 - ({\bf B} \odot {\bf A}) {\bf C}^T)||_F^2$ $\Longleftrightarrow$
$\min_{{\bf C}} ||{\bf W}_3 * {\bf X}_3 - {\bf W}_3 * (({\bf B} \odot {\bf A}) {\bf C}^T)||_F^2$, where matrix ${\bf W}_3$ is the matrix unfolding of tensor ${\bf W}$ obtained in the same way that matrix ${\bf X}_3$ is obtained by unfolding tensor ${\bf X}$. Notice that the Hadamard operation applies to the product $(({\bf B} \odot {\bf A}) {\bf C}^T)$, not to $({\bf B} \odot {\bf A})$ -- and this complicates things. One may think of resorting to column-wise updates, but this does not work either. Instead, if we perform {\em row-wise} updates on ${\bf C}$, then we have to deal with minimizing over ${\bf C}(k,:)$ the squared norm of vector \[
\text{Diag}({\bf W}_3(:,k)) {\bf X}_3(:,k) - \text{Diag}({\bf W}_3(:,k))({\bf B} \odot {\bf A})({\bf C}(k,:))^T,
\]
which is a simple linear least squares problem.

There are two basic alternatives to the above strategy for handling missing data. One is to use derivative-based methods, such as (stochastic) gradient descent (see next two subsections) or Gauss-Newton -- derivatives are easy to compute, even in the presence of $\Wt$. Stochastic gradient descent, in particular, computes gradient estimates by drawing only from the observed values. Effectively bypassing the element-wise multiplication by ${\bf W}$, stochastic gradient methods deal with missing data in a natural and effortless way. This point is well-known in the machine learning community, but seemingly under-appreciated in the signal processing community, which is more used to handling complete data.

The other alternative is to use a form of expectation-maximization to impute the missing values together with the estimation of the model parameters ${\bf A}, {\bf B}, {\bf C}$ \cite{TB05missing}. One can initially impute misses with the average of the available entries (or any other reasonable estimate). More specifically, let $\Xt_a$ be a tensor that contains the available elements, and $\Xt_m$ the imputed ones. Then set $\Xt_c = \Wt * \Xt_a + (1 - \Wt) * \Xt_m$, and fit ${\bf A}, {\bf B}, {\bf C}$ to $\Xt_c$. Set $\Xt_m = \left\llbracket{\bf A}, {\bf B}, {\bf C}\right\rrbracket$, $\Xt_c = \Wt * \Xt_a + (1 - \Wt) * \Xt_m$, and repeat. It is easy to see that the above procedure amounts to alternating optimization over ${\bf A}, {\bf B}, {\bf C}, (1 - \Wt) * \Xt_m$, and it thus decreases the cost function monotonically.

Whether it is best to ignore missing elements or impute them is dependent on the application; but we note that for very big and sparse data, imputation is very inefficient in terms of memory, and is thus avoided.

Note that, as a short-cut in large-scale applications, one may deliberately use only part of the available entries when estimating a decomposition \cite{NV-OD-LS-LDL} (see also the next section); entries that have not been selected, may be used for model cross-validation. If the data structure is sufficiently strong to allow such an approach, it can lead to a very significant speed-up and it should be considered as an alternative to full-scale parallel computation. As a matter of fact, in applications that involve tensors of high order, the latter is not an option due to the curse of dimensionality (i.e., the number of tensor entries depends exponentially on the order and hence quickly becomes astronomically high).

\subsection{Stochastic gradient descent}

{\em Stochastic gradient descent} (SGD) has become popular in the machine learning community for many types of convex and, very recently, non-convex optimization problems as well. In its simplest form, SGD randomly picks a data point $\Xt(i,j,k)$ from the available ones, and takes a gradient step only for those model parameters that have an effect on $\Xt(i,j,k)$; that is, only the $i$-th row of ${\bf A}$, the $j$-th row of ${\bf B}$ and the $k$-th row of ${\bf C}$. We have
\[
\frac{\partial}{\partial {\bf A}(i,f)} \left( \Xt(i,j,k) - \sum_{f=1}^F {\bf A}(i,f) {\bf B}(j,f) {\bf C}(k,f) \right)^2 =
\]
\[
-2 \left( \Xt(i,j,k) - \sum_{f^{'}=1}^F {\bf A}(i,f^{'}) {\bf B}(j,f^{'}) {\bf C}(k,f^{'}) \right) \times
\]
\[
{\bf B}(j,f) {\bf C}(k,f),
\]
so that
\[
\frac{\partial}{\partial {\bf A}(i,:)} = -2 \left( \Xt(i,j,k) - \sum_{f=1}^F {\bf A}(i,f) {\bf B}(j,f) {\bf C}(k,f) \right) \times
\]
\[
\left({\bf B}(j,:) * {\bf C}(k,:)\right).
\]
Notice that the product ${\bf B}(j,:) * {\bf C}(k,:)$ is used once outside and once inside the parenthesis, so the number of multiplications needed is $2F$ for the update of ${\bf A}(i,:)$, and $6F$ for the (simultaneous) SGD update of ${\bf A}(i,:)$, ${\bf B}(j,:)$, ${\bf C}(k,:)$. This makes SGD updates very cheap, but the biggest gain is in terms of random access memory (we only need to load one $\Xt(i,j,k)$, and ${\bf A}(i,:)$, ${\bf B}(j,:)$, ${\bf C}(k,:)$ each time). There is one more inherent advantage to SGD: it can naturally deal with missing elements, as these are simply never ``recalled'' to execute an update. The drawback is that a truly random disk access pattern is a terrible idea (especially if the data is stored in rotating media) as the computation will inevitably be bogged down from the disk I/O. For this reason, we prefer fetching blocks of data from secondary memory, and use intelligent caching strategies. To this end, note that SGD updates involving (stemming from) $\Xt(i,j,k)$ and $\Xt(i^{'},j^{'},k^{'})$ do not conflict with each other and can be executed in parallel, provided
$i^{'} \neq i,~j^{'} \neq j,~k^{'} \neq k$ -- where all three $\neq$ must hold simultaneously. This means that the maximum number of parallel SGD updates is $\min(I,J,K)$ in the three-way case, and $\min(\left\{ I_n \right\}_{n=1}^N)$ in the general $N$-way case. This limits the {\em relative} level of parallelization, especially for high $N$. Another disadvantage is that the convergence can be very slow (sublinear). See \cite{doi:10.1137/1.9781611973440.13} for parallel SGD algorithms for CPD and coupled tensor decomposition. In \cite{NV-LDL} a block sampling variant is discussed that allows one to efficiently decompose TB-size tensors without resorting to parallel computation. The approach leverages CPD uniqueness of the sampled blocks to uniqueness of the CPD of the full tensor. For both SGD and the block sampling variant, the choice of step size is important for convergence. When chosen appropriately, the latter method often converges very fast.
A randomized block-sampling approach for very sparse datasets was proposed in \cite{Papalexakis:2012:PSP:2405473.2405521}, building upon the idea of parallel CPD decomposition of multiple pseudo-randomly drawn sub-tensors, and combining the CPDs using {\em anchor rows}. Sampling is based on mode densities, and identifiability is guaranteed if the sub-tensors have unique CPD.

\subsection{Constraints}

In practice, we are often interested in imposing constraints on a CPD model. One may question the need for this -- after all, CPD is essentially unique under relatively mild conditions. Constraints are nevertheless useful in

\noindent $\bullet$ Restoring identifiability in otherwise non-identifiable cases;

\noindent $\bullet$ Improving estimation accuracy in relatively challenging (low-SNR, and/or barely identifiable, and/or numerically ill-conditioned) cases;

\noindent $\bullet$ Ensuring interpretability of the results (e.g., power spectra cannot take negative values); and

\noindent $\bullet$ As a remedy against ill-posedness.

There are many types of constraints that are relevant in many applications, including those in the ``laundry list'' below.

\noindent $\bullet$ {\it Symmetry or Hermitian (conjugate) symmetry:} ${\bf B} = {\bf A}$, or ${\bf B} = {\bf A}^*$, leading to
$\Xt(:,:,k) = {\bf A} {\bf D}_k({\bf C}) {\bf A}^T$ or $\Xt(:,:,k) = {\bf A} {\bf D}_k({\bf C}) {\bf A}^H$. This is actually only {\em partial symmetry}, with {\em full symmetry} (or simply {\em symmetry}) corresponding to ${\bf C} = {\bf B} = {\bf A}$. Partial symmetry corresponds to joint diagonalization of the frontal slabs, using a non-orthogonal and possibly fat diagonalizer ${\bf A}$ -- that is, the inner dimension can exceed the outer one. Symmetric tensors (with possible conjugation in certain modes) arise when one considers higher-order statistics (HOS).

\noindent $\bullet$ {\it Real-valued parameters:} When $\Xt \in \Real^{I \times J \times K}$, complex-valued ${\bf A}, {\bf B}, {\bf C}$ make little sense, but sometimes do arise because tensor rank is sensitive to the field over which the decomposition is taken. This is an issue in some applications, particularly in Chemistry and Psychology. Engineers are usually not annoyed by complex ${\bf A}, {\bf B}, {\bf C}$, but they may still need the following (stronger) constraint to model, e.g., power spectra.

\noindent $\bullet$ {\it Element-wise non-negativity:} ${\bf A} \geq 0$ and/or ${\bf B} \geq 0$, and/or ${\bf C} \geq 0$. When all three are in effect, the resulting problem is known as {\em non-negative tensor factorization} (NTF). More generally, bound constraints may apply. Non-negativity can help restore uniqueness, because even non-negative matrix factorization (NMF) is unique under certain conditions -- these are much more restrictive than those for CPD uniqueness, but, taken together, low-rank CPD structure and non-negativity can restore identifiability. To appreciate this, note that when $k_{\bf C}=1$ CPD alone cannot be unique, but if NMF of $\Xt(:,:,k) = {\bf A} {\bf D}_k({\bf C}) {\bf B}^T$ is unique (this requires $F<\min(I,J)$ and a certain level of sparsity in ${\bf A}$ and ${\bf B}$), then non-negativity can still ensure essential uniqueness of ${\bf A}, {\bf B}, {\bf C}$.

\noindent $\bullet$ {\it Orthogonality:} This may for instance be the result of prewhitening \cite{MS-LDL-PC-SI-LD}.

\noindent $\bullet$ {\it Probability simplex constraints:} ${\bf A}(i,:) \geq 0$, ${\bf A}(i,:) {\bf 1}=1$, $\forall~i$, or ${\bf A}(:,f) \geq 0$, ${\bf 1}^T {\bf A}(:,f)=1$, $\forall~f$, are useful when modeling allocations or probability distributions.

\noindent $\bullet$ {\it Linear constraints:} More general linear constraints on ${\bf A}, {\bf B}, {\bf C}$ are also broadly used. These can be column-wise, row-wise, or matrix-wise, such as $\text{tr}({\bf W} {\bf A}) \leq b$.

\noindent $\bullet$ {\it Monotonicity and related constraints:} These are useful in cases where one deals with, e.g., concentrations that are known to be decaying, or spectra that are known to have a single or few peaks (unimodality, oligo-modality \cite{BroSidChem1998}).

\noindent $\bullet$ {\it Sparsity:} In many cases one knows (an upper bound on) the number of nonzero elements of ${\bf A}, {\bf B}, {\bf C}$, per column, row, or as a whole; or the number of nonzero columns or rows of ${\bf A}, {\bf B}, {\bf C}$ (group sparsity).

\noindent $\bullet$ {\it Smoothness:} Smoothness can be measured in different ways, but a simple one is in terms of convex quadratic inequalities such as
\[
\left| \left|
\left[
  \begin{array}{rrrrrr}
    -1 & 1 & 0 & \cdots &  & 0 \\
    0 & -1 & 1 & 0 & \cdots & 0 \\
    \vdots &  & \ddots & \ddots &  &  \\
  \end{array}
\right] {\bf A}
\right| \right|_F^{2}.
\]

\noindent $\bullet$ {\it Data model constraints:} All the above constraints apply to the model parameters. One may also be interested in constraints on the reconstructed model of the data, e.g.,
\[
({\bf A} \odot {\bf B} \odot {\bf C}) {\bf 1} \geq 0,~\text{(element-wise)},~ {\bf 1}^T ({\bf A} \odot {\bf B} \odot {\bf C}) {\bf 1} = 1,
\]
if $\Xt$ models a joint probability distribution, or in $({\bf A} \odot {\bf B} \odot {\bf C}) {\bf 1}$ being ``smooth'' in a suitable sense.

\noindent $\bullet$ {\it Parametric constraints:} All the above are {\em non-parametric} constraints. There are also important {\em parametric constraints} that often arise, particularly in signal processing applications. Examples include Vandermonde or Toeplitz structure imposed on ${\bf A}$, ${\bf B}$, ${\bf C}$. Vandermonde matrices have columns that are (complex or real) exponentials, and Toeplitz matrices model linear time-invariant systems with memory (convolution). Factors may further be explicitly modeled as polynomial, sum-of-exponential, exponential polynomial, rational, or sigmoidal functions. This may for instance reduce the number of parameters needed and suppress noise. Various non-parametric constraints can be explicitly parametrized; e.g., non-negativity can be parametrized as ${\bf A}(i,j)=\theta_{i,j}^2,~\theta \in \Real$, a magnitude constraint $|{\bf A}(i,j)|=1$ as ${\bf A}(i,j)=e^{\sqrt{-1} \theta_{i,j}}$, and orthogonality may be parameterized via Jacobi rotations or Householder reflections. Smoothness, probability simplex, and linear constraints can be formulated as parametric constraints as well.

The main issue is then how do we go about enforcing these constraints. The answer depends on the type of constraint considered.
Parametric constraints can be conveniently handled in the framework of derivative-based methods such as quasi-Newton and NLS, by using the chain rule in the computation of derivatives. In this way, the proven convergence properties of these methods can be extended to constrained and possibly coupled decompositions \cite{SDF}.

Linear equality and inequality constraints (including monotonicity) on individual loading matrices (${\bf A}, {\bf B}, {\bf C}$) can be handled in ALS, but change each conditional update from linear least-squares (solving a system of linear equations in the least-squares sense)  to quadratic programming. E.g., non-negative least-squares is much slower than unconstrained least squares, and it requires specialized algorithms. This slows down the outer ALS loop considerably; but see the insert entitled {\em Constrained least squares using ADMM}. Direct application of ADMM for fitting the (nonconvex) CPD model has been considered in \cite{7152968}, using non-negative tensor factorization as the working example. This approach has certain advantages: it can outperform standard approaches in certain scenarios, and it can incorporate various constraints beyond non-negativity, including constraints that couple some of the factor matrices.

The drawback is that direct ADMM requires sophisticated parameter tuning, and even then it is not guaranteed to converge -- in contrast to the AO-ADMM hybrid approach of \cite{HuaSidLiaTSP2016} that soon followed \cite{7152968}. A nice contribution of  \cite{7152968} is that it showed how to parallelize ADMM (and, as a special case, plain ALS) for high-performance computing environments.

\begin{table}
\colorbox{lightgray}{
\begin{minipage}{.48\textwidth}
\normalsize
{\bf Constrained least squares using ADMM:} Consider $\min_{{\bf C} \in {\cal C}} \left| \left| {\bf X}_3 - {\bf M} {\bf C}^T \right| \right|_F^2$, where ${\cal C}$ is a convex constraint set, and in the context of ALS for CPD ${\bf M}:=({\bf B} \odot {\bf A})$. Introduce an auxiliary variable $\widetilde {\bf C}$ and the function
\[
f_{{\cal C}}({\bf C}) := \left\{ \begin{array}{cc}
0, & {\bf C} \in {\cal C} \\
\infty, & \text{otherwise.}
\end{array}
\right.
\]
Then we may equivalently consider
\begin{eqnarray*}
	\min_{{\bf C}, \widetilde {\bf C}} & \frac{1}{2} \left| \left| {\bf X}_3 - {\bf M} {\bf C}^T \right| \right|_F^2 + f_{{\cal C}}(\widetilde {\bf C})\\
	\text{subject to:} & \widetilde {\bf C} = {\bf C}.
\end{eqnarray*}
This reformulation falls under the class of problems that can be solved using the alternating direction method of multipliers (ADMM) -- see \cite{Boyd2011} for a recent tutorial. The ADMM iterates for this problem are
\begin{equation*}
\begin{aligned}[l]
{\bf C}^T & \leftarrow ({\bf M}^T {\bf M} + \rho {\bf I})^{-1}
({\bf M}^T {\bf X}_3 + \rho(\widetilde {\bf C}+ {\bf U})^T),\\
{\widetilde {\bf C}}  & \leftarrow \arg\min_{\widetilde {\bf C}} f_{{\cal C}}(\widetilde {\bf C}) +
\frac{\rho}{2} \| {\widetilde {\bf C}} - {\bf C} + {\bf U} \|_F^2, \\
{\bf U} & \leftarrow {\bf U} + \widetilde {\bf C} - {\bf C}.\\
\end{aligned}
\end{equation*}
Note that ${\bf M}^T {\bf X}_3$ and $({\bf M}^T {\bf M} + \rho {\bf I})^{-1}$ remain fixed throughout the ADMM iterations. We can therefore compute ${\bf M}^T {\bf X}_3$ (or ${\bf X}_3^T {\bf M}$: a MTTKRP computation) and the Cholesky decomposition of $({\bf M}^T {\bf M} + \rho {\bf I})={\bf L} {\bf L}^T$, where ${\bf L}$ is lower triangular, and amortize the cost throughout the iterations. Then each update of ${\bf C}$ can be performed using one forward and one backward substitution, at much lower complexity. The update of $\widetilde {\bf C}$ is the so-called \emph{proximity operator} of the function $(1/\rho) f_{{\cal C}}(\cdot)$, which is easy to compute in many cases of practical interest \cite{Parikh2014}, including (but not limited to)
\begin{itemize}
	\item Non-negativity. In this case, the update simply projects onto the non-negative orthant. Element-wise bounds can be handled in the same way.
	\item Sparsity via $\ell_1$-regularization. The update is the well-known \emph{soft-thresholding} operator.
	\item Simplex constraint. See \cite{duchi2008efficient}.
	\item Smoothness regularization. See \cite{HuaSidLiaTSP2016}.
\end{itemize}
In the context of CPD fitting, the above ADMM loop is embedded within the outer Alternating Optimization (AO) loop that alternates over the matrix variables ${\bf A}$, ${\bf B}$, ${\bf C}$. After several outer iterations, one can use previous iterates to warm-start the inner ADMM loop. This ensures that the inner loop eventually does very few iterations; and, due to computation caching / amortization, each inner loop costs as much as solving an unconstrained linear least squares problem. The net result is that {\em we can perform constrained ALS at roughly the cost of unconstrained ALS}, for a wide variety of constraints, in a mix-and-match, plug-and-play fashion, so long as the proximity operator involved is easy to compute. This is the main attraction of the AO-ADMM approach of \cite{HuaSidLiaTSP2016}, which can also deal with more general loss functions and missing elements, while maintaining the monotone decrease of the cost and conceptual simplicity of ALS. The AO-ADMM framework has been recently extended to handle robust tensor factorization problems where some slabs are grossly corrupted \cite{7208891}.
\end{minipage}}
\vspace{-15pt}
\end{table}

Also note that linear constraints on the reconstructed data model, such as $({\bf A} \odot {\bf B} \odot {\bf C}) {\bf 1} \geq 0$, are nonlinear in ${\bf A}, {\bf B}, {\bf C}$, but become conditionally linear in ${\bf A}$ when we fix ${\bf B}$ and ${\bf C}$, so they too can be handled in the same fashion. Smoothness constraints such as the one above are convex, and can be dualized when solving the conditional mode loading updates, so they can also be handled in ALS, using quadratic programming. Symmetry constraints are more challenging, but one easy way of approximately handling them is to introduce a quadratic penalty term, such as $||{\bf B} - {\bf A}||_F^2$, which has the benefit of keeping the conditional updates in ALS simple. Sparsity can often be handled using $\ell_1$ regularization, which turns  linear least squares conditional updates to LASSO-type updates, but one should beware of latent scaling issues, see \cite{PapSidBro2013}.

Many other constraints though, such as hard $\ell_0$ sparsity, unimodality, or finite-alphabet constraints are very hard to handle.
A general tool that often comes handy under such circumstances is the following. Consider
$\left| \left| {\bf X}_3 - ({\bf B} \odot {\bf A}) {\bf C}^T \right| \right|_F^2$, and let ${\bf M}:=({\bf B} \odot {\bf A})$. Then $\left| \left| {\bf X}_3 - ({\bf B} \odot {\bf A}) {\bf C}^T \right| \right|_F^2 = \left| \left| {\bf X}_3 - \sum_{f=1}^F {\bf M}(:,f)({\bf C}(:,f))^T \right| \right|_F^2$. Fix all the columns of ${\bf C}$ except column $f_0$. Define
$\tilde {\bf X}_3 = {\bf X}_3 - \sum_{f=1,~f \neq f_0}^F {\bf M}(:,f)({\bf C}(:,f))^T$, ${\bf c}:={\bf C}(:,f_0)$, and ${\bf m}:={\bf M}(:,f_0)$, and consider
\[
\min_{{\bf c} \in {\cal C}} \left| \left| \tilde {\bf X}_3 - {\bf m} {\bf c}^T \right| \right|_F^2,
\]
where ${\cal C}$ is a column-wise constraint.

This corresponds to performing ALS over each column of ${\bf C}$, i.e., further breaking down the variable blocks into smaller pieces.
\begin{lemma} \label{oslem}
For {\em any} column-wise constraint set ${\cal C}$ it holds that
\[
\min_{{\bf c} \in {\cal C}} \left| \left| \tilde {\bf X}_3 - {\bf m} {\bf c}^T \right| \right|_F^2 \Longleftrightarrow \min_{{\bf c} \in {\cal C}} \left| \left| \tilde {\bf c} - {\bf c} \right| \right|_F^2,
\]
where $\tilde {\bf c} := \left( \frac{{\bf m}^T}{||{\bf m}||_2^2} \tilde {\bf X}_3\right)^T$, i.e., the optimal solution of the constrained least-squares problem is simply the projection of the unconstrained least-squares solution onto the constraint set ${\cal C}$. This is known as the {\em Optimal Scaling Lemma}; see \cite{BroSidChem1998} and references therein.
\end{lemma}

Armed with Lemma \ref{oslem}, we can easily enforce a wide variety of {\em column-wise} (but not row-wise) constraints. For example,
\begin{itemize}
\item if ${\cal C}=\left\{ {\bf c} \in \Real^{K} ~|~w({\bf c})=s\right\}$, where $w(\cdot)$ counts the number of nonzeros (Hamming weight) of its argument, then ${\bf c}_{\text{opt}}$ is obtained by zeroing out the $K-s$ smallest elements of $\tilde {\bf c}$.
\item If ${\cal C}$ is the set of complex exponentials, then ${\bf c}_{\text{opt}}$ is obtained by peak-picking the magnitude of the Fourier transform of $\tilde {\bf c}$.
\item If ${\cal C}$ is the set of non-decreasing sequences, then ${\bf c}_{\text{opt}}$ is obtained via monotone regression of $\tilde {\bf c}$.
\end{itemize}
The drawback of using Lemma \ref{oslem} is that it splits the optimization variables into smaller blocks, which usually slows down convergence. Variable splitting can also be used to tackle problems that involve constraints that couple the loading matrices. An example for the partial symmetry constraint (${\bf B} = {\bf A}$) is provided in the supplementary material.

\section{Cram\'er-Rao Bound}
\label{sec:CRB}

The \crb~bound is the most commonly used performance benchmarking tool in statistical signal processing. It is a lower bound on the variance of any {\em unbiased} estimator (and thus on mean square error of unbiased estimators), which is expressed in terms of the (pseudo-)inverse of the {\em Fisher information matrix}. In many cases it is hard to prove that an estimator is unbiased, but if the empirical bias is small, the \crb~bound is still used as a benchmark. See the supplementary material for general references and more motivation and details about the \crb~bound. We derive the Fisher information matrix and the corresponding \crb~bound for the CPD model
\begin{equation}\label{eq:cpd3}
\vec{\Xt}=(\C\odot\B\odot\A)\one + \vec{\Nt}.
\end{equation}
Assuming the elements of $\Nt$ come from an i.i.d. Gaussian distribution with variance $\sigma^2$, it has been shown that the FIM can be derived in a simpler way without resorting to taking expectations. Denote $\th$ as the long vector obtained by stacking all the unknowns,
\[
\th = \left[~\vec{\A}^T~\vec{\B}^T~\vec{\C}^T~\right]^T,
\]
and define the nonlinear function
\[
\ph(\th)=(\C\odot\B\odot\A)\one,
\]
then the FIM is simply given by~\cite{basu2000stability} (cf. Proposition~\ref{ppst:crb4gauss} in the supplementary material)
\[
\fim = \frac{1}{\sigma^2}\jacob_{\th}\ph(\th)^T\jacob_{\th}\ph(\th),
\]
where $\jacob_{\th}\ph(\th)$ is the Jacobian matrix of $\ph(\th)$, which can be partitioned into three blocks
\[
\jacob_{\th}\ph(\th) = \begin{bmatrix}
\jacob_{\A}\ph(\th) & \jacob_{\B}\ph(\th) & \jacob_{\C}\ph(\th)
\end{bmatrix}.
\]
Rewrite $\ph(\th)$ as follows
\begin{align}
\ph(\th) & =(\C\odot\B\odot\A)\one = \vec{\A(\C\odot\B)^T} \nonumber \\
&= ((\C\odot\B)\otimes\eye{I})\vec{\A} \label{eq:JA}\\
&= \K_{JK,I}(\A\odot\C\odot\B)\one = \K_{JK,I}\vec{\B(\A\odot\C)^T} \nonumber \\
&= \K_{JK,I}((\A\odot\C)\otimes\eye{J})\vec{\B} \label{eq:JB}\\
&= \K_{K,IJ}(\B\odot\A\odot\C)\one = \K_{K,IJ}\vec{\C(\B\odot\A)^T} \nonumber \\
&= \K_{K,IJ}((\B\odot\A)\otimes\eye{K})\vec{\C}, \label{eq:JC}
\end{align}
where $\K_{m,n}$ represents the commutation matrix~\cite{magnus1979commutation} of size $mn \times mn$. The commutation matrix is a permutation matrix that has the following properties:
\begin{enumerate}
\item $\K_{m,n}\vec{\Sb}=\vec{\Sb^T}$, where $\Sb$ is $m \times n$;
\item $\K_{p,m}(\Sb\otimes\T)=(\T\otimes\Sb)\K_{q,n}$, where $\T$ is $p \times q$;
\item $\K_{p,m}(\Sb\odot\T)=\T\odot\Sb$;
\item $\K_{n,m}=\K_{m,n}^T=\K_{m,n}^{-1}$;
\item $\K_{mp,n}\K_{mn,p}=\K_{m,np}$.
\end{enumerate}
From \eqref{eq:JA}-\eqref{eq:JC}, it is easy to see that
\begin{align*}
\jacob_{\A}\ph(\th) &= ((\C\odot\B)\otimes\eye{I}), \\
\jacob_{\B}\ph(\th) &= \K_{JK,I}((\A\odot\C)\otimes\eye{J}),\\
\jacob_{\C}\ph(\th) &= \K_{K,IJ}((\B\odot\A)\otimes\eye{K}).
\end{align*}
Similarly, we can partition the FIM into nine blocks
\[
\fim = \frac{1}{\sigma^2}\ufim
= \frac{1}{\sigma^2}\begin{bmatrix}
\ufim_{\A,\A} & \ufim_{\A,\B} & \ufim_{\A,\C} \\
\ufim_{\B,\A} & \ufim_{\B,\B} & \ufim_{\B,\C} \\
\ufim_{\C,\A} & \ufim_{\C,\B} & \ufim_{\C,\C}
\end{bmatrix}.
\]
For the diagonal blocks, using the properties of commutation matrices
\begin{align*}
\ufim_{\A,\A} &= ((\C\odot\B)\otimes\eye{I})^T((\C\odot\B)\otimes\eye{I}) \\
&= (\C^T\C*\B^T\B)\otimes\eye{I},\\
\ufim_{\B,\B} &= (\A^T\A*\C^T\C)\otimes\eye{J},\\
\ufim_{\C,\C} &= (\B^T\B*\A^T\A)\otimes\eye{K}.
\end{align*}
For the off-diagonal blocks, we derive $\ufim_{\B,\C}$ here for tutorial purposes
\begin{align*}
\ufim_{\B,\C} & = \jacob_{\B}\ph(\th)^T\jacob_{\C}\ph(\th) \\
&= ((\A\odot\C)\otimes\eye{J})^T\K_{JK,I}^T\K_{K,IJ}((\B\odot\A)\otimes\eye{K}) \\
&= ((\A\odot\C)\otimes\eye{J})^T\K_{IK,J}((\B\odot\A)\otimes\eye{K}).
\end{align*}
To further simplify the expression, let us consider the product of $\ufim_{\B,\C}$ and $\vec{\tilde{\C}}$, where $\tilde{\C}$ is an arbitrary $K \times F$ matrix
\begin{align*}
& \ufim_{\B,\C}\vec{\tilde{\C}} \\
=& ((\A\odot\C)\otimes\eye{J})^T\K_{IK,J}((\B\odot\A)\otimes\eye{K})\vec{\tilde{\C}} \\
=& ((\A\odot\C)\otimes\eye{J})^T\K_{IK,J}\vec{\tilde{\C}(\B\odot\A)^T}\\
=& \K_{F,J}(\eye{J}\otimes(\A\odot\C)^T)(\B\odot\A\odot\tilde{\C})\one \\
=& \K_{F,J}(\B\odot(\A^T\A*\C^T\tilde{\C}))\one \\
=& ((\A^T\A*\C^T\tilde{\C})\odot\B)\one \\
=& \vec{\B(\A^T\A*\C^T\tilde{\C})^T} \\
=& (\eye{F}\otimes\B)\diag{\vec{\A^T\A}}\vec{\tilde{\C}^T\C}\\
=& (\eye{F}\otimes\B)\diag{\vec{\A^T\A}}\K_{F,F}(\eye{F}\otimes\C^T)\vec{\tilde{\C}}.
\end{align*}
This holds for all possible $\tilde{\C}\in\Real^{K \times F}$, implying
\[
\ufim_{\B,\C} = (\eye{F}\otimes\B)\diag{\vec{\A^T\A}}\K_{F,F}(\eye{F}\otimes\C^T).
\]
Similarly, we can derive the expression for $\ufim_{\A,\B}$ and $\ufim_{\A,\C}$.
The entire expression for the FIM $\fim=(1/\sigma^2)\ufim$ is given in~\eqref{eq:fim3}.

\begin{figure*}[!t]
\normalsize
\begin{equation}\label{eq:fim3}
\sigma^2\fim=\ufim = \begin{bmatrix}
\left(\B^T\B* \C^T\C\right)\otimes \I_I &
(\I_F\otimes\A)\F_{\C}(\I_F\otimes\B)^T &
(\I_F\otimes\A)\F_{\B}(\I_F\otimes\C)^T \\
(\I_F\otimes\B)\F_{\C}(\I_F\otimes\A)^T &
\left(\A^T\A* \C^T\C\right)\otimes \I_J &
(\I_F\otimes\B)\F_{\A}(\I_F\otimes\C)^T \\
(\I_F\otimes\C)\F_{\B}(\I_F\otimes\A)^T &
(\I_F\otimes\C)\F_{\A}(\I_F\otimes\B)^T &
\left(\A^T\A* \B^T\B\right)\otimes \I_K
\end{bmatrix} = \bm{\Delta} + \bm{\Upsilon}\F\bm{\Upsilon}^T
\end{equation}
\[
\begin{array}{l}
\F_{\A} = \diag{\vec{\A^T\A}}\K_{F,F} \\
\F_{\B} = \diag{\vec{\B^T\B}}\K_{F,F} \\
\F_{\C} = \diag{\vec{\C^T\C}}\K_{F,F}
\end{array}~~~
\F = \begin{bmatrix}
0 & \F_{\B} & \F_{\C} \\
\F_{\A} & 0 & \F_{\C} \\
\F_{\A} & \F_{\B} & 0
\end{bmatrix}
\]
\[
\bm{\Delta} = \begin{bmatrix}
\left(\B^T\B* \C^T\C\right)\otimes \I_I &
0 & 0 \\ 0 &
\left(\A^T\A* \C^T\C\right)\otimes \I_J &
0 \\ 0 & 0 &
\left(\A^T\A* \B^T\B\right)\otimes \I_K
\end{bmatrix},
\bm{\Upsilon} = \begin{bmatrix}
\I_F\otimes\A &
0 & 0 \\ 0 &
\I_F\otimes\B &
0 \\ 0 & 0 &
\I_F\otimes\C
\end{bmatrix}
\]
\hrulefill
\end{figure*}

Formulae for the Jacobian matrix and FIM have appeared in~\cite{liu2001cramer,vorobyov2005robust,tomasi2006thesis,phan2013low,tichavsky2013cramer}, but the derivation is not as clear and straightforward as the one given here. Furthermore, we show below that $\ufim$ is rank deficient with deficiency at least $2F$, and identify the associated null subspace $\ufim$. When the FIM is singular, it has been shown that we can simply take its pseudo-inverse as CRB~\cite{stoica2001parameter}, albeit this bound might be far from attainable, even in theory. When the size of the tensor is large, it may be computationally intractable to take the pseudo-inverse of $\ufim$ directly, which takes $O((I+J+K)^3F^3)$ flops. Instead of taking this route, we explain how we can compute $\ufim^\dagger$ efficiently when the deficiency is exactly $2F$. Simulations suggest that the deficiency is exactly $2F$ when the model is identifiable.

\begin{proposition}
The rank of the $(I+J+K)F \times (I+J+K)F$ FIM $\fim$ defined in \eqref{eq:fim3} is at most $(I+J+K)F-2F$.
\end{proposition}
\begin{IEEEproof}
Please refer to the proof of Proposition~\ref{ppst:fimrank-cp} for the more general $N$-way tensor case, in the supplementary material.
%
\end{IEEEproof}
When the rank deficiency is equal to $2F$ (which appears to be true almost surely when the model is identifiable, based on our simulations) then we can compute the pseudo-inverse of $\ufim$ efficiently, invoking the following lemma proven in the supplementary material of~\cite{huang2014spm}.
\begin{lemma}\label{lmm:pinv1}
Let matrix $\M$ be symmetric and singular, and the matrix $\L$ satisfying $\range{\L}=\Null{\M}$, then
\begin{equation}\label{eq:pinv1}
    \M^{\dag} = (\M + \L\L^T)^{-1}
                        - (\L^{\dag})^T \L^{\dag}.
\end{equation}
\end{lemma}

A matrix $\L$ that spans the null-space of $\ufim$ can be written as
\[
\L = \bm{\Upsilon}\mathbf{E},
\]
where
\[
\mathbf{E} = \begin{bmatrix}
~\I_F\odot\I_F & ~\I_F\odot\I_F \\
\!-\I_F\odot\I_F & 0 \\
0 & \!-\I_F\odot\I_F
\end{bmatrix}.
\]
Since $\L$ has full column rank, its pseudo-inverse is
\[
\L^\dagger = \inv{\L^T\L}\L^T
= \inv{\mathbf{E}^T\bm{\Upsilon}^T\bm{\Upsilon}\mathbf{E}}\mathbf{E}^T\bm{\Upsilon}^T.
\]

Next, we define $\fullfim$ by ``completing'' the range-space of $\ufim$
\begin{align*}
\fullfim & = \ufim + \L\L^T \\
&= \bm{\Delta} + \bm{\Upsilon}\F\bm{\Upsilon}^T + \bm{\Upsilon}\mathbf{EE}^T\bm{\Upsilon}^T \\
&= \bm{\Delta} + \bm{\Upsilon}\left(\F+\mathbf{EE}^T\right)\bm{\Upsilon}^T,
\end{align*}
where the definitions of $\bm{\Delta}$, $\F$, and $\bm{\Upsilon}$ can be found in~\eqref{eq:fim3}.
If $\F+\mathbf{EE}^T$ is invertible, applying matrix inversion lemma to $\fullfim$ leads to
\[
\fullfim^{-1} \!=\! \bm{\Delta}^{-1} \!-\! \bm{\Delta}^{-1}\bm{\Upsilon}
\inv{\inv{\F\!+\!\mathbf{EE}^T}\!+\!\bm{\Upsilon\Delta}^{-1}\bm{\Upsilon}}
\bm{\Upsilon}^T\bm{\Delta}^{-1}.
\]
Notice that $\bm{\Delta}$ can be efficiently inverted, because it is block diagonal, and each of its diagonal blocks is a Kronecker product. The most expensive step is to compute $\inv{\F+\mathbf{EE}^T}$ and $\inv{\inv{\F\!+\!\mathbf{EE}^T}\!+\!\bm{\Upsilon\Delta}^{-1}\bm{\Upsilon}}$, each taking $O(F^6)$ flops. However, this is still a huge improvement when $F\ll\min(I,J,K)$, compared to directly inverting $\bm{\Delta}$ with $O((I+J+K)^3F^3)$ flops. Finally, according to Lemma~\ref{lmm:pinv1}, $\ufim^\dagger$ can be computed as
\begin{align*}
\ufim^\dagger &= \bm{\Delta}^{-1} - (\L^\dagger)^T\L^\dagger \\
&=\!\bm{\Delta}^{-1} \!-\! \bm{\Delta}^{-1}\bm{\Upsilon}
\inv{\inv{\F\!+\!\mathbf{EE}^T}\!+\!\bm{\Upsilon\Delta}^{-1}\bm{\Upsilon}}
\bm{\Upsilon}^T\bm{\Delta}^{-1} \\
&~~~- \bm{\Upsilon}\mathbf{E}
\invtwo{\mathbf{E}^T\bm{\Upsilon}^T\bm{\Upsilon}\mathbf{E}}\mathbf{E}^T\bm{\Upsilon}^T,
\end{align*}
and the CRB is simply
\[
\fim^\dagger = \sigma^2\ufim^\dagger.
\]

\section{Applications}

\subsection{Blind Multiuser CDMA}

In direct-sequence code-division multiple access (DS-CDMA) communications, a transmitter sends logical digits $d \in \left\{0,1\right\}$ by transmitting one of two analog waveforms $b p(t)$, where $b:=(-1)^d \in \left\{+1,-1\right\}$,
$p(t)=\sum_{\ell=1}^L {\bf s}(\ell) c(t-\ell T_c)$, $t \in \Real$ is the bit signaling pulse, the $L \times 1$ vector ${\bf s}$ is the transmitter's {\em spreading code}, $L$ is the {\em spreading gain}, $T_c$ is the {\em chip period}, and $c(\cdot)$ is the {\em chip pulse}. In practice, the transmitter sends a sequence of digits $\left\{ {\bf d}(n) \right\}_{n=1}^N$ by transmitting $\sum_{n=1}^N {\bf b}(n) p(t-nT)$, and the receiver performs matched filtering with respect to the chip pulse and outputs a vector of chip-rate samples that, under ideal conditions (no multipath, perfect synchronization, no noise), reproduce $\left\{ {\bf y}_n = {\bf s} {\bf b}(n)\right\}_{n=1}^N$ at the other end of the communication link. Collecting these output vectors in an $L \times N$ matrix
\[
{\bf Y} := [{\bf y}_1, {\bf y}_2, \cdots, {\bf y}_N] = {\bf s} {\bf b}^T,
\]
we notice that this is rank-1, and we can therefore recover ${\bf s}$ and ${\bf b}$ from ${\bf Y}$ up to an inherently unresolvable sign ambiguity in this case, simply by reading out a column and row of ${\bf Y}$, respectively. In practice there will be noise and other imperfections, so we will in fact extract the principal component of ${\bf Y}$ instead, using SVD.
This says that we can in fact decode the transmitted sequence even if we do not know the user's ``secret'' spreading code, up to a sign ambiguity.

However, DS-CDMA is very wasteful in terms of bandwidth when used in single-user mode; in fact it is a multiuser multiplexing modality that is used as an alternative to classical frequency- or time-division multiplexing. When there are two co-channel transmitters sending information simultaneously, then (again under idealized conditions)
\[
{\bf Y} = {\bf s}_1 {\bf b}_1^T + {\bf s}_2 {\bf b}_2^T = \left[ {\bf s}_1,~{\bf s}_2\right] \left[ {\bf b}_1,~{\bf b}_2\right]^T = {\bf S} {\bf B}^T,
\]
In this case, if we know ${\bf s}_1$ and ${\bf s}_1 \perp {\bf s}_2$ (i.e., ${\bf s}_1^T {\bf s}_2=0$), then
\[
{\bf s}_1^T {\bf Y} = {\bf s}_1^T {\bf s}_1 {\bf b}_1^T + {\bf s}_1^T {\bf s}_2 {\bf b}_2^T = ||{\bf s}_1||_2^2 {\bf b}_1^T,
\]
and thus perfect interference cancelation and recovery of the transmitted bits of both users is possible. Even if ${\bf s}_1^T {\bf s}_2 \neq 0$, so long as they are linearly independent, we can instead use the so-called {\em zero-forcing} (interference nulling) equalizer
\[
{\bf S}^{\dagger} {\bf Y} = ({\bf S}^T {\bf S})^{-1} {\bf S}^T {\bf Y} = ({\bf S}^T {\bf S})^{-1} {\bf S}^T {\bf S} {\bf B}^T = {\bf B}^T,
\]
and recover both streams of bits. This is the basic idea behind DS-CDMA: it is possible to unmix the user transmissions when we know the spreading codes and these are linearly independent (this requires $L \geq$ the number of users/transmitters). However, unmixing is doomed to fail when we do not know ${\bf S}$, because rank-two or higher matrix factorization is not unique in general.

The first application of CPD to (communication) signal processing was exactly in bypassing this seemingly insurmountable problem, which is of interest in non-cooperative communications \cite{SidGiaBro00}. Suppose that we have two (or more) receivers simultaneously monitoring the same band of interest from different locations. Let ${\bf H}(i,f)$ denote the path loss from the $f$-th transmitter to the $i$-th receiver. Then
\[
{\bf Y}_1 = \left[ {\bf s}_1,~{\bf s}_2\right] \left[ \begin{array}{lr}
                                                      {\bf H}(1,1) & 0 \\
                                                      0 & {\bf H}(1,2)
                                                    \end{array}
\right] \left[ {\bf b}_1,~{\bf b}_2\right]^T = {\bf S} {\bf D}_1({\bf H}) {\bf B}^T,
\]
\[
{\bf Y}_2 = \left[ {\bf s}_1,~{\bf s}_2\right] \left[ \begin{array}{lr}
                                                      {\bf H}(2,1) & 0 \\
                                                      0 & {\bf H}(2,2)
                                                    \end{array}
\right] \left[ {\bf b}_1,~{\bf b}_2\right]^T = {\bf S} {\bf D}_2({\bf H}) {\bf B}^T,
\]
or
\[
\Yt(i,j,k) := {\bf Y}_i(j,k) = \sum_{f=1}^2 {\bf S}(j,f) {\bf H}(i,f) {\bf B}(k,f),
\]
a CPD model of rank $F=2$. When there are more users, we obtain a CPD model of higher rank. The key point here is that the link to CPD allows recovering everything (spreading codes, information bits, and path losses) for all transmitters, up to the inherent (user) permutation and scaling / counter-scaling ambiguities of CPD. This is true even if there are more co-channel transmitters than the length of the spreading codes, $L$, so long as one of the CPD identifiability conditions is satisfied. The conceptual development presented here hides practical concerns, such as multipath, noise, imperfect synchronization, etc. There has been considerable follow-up work to address some of these issues, such as \cite{905948}. 

\subsection{Blind source separation}

Let us consider the model ${\bf y}_n = {\bf A} {\bf s}_n$, $n \in \left\{1, 2, \cdots \right\}$, where ${\bf y}_n$ is $I \times 1$, ${\bf s}_n$ is $F \times 1$, and we ignore additive noise for simplicity of exposition. We adopt {\em one} of the following two assumptions.

\noindent {\sf Assumption A1): Uncorrelated sources of time-varying powers.} In this case,
\[
{\bf R}_n := E \left[ {\bf y}_n {\bf y}_n^T \right] = {\bf A} E \left[ {\bf s}_n {\bf s}_n^T \right] {\bf A}^T =
\]
\[
{\bf A} \left[ \begin{array}{lcr}
                 E[({\bf s}_n(1))^2] & & {\bf 0} \\
                 & \ddots &\\
                 {\bf 0} & & E[({\bf s}_n(F))^2]
               \end{array}
\right] {\bf A}^T = {\bf A} {\bf D}_n {\bf A}^T.
\]
If we assume that the average source powers remain approximately constant over ``dwells'', and then switch to a different ``state'', then we can estimate
\[
\widehat {\bf R}_n :=  \frac{1}{N} \sum_{m=N(n-1)+1}^{Nn} {\bf y}_m {\bf y}_m^T,
\]
yielding a partially symmetric CPD model, which is particularly well-suited for speech signal separation using an array of microphones \cite{NioMokSidPot08}. Given ${\bf A}$, one can use its pseudo-inverse to recover the sources if ${\bf A}$ is tall, else it is possible to mitigate crosstalk using various beamforming strategies. When the speakers are moving, we can even track ${\bf A}$ over time using adaptive CPD approaches \cite{NioSid09}.

\noindent {\sf Assumption A2): Uncorrelated jointly WSS sources having different power spectra.} In this case, we rely instead on correlations at different lags, i.e.,
\[
{\bf R}_{n,\ell} := E \left[ {\bf y}_n {\bf y}_{n-\ell}^T \right] = {\bf A} E \left[ {\bf s}_n {\bf s}_{n-\ell}^T \right] {\bf A}^T =
\]
(since different sources are uncorrelated)
\[
{\bf A} \left[ \begin{array}{lcr}
                 E[{\bf s}_n(1) {\bf s}_{n-\ell}(1)] & & {\bf 0} \\
                 & \ddots &\\
                 {\bf 0} & & E[{\bf s}_n(F) {\bf s}_{n-\ell}(F) ]
               \end{array}
\right] {\bf A}^T =
\]
(since each source is wide-sense stationary (WSS))
\[
{\bf A} \left[ \begin{array}{lcr}
                 r_1(\ell) & & {\bf 0} \\
                 & \ddots &\\
                 {\bf 0} & & r_F(\ell) ]
               \end{array}
\right] {\bf A}^T = {\bf A} {\bf D}_{\ell}({\bf R}) {\bf A}^T,
\]
where ${\bf R}$ is the $L \times F$ (lags considered $\times$ sources) matrix holding the autocorrelation vector for all the sources. The power spectrum is the Fourier transform of the autocorrelation vector, hence different spectra are associated with different autocorrelation vectors, providing the necessary diversity to distinguish the sources and estimate ${\bf A}$. In practice we use $\widehat {\bf R}_{\ell} :=  \frac{1}{N} \sum_{n=\ell+1}^{N+\ell} {\bf y}_n {\bf y}_{n-\ell}^T$. This approach to source separation was pioneered by Belouchrani \cite{554307}. It took some years to recognize that this is a special (partially symmetric) case of CPD. These two approaches are second-order variants of Independent Component Analysis \cite{ComJutBook}. Other assumptions are possible, and may involve higher-order statistics, which are by themselves higher-order tensors.

Tensor factorization can also be applied to source separation in the power spectrum domain \cite{7175044}, which has applications in radio astronomy and dynamic spectrum access.

\subsection{Harmonics}
Consider the harmonic mixture ${\bf Y} = {\bf V} {\bf S}^T$, where ${\bf V}$ is Vandermonde, i.e.,
\[
{\bf V} = \left[ \begin{array}{cccc}
                   1 & 1 & \cdots & 1 \\
                   a_1 & a_2 & \cdots & a_F \\
                   a_1^2 & a_2^2 & \cdots & a_F^2 \\
                   \vdots & \vdots & \vdots & \vdots
                 \end{array}
\right],
\]
where the {\em generators} in the second row can be real or complex. The model ${\bf Y} = {\bf V} {\bf S}^T$ seems unrelated to tensors, however, upon defining
\[
{\bf Y}_1 := {\bf Y}(1:end-1,:)~~~\text{(all rows except last)}
\]
\[
{\bf Y}_2 := {\bf Y}(2:end,:)~~~\text{(all rows except first)}
\]
${\bf V}_1 := {\bf V}(1:end-1,:)$, and ${\bf D} := \text{Diag}([a_1,~a_2,~\cdots,~a_F])$, it is easy to see that
\[
{\bf Y}_1 = {\bf V}_1 {\bf S}^T;~~~ {\bf Y}_2 = {\bf V}_1 {\bf D} {\bf S}^T,
\]
i.e., a two-slab CPD model with Vandermonde structure in one mode. If we instead take ${\bf Y}_1 := {\bf Y}(1:end-2,:)$,
${\bf Y}_2 := {\bf Y}(2:end-1,:)$, ${\bf Y}_3 := {\bf Y}(3:end,:)$, then we obtain a three-slab CPD model with Vandermonde structure in two modes. If ${\bf S}$ is also Vandermonde, then we get Vandermonde (harmonic) stucture in all three modes, leading to a {\em multidimensional harmonic retrieval} (MDHR) problem. There are deep connections between CPD and MDHR (and direction finding using linear and rectangular arrays), originally revealed in \cite{923759,JiaSidtB01}; see also \cite{NionSidRadar}. 

\subsection{Collaborative filtering - based recommender systems}

Switching gears, consider a users $\times$ movies ratings matrix ${\bf R}$ of size $I \times J$, and the bilinear model
${\bf R} \approx {\bf U} {\bf V}^T$, where ${\bf U}$ is $I \times F$, with $F \ll \min(i,J)$, and its $i$-th row contains a reduced-dimension latent description of user $i$, ${\bf V}$ is $J \times F$ and its $j$-th row contains a reduced-dimension latent description of movie $j$, and the model ${\bf R} = {\bf U} {\bf V}^T$ implies that user $i$'s rating of movie $j$ is approximated as ${\bf R}(i,j) \approx {\bf U}(i,:) ({\bf V}(j,:))^T$, i.e., the inner product of the latent descriptions of the $i$-th user and the $j$-th movie. The premise of this type of modeling is that every user is a linear combination of $F$ (few) user ``types'' (e.g., child, college student, ... - these correspond to rows of ${\bf V}^T$ / columns of ${\bf V}$); and every movie is a linear combination of few movie types (e.g., comedy, drama, documentary, ... - these correspond to columns of ${\bf U}$). Typically, only a very small percentage of the entries of ${\bf R}$ is available -- between 1 per thousand and 1 per $10^5$ in practice. Recommender systems try to predict a user's missing ratings using not only that user's past ratings but also the ratings of all other users -- hence the term {\em collaborative filtering}. Notice that, if we can find ${\bf U}$ and ${\bf V}$ from the available ratings, then we can impute the missing ones using inner products of columns of ${\bf U}$ and ${\bf V}$. This suggests using the following formulation.
\[
\min_{{\bf U},~{\bf V}} \left| \left| {\bf W} * \left( {\bf R} - {\bf U} {\bf V}^T \right) \right| \right|_F^2,
\]
where ${\bf W}(i,j)=1$ if ${\bf R}(i,j)$ is available, $0$ otherwise. In practice it is unclear what would be a good $F$, so we typically over-estimate it and then use a rank penalty to control over-fitting and improve generalization ability. The rank of ${\bf X}$ is equal to the number of nonzero singular values of ${\bf X}$, and the nuclear norm $||{\bf X}||_*$ (sum of singular values) is a commonly used convex surrogate for rank ($||\cdot||_1$ vs. $||\cdot||_0$ of the vector of singular values). It has been shown \cite{doi:10.1137/070697835} that
\[
||{\bf X}||_* = \min_{{\bf U},{\bf V}~|~{\bf X} = {\bf U} {\bf V}^T} \frac{1}{2} \left( ||{\bf U}||_F^2 +  ||{\bf V}||_F^2 \right),
\]
giving rise to the following formulation
\[
\min_{{\bf U},~{\bf V}} \left| \left| {\bf W} * \left( {\bf R} - {\bf U} {\bf V}^T \right) \right| \right|_F^2 + \frac{\lambda}{2} \left( ||{\bf U}||_F^2 +  ||{\bf V}||_F^2 \right).
\]

The above ``flattened'' matrix view of ratings hides the fact that additional information on the context in which the ratings were given is often available. This may range from time stamps to detailed information regarding social context, etc. Every different type of context can be treated as one additional mode, giving rise to (very sparse) higher-order tensors. Taking time stamp (rating time) as an example, consider the user $\times$ movie $\times$ time tensor $\Rt$ with elements $\Rt(i,j,k)$. This can be modeled using CPD as
\[
\min_{{\bf A},~{\bf B},~{\bf C}} \sum_{k=1}^K \left| \left| {\Wt}(:,:,k) * \left( {\Rt}(:,:,k) - {\bf A} {\bf D}_k({\bf C}) {\bf B}^T \right) \right| \right|_F^2,
\]
where we have switched variables to the familiar ones for CPD. We can use similar rank regularization surrogates in the tensor case as well. We may also impose smoothness in the temporal mode (columns of ${\bf C}$) to reflect our expectation that user preferences change slowly over time. These have been considered by Xiong {\em et al.} in \cite{doi:10.1137/1.9781611972801.19} from a Bayesian point of view, which also proposed a probabilistic model for the hyper-parameters coupled with Markov Chain Monte-Carlo (MCMC) techniques for automated parameter tuning. At about the same time, Karatzoglou {\em et al.} \cite{Karatzoglou:2010:MRN:1864708.1864727} used age as the third mode, binned into three groups: under 18, 18-50, and over 50. More generally, they proposed adding a new mode for every piece of contextual information provided. The problem with this is that the more modes one adds, the sparser the resulting tensor, and very sparse tensors require high rank to model (recall that a diagonal matrix with nonzero entries on the diagonal is full rank). Karatzoglou {\em et al.} proposed using a non-orthogonal Tucker model instead of CPD. In particular, they proposed 
\[
\min_{{\bf U},{\bf V},{\bf W}, \Gt} \left| \left| \Gammat * \left( {\Rt} - ({\bf U}, {\bf V}, {\bf W}, \Gt) \right) \right| \right|_F^2 +~~~~~~~~~~~~~~~~~~
\]
\[
~~~~~\lambda \left( ||{\bf U}||_F^2 +  ||{\bf V}||_F^2 + ||{\bf W}||_F^2 \right) + \mu ||\Gt||_F^2,
\]
where $\Gammat(i,j,k)=1$ if $\Rt(i,j,k)$ is available, 0 otherwise; $({\bf U}, {\bf V}, {\bf W}, \Gt)$ stands for the Tucker model that is generated by ${\bf U}, {\bf V}, {\bf W}, \Gt$; and $||\Xt||_F^2$ is the sum of squared elements of tensor $\Xt$. Note that orthogonality is not imposed on ${\bf U}, {\bf V}, {\bf W}$ {\em because it is desired to penalize the model's rank} -- so constraints of type ${\bf U}^T {\bf U} = {\bf I}$ cannot be imposed (recall that $||{\bf U}||_F^2 = \text{Tr}({\bf U}^T {\bf U})$).

In recommender systems one may additionally want to exploit side information such as a user similarity matrix that cannot simply be considered as an extra slice of the tensor $\Rt$. In such cases, {\em coupled} matrix-tensor decomposition (possibly involving several matrices and tensors) can be used. There are many possibilities and design choices in this direction; we refer the reader to \cite{AcKoDu11,Acar201353,SDM2014}, and \cite{SDF} which introduces a domain specific language for fast prototyping.

\begin{table}[t]
\colorbox{lightgray}{
\begin{minipage}{.48\textwidth}\normalsize
{\bf Multilinear maps for classification:} In support vector machine (SVM) classification, we use a linear mapping ${\bf w}^T {\bf x} = \sum_i {\bf w}(i) {\bf x}(i)$ to discriminate between vectors belonging to two different classes. When the classes are not linearly separable, one can employ a bilinear mapping ${\bf x}^T {\bf W} {\bf x} = \sum_{i,j} {\bf W}(i,j) {\bf x}(i) {\bf x}(j) = \text{vec}\left( {\bf x}^T {\bf W} {\bf x} \right) = \left({\bf x}^T \otimes {\bf x}^T\right) \text{vec}({\bf W}) = \left({\bf x} \otimes {\bf x} \right)^T \text{vec}({\bf W})$, or even a multilinear one $\left({\bf x} \otimes {\bf x} \otimes {\bf x} \right)^T \text{vec}({\Wt}) = \sum_{i,j,k} \Wt(i,j,k) {\bf x}(i) {\bf x}(j) {\bf x}(k)$. Notice that by augmenting ${\bf x}$ with a unit as last element (i.e., replacing ${\bf x}$ by $[{\bf x}, 1]^T$), higher-order mappings include lower-order ones, hence it suffices to consider the highest order. In such cases, the classifier design problem boils down to designing a suitable matrix or tensor $\Wt$ of weights. In order to keep the number of model parameters low relative to the number of training samples (to enable statistically meaningful learning and generalization), a low-rank tensor model such as CPD \cite{5694074} or low multilinear rank one such as Tucker can be employed, and the model parameters can be learned using a measure of classification error as the cost function. A simple optimization solution is to use SGD, drawing samples from the training set at random.
\end{minipage}}
\vspace{-15pt}
\end{table}

\subsection{Gaussian mixture parameter estimation}
Consider $F$ Gaussians ${\cal N}(\bm{\mu}_f, \sigma_f^2 {\bf I})$, where $\bm{\mu}_f \in \Real^{I \times 1}$ is the mean vector and $\sigma_f^2$ is the variance of the elements of the $f$-th Gaussian. Let $\bm{\pi} = [\pi_1,\cdots,\pi_F]^T$ be a prior distribution, and consider the following experiment: first draw $f_m \sim \bm{\pi}$; then draw ${\bf x}_m \sim {\cal N}(\bm{\mu}_{f_m}, \sigma_{f_m}^2 {\bf I})$. The distribution of ${\bf x}_m$ is then a mixture of the $F$ Gaussians, i.e.,
$\sum_{f=1}^F \pi_f {\cal N}(\bm{\mu}_f, \sigma_f^2 {\bf I})$. Now run $M$ independent trials of the above experiment to create $\left\{ {\bf x}_m \right\}_{m=1}^M$. Given $\left\{ {\bf x}_m \right\}_{m=1}^M$, the problem of interest is to estimate the mixture parameters $\{ \bm{\mu}_f,\sigma_f^2,\pi_f \}_{f=1}^F$. Note that it is possible to estimate $F$ from $\left\{ {\bf x}_m \right\}_{m=1}^M$, but we will assume it given for the purposes of our discussion. Note the conceptual similarity of this problem and $k$-means (here: $F$-means) clustering or vector quantization (VQ): the main difference is that here we make an additional modeling assumption that the ``point clouds'' are isotropic Gaussian about their means. Let us consider
\[
E[{\bf x}_m] = \sum_{f_m=1}^F E[{\bf x}_m|f_m] \pi_{f_m} = \sum_{f=1}^F \bm{\mu}_{f} \pi_{f} = {\bf M} \bm{\pi},
\]
where ${\bf M} := \left[\bm{\mu}_1,\cdots,\bm{\mu}_F \right]$ ($I \times F$). Next, consider
\begin{align*}
E[{\bf x}_m {\bf x}_m^T] &= \sum_{f_m=1}^F E[{\bf x}_m {\bf x}_m^T|f_m] \pi_{f_m} =
\sum_{f=1}^F \left( \bm{\mu}_{f} \bm{\mu}_{f}^T + \sigma_f^2 {\bf I} \right) \pi_{f}\\
&= {\bf M} \text{Diag}(\bm{\pi}) {\bf M}^T + \bar{\sigma}^2 {\bf I},~\bar{\sigma}^2 := \sum_{f=1}^F \sigma_f^2 \pi_f.
\end{align*}
It is tempting to consider third-order moments, which are easier to write out in scalar form
\[
E[{\bf x}_m(i) {\bf x}_m(j) {\bf x}_m(k)] = \sum_{f=1}^F E[{\bf x}_m(i) {\bf x}_m(j) {\bf x}_m(k)|f] \pi_f.
\]
Conditioned on $f$,  ${\bf x}_m(i) = \bm{\mu}_f(i) + {\bf z}_m(i)$, where ${\bf z}_m \sim {\cal N}({\bf 0}, \sigma_f^2 {\bf I})$, and likewise for ${\bf x}_m(j)$ and ${\bf x}_m(k)$. Plugging these back into the above expression, and using that

\noindent $\bullet$ If two out of three indices $i,j,k$ are equal, then $E[{\bf z}_m(i) {\bf z}_m(j) {\bf z}_m(k)|f]=0$, due to zero mean and independence of the third; and

\noindent $\bullet$ If all three indices are equal, then $E[{\bf z}_m(i) {\bf z}_m(j) {\bf z}_m(k)|f]=0$ because the third moment of a zero-mean Gaussian is zero, we obtain
\begin{align*}
& E[{\bf x}_m(i) {\bf x}_m(j) {\bf x}_m(k)|f] = \bm{\mu}_f(i) \bm{\mu}_f(j) \bm{\mu}_f(k) + \\
& \sigma_f^2 \left( \bm{\mu}_f(i) \delta(j-k) + \bm{\mu}_f(j) \delta(i-k) + \bm{\mu}_f(k) \delta(i-j) \right),
\end{align*}
where $\delta(\cdot)$ is the Kronecker delta. Averaging over $\pi_f$,
\begin{align*}
\Rt(i,j,k) \!:=\! E[{\bf x}_m(i) {\bf x}_m(j) {\bf x}_m(k)]
=\sum_{f=1}^F \pi_f \bm{\mu}_f(i) \bm{\mu}_f(j) \bm{\mu}_f(k) +\\
\sum_{f=1}^F \pi_f \sigma_f^2 \left( \bm{\mu}_f(i) \delta(j-k) + \bm{\mu}_f(j) \delta(i-k) + \bm{\mu}_f(k) \delta(i-j) \right).
\end{align*}
At this point, let us further assume, for simplicity, that $\sigma_f^2 = \sigma^2$, $\forall f$, and $\sigma^2$ is known. Then
$\sum_{f=1}^F \pi_f \bm{\mu}_f(i) = E[{\bf x}_m(i)]$ can be easily estimated. So we may pre-compute the second term in the above equation, call it $\Gammat(i,j,k)$, and form
\[
\Rt(i,j,k) - \Gammat(i,j,k)  =  \sum_{f=1}^F \pi_f \bm{\mu}_f(i) \bm{\mu}_f(j) \bm{\mu}_f(k),
\]
which is evidently a symmetric CPD model of rank (at most) $F$. Note that, due to symmetry and the fact that $\pi_f \geq 0$, there is no ambiguity regarding the sign of $\bm{\mu}_f$; but we can still set e.g., $\bm{\mu}_1^{'} = \rho^{\frac{1}{3}} \bm{\mu}_1$, $\pi_1^{'}=\frac{1}{\rho} \pi_1$, $\pi_2^{'}=\pi_1 + \pi_2 - \pi_1^{'} = \frac{\rho-1}{\rho} \pi_1 + \pi_2$, $\frac{1}{\gamma}=\frac{\pi_2^{'}}{\pi_2}$, and $\bm{\mu}_2^{'}= \gamma^{\frac{1}{3}} \bm{\mu}_2$, for some $\rho > 0$. However, we must further ensure that $\pi_2^{'} > 0$, and $\pi_1^{'} < \pi_1 + \pi_2$; both require $\rho > \frac{\pi_1}{\pi_1+\pi_2}$.
We see that scaling ambiguity remains, and is important to resolve it here, otherwise we will obtain the wrong means and mixture probabilities. Towards this end, consider lower-order statistics, namely $E[{\bf x}_m {\bf x}_m^T]$ and $E[{\bf x}_m]$. Note that,
\[
({\bf M} \odot {\bf M} \odot {\bf M}) \bm{\pi} = (({\bf M} {\bf D}^{1/3}) \odot ({\bf M} {\bf D}^{1/3}) \odot ({\bf M} {\bf D}^{1/3})) {\bf D}^{-1} \bm{\pi}
\]
but
\[
E[{\bf x}_m] = {\bf M} \bm{\pi} \neq ({\bf M} {\bf D}^{1/3}) {\bf D}^{-1} \bm{\pi},
\]
\[
E[{\bf x}_m {\bf x}_m^T] - \bar{\sigma}^2 {\bf I} = {\bf M} \text{Diag}(\bm{\pi}) {\bf M}^T
\]
\[
\stackrel{\text{vec}(\cdot)}{\rightarrow} ({\bf M} \odot {\bf M}) \bm{\pi} \neq (({\bf M} {\bf D}^{1/3}) \odot ({\bf M} {\bf D}^{1/3})) {\bf D}^{-1} \bm{\pi}.
\]
This shows that no scaling ambiguity remains when we jointly fit third and second (or third and first) order statistics. For the general case, when the variances $\left\{\sigma_f^2\right\}_{f=1}^F$ are unknown and possibly different, see  \cite{Hsu:2013:LMS:2422436.2422439}. A simpler work-around is to treat ``diagonal slabs'' (e.g., corresponding to $j=k$) as missing, fit the model, then use it to estimate $\left\{\sigma_f^2\right\}_{f=1}^F$ and repeat.

\subsection{Topic modeling}

Given a dictionary ${\cal D} = \left\{w_1,\cdots,w_I\right\}$ comprising $I$ possible words, a {\em topic} is a probability mass function (pmf) over ${\cal D}$. Assume there are $F$ topics overall, let ${\bf p}_f := \text{Pr}(w_i|f)$ be the pmf associated with topic $f$, $\pi_f$ be the probability that one may encounter a document associated with topic $f$, and $\bm{\pi}:=[\pi_1,\cdots,\pi_f]^T$. Here we begin our discussion of topic modeling by assuming that each document is related to one and only one topic (or, document ``type''). Consider the following experiment:

\noindent 1) Draw a document at random;

\noindent 2) Sample $m$ words from it, independently, and at random (with replacement -- and the order in which words are drawn does not matter);

\noindent 3) Repeat (until you collect ``enough samples'' -- to be qualified later).

Assume for the moment that $F$ is known. Your objective is to estimate $\left\{ {\bf p}_f, \pi_f \right\}_{f=1}^F$. Clearly, $\text{Pr}(w_i)=\sum_{f=1}^F \text{Pr}(w_i|f) \pi_f$; furthermore, the word co-occurrence probabilities
$\text{Pr}(w_i,w_j) := $Pr(word $i$ and word $j$ are drawn from the same document) satisfy
\[
\text{Pr}(w_i,w_j) = \sum_{f=1}^F \text{Pr}(w_i,w_j|f) \pi_f = \sum_{f=1}^F {\bf p}_f(i) {\bf p}_f(j) \pi_f,
\]
since the words are independently drawn from the document. Define the matrix of word co-occurrence probabilities ${\bf P}^{(2)}$ with elements ${\bf P}^{(2)}(i,j) := \text{Pr}(w_i,w_j)$, and the matrix of conditional pmfs ${\bf C} := [{\bf p}_1,\cdots,{\bf p}_F]$. Then
\[
{\bf P}^{(2)} = {\bf C} \text{Diag}(\bm{\pi}) {\bf C}^T.
\]
Next, consider ``trigrams'' -- i.e., probabilities of triples of words being drawn from the same document
\[
\text{Pr}(w_i,\!w_j,\!w_k) \!=\! \sum_{f=1}^F \text{Pr}(w_i,\!w_j,\!w_k|f) \pi_f \!=\!
\sum_{f=1}^F {\bf p}_f(i) {\bf p}_f(j) {\bf p}_f(k) \pi_f.
\]
Define tensor $\Pt^{(3)}$ with elements $\Pt^{(3)}(i,\!j,\!k) \!:=\! \text{Pr}(w_i,\!w_j,\!w_k)$. Then $\Pt^{(3)}$ admits a symmetric non-negative CPD model of rank (at most) $F$:
 \[
 \Pt^{(3)} = ({\bf C} \odot {\bf C} \odot {\bf C}) \bm{\pi}.
 \]
Similar to\footnote{But in fact simpler from, since here, due to sampling with replacement, the same expression holds even if two or three indices $i,j,k$ are the same.} Gaussian mixture parameter estimation, we can estimate ${\bf C}$ and $\bm{\pi}$ from the tensor $\Pt^{(3)}$ and the matrix ${\bf P}^{(2)}$. In reality, we will use empirical word co-occurrence counts to estimate $\Pt^{(3)}$ and ${\bf P}^{(2)}$, and for this we need to sample enough triples (``enough samples''). Once we have ${\bf C}$, we can classify any document by estimating (part of) its conditional word pmf and comparing it to the columns of ${\bf C}$.

Next, consider the more realistic situation where each document is a mixture of topics, modeled by a pmf ${\bf q}$ $(F \times 1)$ that is itself drawn from a distribution $\delta(\cdot)$ over the $(F-1)$-dimensional {\em probability simplex} -- see Fig. \ref{fig-nikos7}.
\begin{figure}
\vspace{-10pt}
\includegraphics[height=2in]{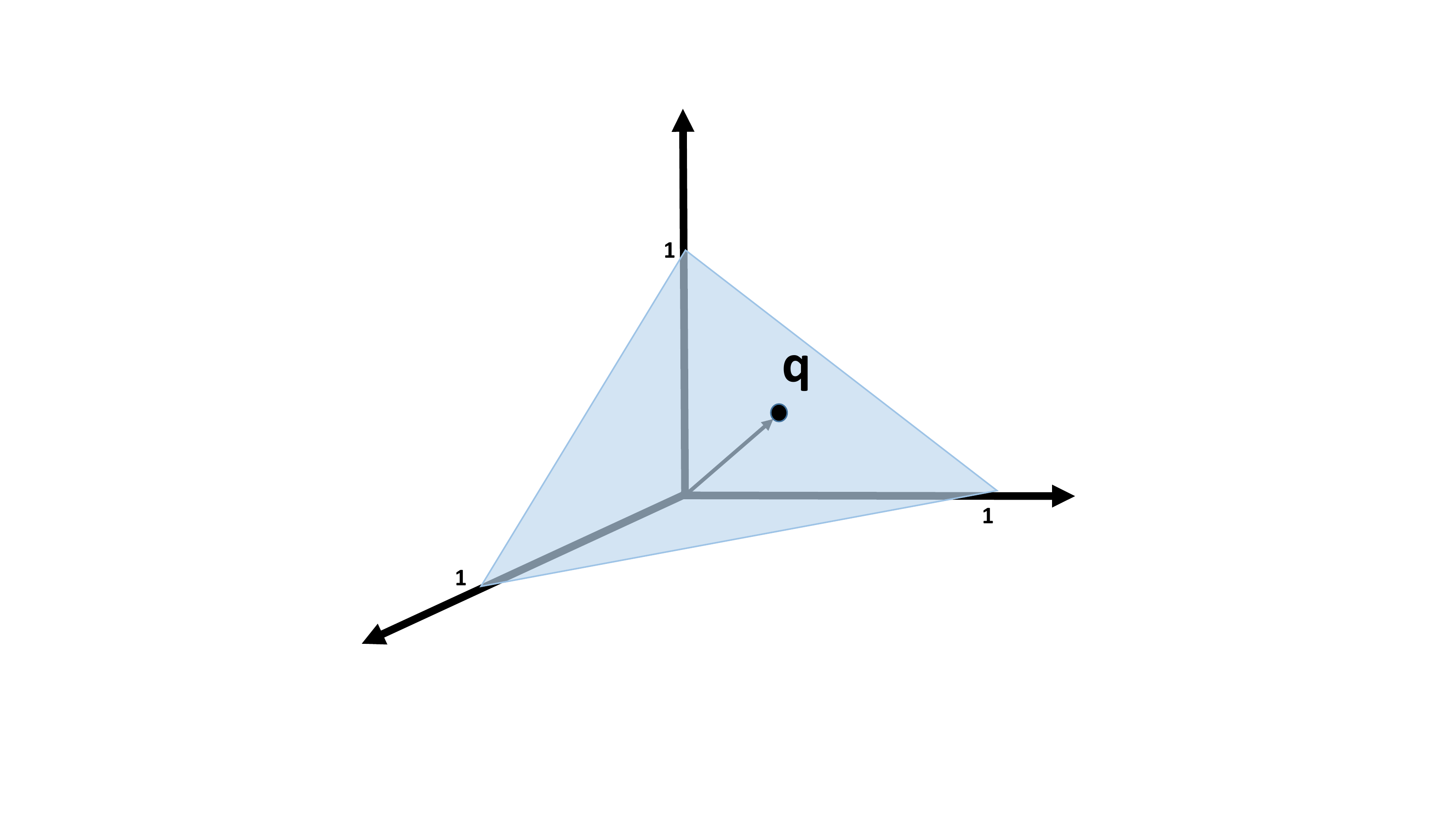}
\vspace{-40pt}
\caption{$2$-D probability simplex in $3$-D space.}\label{fig-nikos7}
\vspace{-10pt}
\end{figure}
Our working experiment is now modified as follows.

\noindent 1) For every document we sample, we draw ${\bf q} \sim \delta(\cdot)$;

\noindent 2) For every word we sample from the given document, we first draw a topic $t$ from ${\bf q}$ -- i.e., topic $f$ is selected with probability ${\bf q}(f)$;

\noindent 3) Next, we draw a word $\sim {\bf p}_t$;

\noindent 4) Goto 2, until you have sampled the desired number of words (e.g., 3) from the given document;

\noindent 5) Goto 1, until you have collected enough samples (e.g., enough triples).

Then,
\begin{align*}
&\text{Pr}(w_i,w_j|t_1,t_2,{\bf q}) = {\bf p}_{t_1}(i) {\bf p}_{t_2}(j) \Longrightarrow\\
&\text{Pr}(w_i,w_j|{\bf q}) = \sum_{t_1=1}^F \sum_{t_2=1}^F  {\bf p}_{t_1}(i) {\bf p}_{t_2}(j) {\bf q}(t_1) {\bf q}(t_2) \Longrightarrow\\
&\text{Pr}(w_i,w_j) = \sum_{t_1=1}^F \sum_{t_2=1}^F  {\bf p}_{t_1}(i) {\bf p}_{t_2}(j) E\left[{\bf q}(t_1) {\bf q}(t_2)\right],
\end{align*}
where we notice that what comes into play is the second-order statistics $E\left[{\bf q}(t_1) {\bf q}(t_2)\right]$ (the correlation) of $\delta(\cdot)$. Likewise, it follows that, for the trigrams, $\text{Pr}(w_i,w_j,w_k) =$
\[
\sum_{t_1=1}^F \sum_{t_2=1}^F \sum_{t_3=1}^F  {\bf p}_{t_1}(i) {\bf p}_{t_2}(j)  {\bf p}_{t_3}(k) E\left[{\bf q}(t_1) {\bf q}(t_2) {\bf q}(t_3)\right],
\]
which involves the third-order statistics tensor $\Gt$ of $\delta(\cdot)$ with elements $\Gt(i,j,k) := E\left[{\bf q}(t_1) {\bf q}(t_2) {\bf q}(t_3)\right]$. Defining the $I \times I \times I$ tensor $\Pt$ with elements $\Pt(i,j,k) := \text{Pr}(w_i,w_j,w_k)$, it follows that $\Pt$ admits a symmetric Tucker decomposition, $\Pt = \text{Tucker}({\bf C},{\bf C},{\bf C},\Gt)$, with ${\bf C} = [{\bf p}_1,\cdots,{\bf p}_F]$. Note that ${\bf C}$ is element-wise non-negative, but in principle $\Gt$ may have negative elements. As we know, Tucker models are not identifiable in general -- there is linear transformation freedom. This can be alleviated when one can assume sparsity in ${\bf C}$ \cite{DBLP:journals/corr/AnandkumarHJK13}, $\Gt$, or both (intuitively, this is because linear transformations generally do not preserve sparsity).

\subsection{Multilinear discriminative subspace learning}

Consider the following {\em discriminative subspace learning} problem: given ${\bf X} = [{\bf x}_1,\cdots,{\bf x}_M]$ ($N \times M$) and associated class labels ${\bf z} = [z_1,\cdots,z_M]$ ($1 \times M$) for the columns of ${\bf X}$, find a dimensionality-reducing linear transformation ${\bf U}$ of size $N \times F$, $F < N$ (usually $F \ll N$) such that
\[
\min_{{\bf U}|{\bf U}^T {\bf U}={\bf I}} \sum_{m=1}^{M} \left\{ (1-\lambda) \sum_{\ell=1|z_{\ell} = z_m}^M ||{\bf U}^T {\bf x}_m  - {\bf U}^T {\bf x}_{\ell}||_2^2 - \right.
\]
\[
\left.
\lambda \sum_{\ell=1|z_{\ell} \neq z_m}^M ||{\bf U}^T {\bf x}_m  - {\bf U}^T {\bf x}_{\ell}||_2^2 \right\},
\]
where the first (second) term measures the within-class (across-class) distance in reduced dimension space. We are therefore trying to find a dimensionality-reducing transformation that will map points close in terms of Euclidean distance if they have the same class label, far otherwise. Another way to look at it is that we are trying to find a subspace to project onto where we can easily visualize (if $F=2$ or $3$) the point clouds of the different classes. Upon defining
\[
w_{m,\ell} := (1-\lambda)^{1(z_{\ell}=z_m)} (-\lambda)^{1-1(z_{\ell}=z_m)},
\]
where ${1(z_{\ell}=z_m)}=1$ if $z_{\ell}=z_m$, $0$ otherwise, we can compactly write the problem as follows
\[
\min_{{\bf U}|{\bf U}^T {\bf U}={\bf I}} \sum_{m=1}^{M} \sum_{\ell=1}^M ||{\bf U}^T {\bf x}_m  - {\bf U}^T {\bf x}_{\ell}||_2^2 w_{m,\ell}.
\]
Expanding the squared norm and using properties of $\text{Tr}(\cdot)$, we can write the cost function as
\[
\sum_{m=1}^{M} \sum_{\ell=1}^M ||{\bf U}^T {\bf x}_m  - {\bf U}^T {\bf x}_{\ell}||_2^2 w_{m,\ell} = \text{Tr}({\bf U} {\bf U}^T {\bf Y}),
\]
where
\[
{\bf Y} := \sum_{m=1}^{M} \sum_{\ell=1}^M  w_{m,\ell} ({\bf x}_m - {\bf x}_{\ell}) ({\bf x}_m - {\bf x}_{\ell})^T.
\]
Notice that $w_{m,\ell} = w_{\ell,m}$ by definition, and ${\bf Y}$ is symmetric. Let ${\bf Y} = {\bf V} \bm{\Lambda} {\bf V}^T$ be the eigendecomposition of ${\bf Y}$, and note that $\text{Tr}({\bf U} {\bf U}^T {\bf Y})=\text{Tr}({\bf U}^T {\bf Y} {\bf U} )$. Clearly, ${\bf U}_{\text{opt}}$ $=$ $F$ minor eigenvectors of ${\bf Y}$ (columns of ${\bf V}$ corresponding to the $F$ smallest elements on the diagonal of $\bm{\Lambda}$).

Now, suppose that the columns in ${\bf X}$ are in fact vectorized tensors. As an example, suppose that there exist {\em common} bases ${\bf U}$ ($I \times r_1$), ${\bf V}$ ($J \times r_2$), ${\bf W}$ ($K \times r_3$), such that
\[
{\bf x}_m \approx ({\bf U} \otimes {\bf V} \otimes {\bf W}) {\bf g}_m,~~~\forall m \in \left\{ 1,\cdots,M \right\},
\]
i.e., each ${\bf x}_m$ can be modeled using a $\perp$-Tucker model with common mode bases, but different cores for different $m$.  We can think of $({\bf U} \otimes {\bf V} \otimes {\bf W})^T$ ($r_1 r_2 r_3 \times IJK$) as a (Kronecker) structured dimensionality reducing transformation, and the vectorized core array ${\bf g}_m$ as the low-dimensional ($r_1 r_2 r_3 \times 1$) representation of ${\bf x}_m$. We want to find ${\bf U}$, ${\bf V}$, ${\bf W}$ such that the ${\bf g}$'s corresponding to ${\bf x}$'s in the same (different) class are close (far) from each other. Following the same development as before, using
\[
\widehat {\bf g}_m = ({\bf U} \otimes {\bf V} \otimes {\bf W})^T {\bf x}_m
\]
as the projection of ${\bf x}_m$ in reduced-dimension space, we arrive at
\[
\min_{{\bf U},{\bf V},{\bf W}} \text{Tr}\left( ({\bf U} \otimes {\bf V} \otimes {\bf W}) ({\bf U} \otimes {\bf V} \otimes {\bf W})^T {\bf Y}\right),
\]
\[
\text{subject to:}~ {\bf U}^T {\bf U}={\bf I},~{\bf V}^T {\bf V}={\bf I},~{\bf W}^T {\bf W}={\bf I},
\]
or, equivalently,
\[
\min_{{\bf U},{\bf V},{\bf W}} \text{Tr}\left( (({\bf U} {\bf U}^T) \otimes ({\bf V} {\bf V}^T) \otimes ({\bf W} {\bf W}^T)) {\bf Y}\right),
\]
\[
\text{subject to:}~ {\bf U}^T {\bf U}={\bf I},~{\bf V}^T {\bf V}={\bf I},~{\bf W}^T {\bf W}={\bf I},
\]
from which it is clear that, conditioned on, say, ${\bf U}$ and ${\bf V}$, the update with respect to ${\bf W}$ boils down to
\[
\min_{{\bf W}|{\bf W}^T {\bf W}={\bf I}}
\text{Tr}\left( {\bf W} {\bf W}^T {\bf Z}\right),
\]
for some matrix ${\bf Z}$ that depends on the values of ${\bf U}$ and ${\bf V}$. See \cite{4032832,Lu:2011:SMS:1950989.1951185} for more on the topic of multilinear subspace learning.

The applications that we reviewed in some depth are by no means exhaustive -- there are a lot more success stories using tensors for data mining and machine learning, e.g., for higher-order web link analysis \cite{KoBa06}, and spotting misbehaviors in location-based social networks \cite{Papalexakis:2014:SML:2567948.2576950}; see also \cite{Papalexakis:2012:PSP:2405473.2405521}.

\section{Software, demos, history, and what lies ahead}

As we wrap up this admittedly long article, we would like to point out some widely available resources that can help bring the reader up to speed experimenting with tensors in minutes. Matlab provides native support for tensor objects, but working with tensors is facilitated by these freely available toolboxes:
\begin{enumerate}
\item The {\tt n-way toolbox} \url{http://www.models.life.ku.dk/nwaytoolbox} by Bro {\it et al.} \cite{andersson2000n}, based on ALS (with Gauss-Newton, line-search and other methods as an option) incorporates many {\em non-parametric} types of constraints, such as non-negativity;
\item The {\tt tensor toolbox} \url{http://www.sandia.gov/~tgkolda/TensorToolbox/index-2.6.html} by Kolda {\it et al.} \cite{bader2007efficient,TTB_Software} was the first to provide support for sparse, dense, and factored tensor classes, alongside standard ALS and all-at-once algorithms for CPD, MLSVD, and other factorizations;
\item {\tt Tensorlab} \url{http://www.tensorlab.net/} by De Lathauwer {\it et al.} \cite{tensorlab}, builds upon the complex optimization framework and offers numerical algorithms for computing CPD, MLSVD and more general block term decompositions. It includes a library of constraints and regularization penalties and offers the possibility to combine and jointly factorize dense, sparse, structured and incomplete tensors. It provides special routines for large-scale problems and visualization.  
\item {\tt SPLATT} \url{http://glaros.dtc.umn.edu/gkhome/splatt/overview} by Smith {\it et al.} is a high-performance computing software toolkit for parallel sparse tensor factorization. It contains memory- and operation-efficient algorithms that allows it to compute PARAFAC decompositions of very large sparse datasets. SPLATT is written in C and OpenMP.
\item The {\tt TensorPackage}
\url{http://www.gipsa-lab.fr/~pierre.comon/TensorPackage/tensorPackage.html} by Comon {\it et al.}, which includes various algorithms for CPD and employs enhanced line search \cite{doi:10.1137/06065577}.
\end{enumerate}
While these toolboxes are great to get you going and for rapid prototyping, when it comes to really understanding what you're doing with tensors, there is nothing as valuable as programming ALS for CPD and $\perp$-Tucker yourself, and trying them on real data. Towards this end, we have produced educational ``plain-vanilla'' programs (CPD-ALS, MLSVD, $\perp$-Tucker-ALS, CPD-GD, CPD-SGD), and simple but instructive demos (multichannel speech separation, and faces tensor compression) which are provided as supplementary material together with this article.

Tensor decomposition has come a long way since Hitchcock '27, \cite{SAPM:SAPM192761164}, Cattell '44 \cite{Cattell1944}, and later Harshman '70-'72 \cite{Har70,Har72}, Carroll and Chang \cite{CarCha70}, and Kruskal's '77 \cite{Kru77} seminal papers. It is now a vibrant field that is well-represented in major IEEE, ACM, SIAM, and other mainstream conferences. The cross-disciplinary community that nurtured tensor decomposition research during the years that it was a niche area has two dedicated workshops that usually happen every three years: the TRICAP (Three-way methods In Chemistry and Psychology) workshop, which was last organized in 2015 at Pecol -- Val di Zoldo (Belluno), Italy \url{http://people.ece.umn.edu/~nikos/TRICAP_home.html}; and the TDA (Tensor Decompositions and Applications) workshop, which was last organized in 2016 at Leuven, Belgium \url{http://www.esat.kuleuven.be/stadius/TDA2016/}.

In terms of opportunities and challenges ahead, we see the need for more effective and tractable tensor rank detection criteria, and flexible and scalable algorithms that can handle very big datasets while offering identifiability, convergence, and parameter RMSE performance guarantees -- at least under certain reasonable conditions. Data fusion, in the form of coupled decomposition of several related tensors and matrices is a very promising direction with numerous applications. More broadly, we believe that machine learning offers a wide range of potential applications, and this is one of the reasons why we wrote this article. Tensor decomposition in higher rank blocks is another interesting but relatively under-explored area with many applications. Finally, using multilinear models such as tensor trains as ``universal approximants'' has been gaining traction and will continue to grow in the foreseeable future, as a way to get away from the ``curse of dimensionality''.

\bibliographystyle{IEEEbib}
\bibliography{refs/NikosRefs,refs/crb,refs/extrarefsDec13}

\newpage
\section{Supplementary Material}

\subsection{Detailed proof of uniqueness via eigendecomposition}
Consider an $I \times J \times 2$ tensor $\Xt$ of rank $F \leq \min(I,J)$. Let us look at the two frontal slabs of $\Xt$. Since $\text{rank}(\Xt)=F$, it follows that
\[
{\bf X}^{(1)} = \Xt(:,:,1)={\bf A} {\bf D}_1({\bf C}) {\bf B}^T,
\]
\[
{\bf X}^{(2)} = \Xt(:,:,2)={\bf A} {\bf D}_2({\bf C}) {\bf B}^T,
\]
where ${\bf A}$, ${\bf B}$, ${\bf C}$ are $I \times F$, $J \times F$, and $2 \times F$, respectively. Assume, for the moment, that there is no zero element in ${\bf D}_1({\bf C})$ or ${\bf D}_2({\bf C})$. Define $\widetilde {\bf A} := {\bf A} {\bf D}_1({\bf C})$, and ${\bf D} := ({\bf D}_1({\bf C}))^{-1} {\bf D}_2({\bf C})$. Then,
\[
{\bf X}^{(1)} = \widetilde {\bf A} {\bf B}^T,~~~
{\bf X}^{(2)} = \widetilde {\bf A} {\bf D} {\bf B}^T,
\]
or
\[
\left[
\begin{array}{c}
{\bf X}^{(1)}\\
{\bf X}^{(2)}\\
\end{array}
\right] = \left[
\begin{array}{l}
\widetilde {\bf A}\\
\widetilde {\bf A} {\bf D}\\
\end{array}
\right] {\bf B}^T.
\]
With $[{\bf U},\bSigma,{\bf V}]=\text{svd}\left(\left[
\begin{array}{c}
{\bf X}^{(1)}\\
{\bf X}^{(2)}\\
\end{array}
\right]\right)$, i.e.,
$\left[
\begin{array}{c}
{\bf X}^{(1)}\\
{\bf X}^{(2)}\\
\end{array}
\right] = {\bf U} \bSigma {\bf V}^T,
$
and assuming that $\text{rank}\left({\bf X}^{(1)}\right)=\text{rank}\left({\bf X}^{(2)}\right)=F$ (which implies that the rank of all matrices involved is $F$), we have
\[
{\bf U} =
\left[
\begin{array}{c}
{\bf U}_1\\
{\bf U}_2\\
\end{array}
\right] = \left[
\begin{array}{l}
\widetilde {\bf A}\\
\widetilde {\bf A} {\bf D}\\
\end{array}
\right] {\bf M} = \left[
\begin{array}{l}
\widetilde {\bf A} {\bf M}\\
\widetilde {\bf A} {\bf D} {\bf M}\\
\end{array}
\right],
\]
where matrix ${\bf M}$ is $F \times F$ nonsingular. Compute auto- and cross-correlation matrices
\[
{\bf R}_1 = {\bf U}_1^T {\bf U}_1 = {\bf M}^T {\widetilde {\bf A}}^T \widetilde {\bf A} {\bf M}=: {\bf Q} {\bf M},
\]
\[
{\bf R}_2 = {\bf U}_1^T {\bf U}_2 = {\bf M}^T {\widetilde {\bf A}}^T \widetilde {\bf A} {\bf D} {\bf M}= {\bf Q} {\bf D} {\bf M}.
\]
Notice that both ${\bf R}_1$ and ${\bf R}_2$ are $F \times F$ nonsingular. So, what we have accomplished with these transformations is that we obtained a pair of equations involving square nonsingular matrices (instead of possibly tall, full column rank ones). It follows that
\[
\left({\bf R}_1^{-1} {\bf R}_2\right) {\bf M}^{-1} = {\bf M}^{-1} {\bf D},
\]
i.e., ${\bf M}^{-1}$ holds the eigenvectors of matrix $\left({\bf R}_1^{-1} {\bf R}_2\right)$, and ${\bf D}$ holds the corresponding eigenvalues (assumed to be distinct, for the moment). There is freedom to scale eigenvectors (they remain eigenvectors), and obviously one cannot recover the order of the columns of ${\bf M}^{-1}$. This means that there is permutation and scaling ambiguity in recovering ${\bf M}^{-1}$ from eigendecomposition of $\left({\bf R}_1^{-1} {\bf R}_2\right)$. That is, what we do recover is actually ${\widetilde {\bf M}}^{-1} = {\bf M}^{-1} \bPi \bLambda$, where $\bPi$ is a permutation matrix and $\bLambda$ is a nonsingular diagonal scaling matrix. If we use ${\widetilde {\bf M}}^{-1}$ to recover $\widetilde {\bf A}$ from equation ${\bf U}_1 = \widetilde {\bf A} {\bf M}$ $\Rightarrow$ $\widetilde {\bf A} = {\bf U}_1 {\bf M}^{-1}$, we will in fact recover $\widetilde {\bf A} \bPi \bLambda$ -- that is, $\widetilde {\bf A}$ up to the same column permutation and scaling that stem from the ambiguity in recovering ${\bf M}^{-1}$. It is now easy to see that we can recover ${\bf B}$ and ${\bf C}$ by going back to the original equations for ${\bf X}^{(1)}$ and ${\bf X}^{(2)}$ and left-inverting ${\bf A}$.

During the course of the derivation, we have made assumptions in passing: i) that the slabs of $\Xt$ have rank $F=\text{rank}(\Xt)$, and ii) that the eigenvalues in ${\bf D}$ are distinct ($\Rightarrow$ one row of ${\bf C}$ has no zero elements). We now revisit those working assumptions, starting from the last one, and show that they can be made without loss of generality. First note that $\left({\bf R}_1^{-1} {\bf R}_2\right)$ is diagonalizable (i.e., has a full set of linearly independent eigenvectors) {\em by construction} under our working assumptions. If two or more of its eigenvalues are identical though, then linear combinations of the corresponding eigenvectors are also eigenvectors, corresponding to the same eigenvalue. Hence distinct eigenvalues (elements of ${\bf D}$) are {\em necessary} for uniqueness.

Consider creating two random slab mixtures from the given slabs ${\bf X}^{(1)}$ and ${\bf X}^{(2)}$, as follows
\[
\widetilde {\bf X}^{(1)} = \gamma_{1,1} {\bf X}^{(1)} + \gamma_{1,2} {\bf X}^{(2)} =
\]
\[
{\bf A} \left( \gamma_{1,1} {\bf D}_1({\bf C}) + \gamma_{1,2} {\bf D}_2({\bf C}) \right) {\bf B}^T,
\]
\[
\widetilde {\bf X}^{(2)} = \gamma_{2,1} {\bf X}^{(1)} + \gamma_{2,2} {\bf X}^{(2)} =
\]
\[
{\bf A} \left( \gamma_{2,1} {\bf D}_1({\bf C}) + \gamma_{2,2} {\bf D}_2({\bf C}) \right) {\bf B}^T.
\]
The net effect of such slab mixing is that one replaces ${\bf C}$ by $\widetilde {\bf C} := \bGamma {\bf C}$, and we draw
$\bGamma:=\left[
                 \begin{array}{cc}
                   \gamma_{1,1} & \gamma_{1,2} \\
                   \gamma_{2,1} & \gamma_{2,2} \\
                 \end{array}
               \right]$
to have i.i.d. elements uniformly distributed over $[0,1]$. It is now easy to see that all elements of $\widetilde {\bf C}$ will be nonzero almost surely, and $\text{rank} \left( \widetilde {\bf X}^{(1)} \right) = \text{rank} \left( \widetilde {\bf X}^{(2)} \right) = F$ almost surely. Furthermore, since any two columns of $\widetilde {\bf C}$ are equal to $\bGamma$ times the corresponding two columns of ${\bf C}$, it follows that the ratios $\widetilde {\bf C}(1,:) / \widetilde {\bf C}(2,:)$ (where the division is element-wise) will all be distinct almost surely, if and only if {\em any two} columns of ${\bf C}$ are linearly independent -- in which case ${\bf C}$ has Kruskal rank $\geq 2$.
We have therefore proven Theorem \ref{theor:GEVD}; with slightly more work, using a Lemma in \cite{SidKyrSPL2012}, we can prove Theorem \ref{theor:uniq2xfcr}.

\subsection{An intermediate result that conveys the flavor of Kruskal's}
\begin{theorem}
Given $\Xt = \left\llbracket{\bf A},{\bf B},{\bf C}\right\rrbracket$, with
${{\bf A}}:~I \times F$, ${{\bf B}}:~J \times F$, and ${{\bf
C}}:~K \times F$, it is {\it necessary} for uniqueness of ${\bf
A}$, ${\bf B}$, ${\bf C}$ that
\begin{equation}
\min(r_{{\bf A}\odot{\bf B}}, r_{{\bf B}\odot{\bf C}}, r_{{\bf C}\odot{\bf A}}) = F. \label{equ_cond1}
\end{equation}
If $F > 1$, then it is also necessary that
\begin{equation}
\min(k_{{\bf A}},k_{{\bf B}},k_{{\bf C}}) \geq 2.
\label{equ_cond2}
\end{equation}
If, in addition\footnote{Conditions (\ref{equ_cond3}) and (\ref{equ_cond4}) imply the necessary conditions (\ref{equ_cond1}) and (\ref{equ_cond2}), as we will see shortly.},
\begin{equation}
r_{{\bf C}} = F, \label{equ_cond3}
\end{equation}
and
\begin{equation}
k_{{\bf A}} + k_{{\bf B}} \geq F+2, \label{equ_cond4}
\end{equation}
then the decomposition of $\Xt$ in terms of ${\bf A}$, ${\bf B}$, and ${\bf C}$ is essentially unique. Due to role symmetry, conditions (\ref{equ_cond3}) and (\ref{equ_cond4}) can be replaced by either
\[
r_{{\bf B}} = F,~~~{\it and}~k_{{\bf A}} + k_{{\bf C}} \geq F+2,
\]
or
\[
r_{{\bf A}} = F,~~~{\it and}~k_{{\bf B}} + k_{{\bf C}} \geq F+2.
\]
\end{theorem}
\begin{proof}
	First consider the necessary conditions
	(\ref{equ_cond1})-(\ref{equ_cond2}). The necessity of
	(\ref{equ_cond1}) follows from
	\begin{equation}
	{\bf X}^{(JI \times K)} = \left( {\bf A} \odot {\bf B} \right)
	{\bf C}^{T}, \label{chap3eqn_matricizedview1}
	\end{equation}
	\begin{equation}
	{\bf X}^{(IK\times{J})} = ({\bf C}\odot{\bf A}){\bf B}^T, \label{chap3eqn_matricizedview2}
	\end{equation}
	\begin{equation}
	{\bf X}^{(KJ\times{I})} = ({\bf B}\odot{\bf C}){\bf A}^T,
	\label{chap3eqn_matricizedview3}
	\end{equation}
	where the superscripts denote the size of different matrix rearrangements of tensor ${\bf X}$.
	If any of the three Khatri--Rao products fails to be full column rank, then
	one may add a vector in the right null space of that Khatri-Rao product to any of
	the rows of the corresponding third matrix without affecting the
	data. In order to see that (\ref{equ_cond2}) is necessary, it
	suffices to consider the $F=2$ case (if one can mix two rank-1 factors
	[meaning: use two linear combinations of the two factors instead of the pure factors themselves] without affecting their contribution to the data, then the model is not unique, irrespective of the remaining factors). Hence
	consider the $I \times 2$, $J \times 2$, $K \times 2$ case, and
	assume without loss of generality that $k_{{\bf A}} = 1$. This
	means that the two columns of ${\bf A}$ are collinear, i.e.,
	\[
	x_{i,j,k} = a_{i,1}(b_{j,1} c_{k,1} + \lambda b_{j,2} c_{k,2}),
	\]
	which implies that the $i$-th slab of $\Xt$ along
	the first mode is given by
	\[
	{\bf X}_{i} = a_{i,1} {\bar {\bf B}} {\bf C}^{T},~i=1,\cdots,I,
	\]
	where ${\bar {\bf B}} = \left[ {\bf b}_{1} ~ \lambda {\bf b}_{2}
	\right]$, and ${\bf C} = \left[ {\bf c}_{1} ~ {\bf c}_{2} \right]$.
	Therefore, slabs along the first mode are multiples of each other;
	${\bf a}_{1} = \left[a_{1,1},\cdots,a_{I,1}\right]^{T}$ can be
	uniquely determined up to global scaling (hence ${\bf A}$ can be determined
	up to scaling of its columns), but there exists linear transformation freedom in
	choosing ${\bf B}$ and ${\bf C}$. In this case, the tensor $\Xt$ is just a
	matrix, say the first slab ${\bf X}_{1}$, and then the other slabs are scaled copies of that matrix -- so subject to the same ambiguities as matrix decomposition.
	
	The sufficiency of (\ref{equ_cond1})-(\ref{equ_cond2}) and
	(\ref{equ_cond3})-(\ref{equ_cond4}) will be shown by contradiction. Without loss of
	generality, assume $r_{{\bf C}} = F$ (this implies that ${\bf C}$
	is tall or square), and $r_{{\bf A}\odot{\bf B}} = F$. Note that
	it suffices to consider the case of square ${\bf C}$, for
	otherwise ${\bf C}$ contains a square submatrix consisting of
	linearly independent rows. Considering this square nonsingular
	submatrix amounts to discarding the remaining rows of ${\bf C}$,
	or, equivalently, dispensing with certain data slabs taken along
	the third mode. It is sufficient to prove that the
	parameterization of $\Xt$ in terms of ${\bf A}$,
	${\bf B}$, and the row-truncated ${\bf C}$ is unique based on part
	of the data. The uniqueness of the full ${\bf C}$ will then follow
	trivially.
	
	We will need the following elementary fact, which is a very
	special case of the {\it Permutation Lemma} in [Kruskal, '77].
	Let ${\it w}({\bf v})$ denote the number of nonzero elements (the
	{\it weight}) of
	${\bf v} \in \Complex^{K}$. Consider two $F \times F$ nonsingular
	matrices ${\bf C}$ and $\bar {\bf C}$. Suppose that
	\begin{equation}
	{\it w}({\bf v}^{T} {\bf C}) = 1,~\forall {\bf v} ~~|~~ {\it
		w}({\bf v}^{T} \bar {\bf C}) = 1. \label{equ_plc}
	\end{equation}
	(Meaning: for all ${\bf v}$ such that ${\it
		w}({\bf v}^{T} \bar {\bf C}) = 1$, it holds that ${\it w}({\bf v}^{T} {\bf C}) = 1$ as well.)
	It then follows that $\bar {\bf C} = {\bf C} {\bPi} {\bLambda}$,
	where ${\bPi}$ is a permutation matrix, and ${\bLambda}$ is a
	nonsingular diagonal scaling matrix. For a proof, note that if
	condition (\ref{equ_plc}) holds, then
	\[
	{\bar {\bf C}}^{-1} {\bar {\bf C}} = {\bf I} \Longrightarrow {\bar
		{\bf C}}^{-1} {\bf C} = {\bPi}^{T} {\bf D},
	\]
	where ${\bf D}$ is a nonsingular diagonal matrix, and we have used
	that the product ${\bar {\bf C}}^{-1} {\bf C}$ is full rank, and
	its rows have weight one. It then follows that
	\[
	{\bf C} = {\bar {\bf C}} {\bPi}^{T} {\bf D} \Longleftrightarrow
	{\bar {\bf C}} = {\bf C} {\bf D}^{-1} {\bPi} = {\bf C} {\bPi}
	{\bLambda}.
	\]
	With these preliminaries at hand, we are ready to give the main (previously unpublished) proof.
	
	Suppose $\Xt = \left\llbracket{\bf A},{\bf B},{\bf C}\right\rrbracket = \left\llbracket{\bar
		{\bf A}},{\bar {\bf B}},{\bar {\bf C}}\right\rrbracket$. From
	(\ref{chap3eqn_matricizedview1}), it follows that
	\begin{equation}
	\left( {\bf A} \odot {\bf B} \right) {\bf C}^{T} = {\bf X}^{(JI
		\times K)} = \left( {\bar {\bf A}} \odot {\bar {\bf B}} \right)
	{\bar {\bf C}}^{T}. \label{chap3eqn_willneedthisagain}
	\end{equation}
	Since $r_{{\bf A}\odot{\bf B}} = r_{{\bf C}} = F$, it follows that
	$r_{{\bar {\bf A}}\odot{\bar {\bf B}}} = r_{{\bar {\bf C}}} = F$.
	
	Taking linear combinations of the slabs along the third mode in
	(\ref{chap3eqn_matricizedview2}), we obtain
	\begin{equation}
	\sum_{k=1}^{K} v_{k} {\bf X}(:,:,k) = {\bf A} diag({\bf v}^{T} {\bf
		C}) {\bf B}^{T} = \bar {\bf A} diag({\bf v}^{T} \bar {\bf C}) \bar
	{\bf B}^{T}, \label{chap3eqn_keyform}
	\end{equation}
	for all ${\bf v} := \left[v_{1},\cdots,v_{F} \right]^{T} \in
	\Complex^{F}$.\\ The rank of a matrix product is always less than or
	equal to the rank of any factor, and thus
	\begin{equation}
	{\it w}({\bf v}^{T} \bar {\bf C}) = r_{diag({\bf v}^{T} \bar {\bf
			C})} \geq r_{\bar {\bf A} diag({\bf v}^{T} \bar {\bf C}) \bar {\bf
			B}^{T}} = r_{{\bf A} diag({\bf v}^{T} {\bf C}) {\bf B}^{T}}.
	\label{chap3eqn_6a}
	\end{equation}
	
	Assume ${\it w}({\bf v}^{T} \bar {\bf C}) = 1$; then (\ref{chap3eqn_6a})
	implies $r_{{\bf A} diag({\bf v}^{T} {\bf C}) {\bf B}^{T}} \leq 1$,
	and we wish to show that ${\it w}({\bf v}^{T} {\bf C}) = 1$.
	Let us use the shorthand ${\it w} := {\it w}({\bf v}^{T} {\bf C})$.
	Using Sylvester's inequality and the definition of k-rank:
	\[
	r_{{\bf A}diag({\bf v}^{T} {\bf C}) {\bf B}^{T}} \geq \min(k_{{\bf A}},{\it w}) +
	\min(k_{{\bf B}},{\it w}) - {\it w}.
	\]
	Hence
	\begin{equation}
	\min(k_{{\bf A}},{\it w}) + \min(k_{{\bf B}},{\it w}) - {\it w} \leq 1.
	\label{chap3eqn_keyineq}
	\end{equation}
	Let us consider cases:
	\begin{enumerate}
		\item Case of ${\it w} \leq \min(k_{{\bf A}},k_{{\bf B}})$: then
		(\ref{chap3eqn_keyineq}) implies ${\it w} \leq 1$, hence
		${\it w} = 1$, because ${\bf C}$ is nonsingular and ${\bf v} \neq {\bf 0}$;
		\item Case of $\min(k_{{\bf A}},k_{{\bf B}}) \leq {\it w} \leq \max(k_{{\bf A}},k_{{\bf B}})$: then (\ref{chap3eqn_keyineq}) implies
		$\min(k_{{\bf A}},k_{{\bf B}}) \leq 1$, which contradicts (\ref{equ_cond2}),
		thereby excluding this range of ${\it w}$ from consideration;
		\item Case of ${\it w} \geq \max(k_{{\bf A}},k_{{\bf B}})$:
		then (\ref{chap3eqn_keyineq}) implies that
		${\it w} \geq k_{{\bf A}}+k_{{\bf B}}-1$. Under (\ref{equ_cond4}),
		however, this yields another
		contradiction, as it requires that ${\it w} \geq F+1$, which is
		impossible since the maximum possible ${\it w} = {\it w}({\bf v}^{T} {\bf C})$ is $F$.
	\end{enumerate}
	
	We conclude that, under (\ref{equ_cond1})-(\ref{equ_cond2}) and
	(\ref{equ_cond3})-(\ref{equ_cond4}),
	${\it w}({\bf v}^{T} \bar {\bf C}) = 1$ implies
	${\it w}({\bf v}^{T} {\bf C}) = 1$.
	From the elementary version of the Permutation Lemma, it follows
	that ${\bar {\bf C}} = {\bf C} {\bPi}{\bLambda}$.
	
	From (\ref{chap3eqn_willneedthisagain}) we now obtain
	\[
	\left[ \left( {\bf A} \odot {\bf B} \right) - \left( {\bar {\bf
			A}} \odot {\bar {\bf B}} \right) {\bLambda} {\bPi}^{T} \right]
	{\bf C}^{T} = {\bf 0},
	\]
	and since ${\bf C}$ is nonsingular,
	\begin{equation}
	\left( {\bar {\bf A}} \odot {\bar {\bf B}} \right) = \left( {\bf
		A} \odot {\bf B} \right) {\bPi} {\bLambda}^{-1}.
	\label{chap3eqn_onemorestep}
	\end{equation}
	It follows that, for every column ${{\bf a}}_{f} \otimes {{\bf
			b}}_{f}$ of ${\bf A} \odot {\bf B}$ there exists a unique column
	${\bar {\bf a}}_{f^{'}} \otimes {\bar {\bf b}}_{f^{'}}$ of ${\bar
		{\bf A}} \odot {\bar {\bf B}}$ such that
	\[
	{{\bf a}}_{f} \otimes {{\bf b}}_{f} = {\bar {\bf a}}_{f^{'}}
	\otimes {\bar {\bf b}}_{f^{'}} \lambda_{f^{'}}.
	\]
	It only remains to account for uniqueness of the truncated rows of
	a possibly tall ${\bf C}$, but this is now obvious from
	(\ref{chap3eqn_willneedthisagain}), (\ref{chap3eqn_onemorestep}), and
	(\ref{equ_cond1}). This completes the proof.
\end{proof}

\subsection{Rank and k-rank of the Khatri--Rao product}

\begin{property} [Sidiropoulos \& Liu, '99] \label{frpkrp}
If $k_{\bf A} \geq 1$ and $k_{\bf B} \geq 1$, then it holds that
\[
k_{{\bf B} \odot {\bf A}} \geq \min(k_{\bf A} + k_{\bf B} -1, F),
\]
whereas if $k_{\bf A} = 0$ or $k_{\bf B} = 0$
\[
k_{{\bf B} \odot {\bf A}} = 0.
\]
\end{property}

\begin{proof}
If $k_{{\bf B} \odot {\bf A}} = F$, then the result holds trivially;
hence consider the case $k_{{\bf B} \odot {\bf A}} < F$.
The proof is by contradiction.
Let $S$ be the {\it smallest} number of linearly dependent columns that
can be drawn from ${\bf B} \odot {\bf A}$, denoted by
${\bf b}_{f_{1}} \otimes {\bf a}_{f_{1}}$, $\cdots$,
${\bf b}_{f_{S}} \otimes {\bf a}_{f_{S}}$.
Since $k_{{\bf B} \odot {\bf A}} < F$, $S \leq F$ and $k_{{\bf B} \odot {\bf A}} = S-1$.
Then it holds that
there exist $\mu_{1} \in \Complex, \cdots, \mu_{S} \in \Complex$, with
$\mu_{1} \neq 0, \cdots, \mu_{S} \neq 0$ (since $S$ is {\it smallest}) such that
\[
\mu_{1} {\bf b}_{f_{1}} \otimes {\bf a}_{f_{1}} + \cdots +
\mu_{S} {\bf b}_{f_{S}} \otimes {\bf a}_{f_{S}} = {\bf 0}_{IJ \times 1},
\]
or, equivalently,
\begin{equation}
\tilde {\bf A} diag \left( \left[ \mu_{1}, \cdots, \mu_{S} \right] \right)
\tilde {\bf B}^{T} = {\bf 0}_{I \times J},
\label{chap2eqn_loveboat}
\end{equation}
with
\[
\tilde {\bf A} := \left[ {\bf a}_{f_{1}}, \cdots, {\bf a}_{f_{S}} \right],
\]
and
\[
\tilde {\bf B} := \left[ {\bf b}_{f_{1}}, \cdots, {\bf b}_{f_{S}} \right].
\]
Invoking Sylvester's inequality
\[
0 = rank\left( {\bf 0} \right) = rank \left( \tilde {\bf A} diag \left( \left[ \mu_{1}, \cdots, \mu_{S} \right] \right)
\tilde {\bf B}^{T} \right) \geq
\]
\[
rank(\tilde {\bf A}) +
rank(\tilde {\bf B}) - S.
\]
However, by definition of k-rank
\[
rank(\tilde {\bf A}) \geq \min(k_{\bf A},S); ~~~
rank(\tilde {\bf B}) \geq \min(k_{\bf B},S),
\]
and thus
\begin{equation}
0 \geq \min(k_{\bf A},S) + \min(k_{\bf B},S) - S.
\label{chap2eqn_2considercases}
\end{equation}
Now consider the following cases for (\ref{chap2eqn_2considercases}):
\begin{itemize}
\item if $1 \leq S \leq \min(k_{\bf A},k_{\bf B})$,
then (\ref{chap2eqn_2considercases}) gives $0 \geq S \geq 1$, which is a contradiction;
\item else if $\min(k_{\bf A},k_{\bf B}) < S < \max(k_{\bf A},k_{\bf B})$,
then inequality (\ref{chap2eqn_2considercases}) gives $0 \geq \min(k_{\bf A},k_{\bf B}) + S - S = \min(k_{\bf A},k_{\bf B})$ $\geq 1$, another contradiction;
\item else if $\max(k_{\bf A},k_{\bf B}) \leq S$, then (\ref{chap2eqn_2considercases}) gives $0 \geq k_{\bf A}+k_{\bf B} - S$, so $S \geq k_{\bf A}+k_{\bf B}$ is the only option.
\end{itemize}
The conclusion is that if $k_{{\bf B} \odot {\bf A}} < F$, then $k_{{\bf B} \odot {\bf A}} = S-1 \geq k_{\bf A}+k_{\bf B} - 1$.
Note that $k_{\bf A} = 0$ if and only if ${\bf A}$ contains at least
one identically zero column, in which case the column-wise
Kronecker product will have at least one identically zero column,
hence its k-rank will be zero.
This completes the proof.
\end{proof}

One can show that full rank (even full $k$-rank) of both ${\bf A}$
and ${\bf B}$ does not necessarily guarantee that the Khatri--Rao
product ${\bf A} \odot {\bf B}$ is full rank (let alone full
$k$-rank). For example, let $F = 6$ and
${\bf A}$, ${\bf B}$ Vandermonde with the following generators:
\begin{equation*}
\label{JSJ-52} \alpha_{1} = 1, \alpha_{2} = 2, \alpha_{3} = 3,
\alpha_{4} = 4, \alpha_{5} = 5, \alpha_{6} = 6.
\end{equation*}
\begin{equation*}
\label{JSJ-53} \beta_{1} = 1, \beta_{2} = \sqrt{2}, \beta_{3}=
\sqrt{3}, \beta_{4} = \sqrt{4}, \beta_{5} = \sqrt{5}, \beta_{6} =
\sqrt{6}.
\end{equation*}
With this choice of generators, ${\bf A}$ and ${\bf B}$ are
full $k$-rank. When $I=3$ and $J=2$, the $6 \times 6$ Khatri--Rao product
${\bf A} \odot {\bf B}$ is
full rank, hence also full $k$-rank:
$k_{{\bf A} \odot {\bf B}} = r_{{\bf A} \odot {\bf B}} = 6$.
Now set $I=2$ and $J=3$; the Khatri--Rao product is still $6 \times 6$,
but its rank is $5$. As it turns out, this phenomenon is uncommon:
\begin{property} [Jiang, Sidiropoulos, ten Berge, '01] \label{asfrpkrp}
For a pair of matrices ${\bf A} \in \Complex^{I \times F}$ and ${\bf B}
\in \Complex^{J \times F}$,
\begin{equation}
\label{myJSJ-13} r_{ {\bf A} \odot {\bf B} } = k_{ {\bf A} \odot
{\bf B} } = \min(I J,F), ~~~P_{{\cal L}}(\Complex^{(I+J)F})-a.s.,
\end{equation}
where $P_{{\cal L}}(\Complex^{(I+J)F})$ is the distribution used to draw the
$(I+J)F$ complex elements of $A$ and $B$, assumed continuous with respect to
the Lebesgue measure in $\Complex^{(I+J)F}$. In plain words,
$r_{ {\bf A} \odot {\bf B} } = k_{ {\bf A} \odot
{\bf B} } = \min(I J,F)$ for almost every ${\bf A} \in \Complex^{I \times F}$,
${\bf B} \in \Complex^{J \times F}$.
\end{property}

\begin{proof}
If $I J \leq F$, it suffices to prove that
an {\it arbitrary selection} of $I J$ columns yields an almost surely
nonsingular matrix. Any such matrix is a square Khatri--Rao product,
and its determinant is an analytic function of
(a subset of) the elements of ${\bf A}$, and
${\bf B}$.
If, on the other hand, $I J \geq F$,
then it suffices to show that the upper $F \times F$ part
of the Khatri--Rao product is almost surely nonsingular. The
determinant of this upper square part is likewise analytic in
(a subset of) the elements of ${\bf A}$,
${\bf B}$. What remains is to show that these functions are non-trivial,
or, equivalently,

\begin{quote}
Show that, with $IJ \geq F$, it is possible to construct
a Khatri--Rao product whose upper square part is nonsingular,
for otherwise arbitrary $I, J, F$.
\end{quote}

The key idea is to pick ${\bf A}$,
${\bf B}$ in such a way that the resulting Khatri--Rao product is
a Vandermonde matrix with FFT-grid generators. This is done as
follows. Let ${\bf A}$ be Vandermonde with generators
$\alpha_f= e^{\sqrt{-1} \frac{2 \pi}{F} J
(f-1)}$, and likewise
${\bf B}$ Vandermonde with generators
$\beta_f = e^{\sqrt{-1} \frac{2 \pi}{F} (f-1)}$ for
$f=1, \cdots, F$. Then ${\bf A} \odot {\bf B}$ is
itself a Vandermonde matrix with generators $ (1, e^{\sqrt{-1}
\frac{2 \pi}{F}}, \cdots, e^{\sqrt{-1} \frac{2 \pi}{F}(F-1)} )$,
and its upper square part is therefore nonsingular. This completes
the proof.
\end{proof}

\begin{remark}
Note that the order of Khatri--Rao
multiplications only affects the order of rows in the final
result; in particular, the rank/k-rank of the final result
is not affected by the order in which
the multiplications are carried out.
\end{remark}

\subsection{Optimal scaling lemma}
\begin{proof} (Lemma \ref{oslem})
Let ${\bf U}$ be a basis of the orthogonal complement of ${\bf v}:= {\bf m} / ||{\bf m}||$. Then
\[
\left| \left| \tilde {\bf X}_3 - {\bf m} {\bf c}^T \right| \right|_F^2 = \left| \left| [{\bf v},~{\bf U}]^T \left( \tilde {\bf X}_3 - {\bf m} {\bf c}^T \right) \right| \right|_F^2
\]
\[
= \left| \left| \left[
                \begin{array}{c}
                  \frac{{\bf m}^T}{||{\bf m}||} \tilde {\bf X}_3 - ||{\bf m}|| {\bf c}^T\\
                  {\bf U}^T \tilde {\bf X}_3\\
                \end{array}
              \right]
 \right| \right|_F^2.
\]
It follows that
\[
\min_{{\bf c} \in {\cal C}} \left| \left| \tilde {\bf X}_3 - {\bf m} {\bf c}^T \right| \right|_F^2  \Longleftrightarrow
\min_{{\bf c} \in {\cal C}} \left| \left| \frac{{\bf m}^T}{||{\bf m}||} \tilde {\bf X}_3 - ||{\bf m}|| {\bf c}^T \right| \right|_F^2
\]
\[
\Longleftrightarrow
\min_{{\bf c} \in {\cal C}} \left| \left| \frac{{\bf m}^T}{||{\bf m}||^2} \tilde {\bf X}_3 - {\bf c}^T \right| \right|_F^2.
\]
\end{proof}

\subsection{Enforcing partial symmetry (${\bf B} = {\bf A}$) in CPD ALS}

Adopting a rank-1 update strategy leads to problems of type
\[
\min_{{\bf a}_{f}} \sum_{k=1}^K \left| \left| \tilde {\bf X}(:,:,k) - {\bf a}_{f} {\bf C}(k,f) {\bf a}_{f}^T  \right| \right|_F^2.
\]
Without loss of generality, we may assume $||{\bf a}_{f}||_2=1$ and absorb any scaling in the ${\bf C}$ matrix. Let ${\bf U}$ be a basis for the orthogonal complement of ${\bf a}_{f}$. Using $||{\bf M}||_F^2=\text{Tr}({\bf M}^T {\bf M})$, we have
\[
\left| \left| \tilde {\bf X}(:,:,k) - {\bf a}_{f} {\bf C}(k,f) {\bf a}_{f}^T  \right| \right|_F^2 = \left| \left| \tilde {\bf X}(:,:,k)\right| \right|_F^2 -
\]
\[
2 \text{Tr}\left((\tilde {\bf X}(:,:,k))^T {\bf a}_{f} {\bf C}(k,f) {\bf a}_{f}^T\right) + ({\bf C}(k,f))^2 ||{\bf a}_{f}||_2^4 =
\]
\[
\left| \left| \tilde {\bf X}(:,:,k)\right| \right|_F^2 - 2 \text{Tr}\left((\tilde {\bf X}(:,:,k))^T {\bf a}_{f} {\bf C}(k,f) {\bf a}_{f}^T\right) + ({\bf C}(k,f))^2.
\]
So we may equivalently maximize
\[
2 \text{Tr}\left((\tilde {\bf X}(:,:,k))^T {\bf a}_{f} {\bf C}(k,f) {\bf a}_{f}^T\right) = 2 {\bf a}_{f}^T (\tilde {\bf X}(:,:,k))^T {\bf a}_{f} {\bf C}(k,f)
\]
\[
= {\bf a}_{f}^T {\bf C}(k,f) \left((\tilde {\bf X}(:,:,k))^T + \tilde {\bf X}(:,:,k) \right) {\bf a}_{f},
\]
and collecting terms $\forall k$,
\[
\max_{||{\bf a}_{f}||_2=1} {\bf a}_{f}^T {\bf Q} {\bf a}_{f},
\]
\[
{\bf Q} := \sum_{k=1}^K  {\bf C}(k,f) \left( (\tilde {\bf X}(:,:,k))^T + \tilde {\bf X}(:,:,k) \right).
\]
So the optimal update for ${\bf a}_{f}$ amounts to solving for the eigenvector corresponding to the maximum eigenvalue of matrix ${\bf Q}$; see \cite{1011195}.

\subsection{CRB for matrix and CP tensor factorization}

The \crb~bound (CRB) \cite[Ch.~3]{kay1993fundamentals} is the most widely used estimation benchmark in signal processing. In many cases it is relatively easy to compute, and it is asymptotically achievable by maximum likelihood (ML) estimators in high signal to noise ratio (SNR) scenarios \cite[pp. 164]{kay1993fundamentals}. In other cases, there may be technical difficulties in deriving (or complexity issues in computing) the pertinent CRB; but due to the central role of this bound in signal processing research, work on developing CRB tools continues \cite{gorman1990lower,stoica1998cramer,stoica2001parameter,ben2009constrained}, thereby enlarging the set of problems for which the CRB can be used in practice.

After a brief review of the \crb~bound and some of its modern developments, we will derive the \crb~for both matrix and CP factorizations. The Fisher information matrices (FIM) for the matrix and CP factorization models are not affected by constraints imposed onto the latent factors, even though some constraints (like non-negativity) are crucial in terms of identifiability. We also discuss efficient ways to (peudo-)invert the FIM to avoid the massive computation and memory requirements of direct inversion when the problem size is moderately large, since the FIM can easily become huge as it is a symmetric matrix with number of rows equal to the number of parameters we want to estimate.

\subsubsection{The \crb~Bound}
Suppose a set of measurements $\y$ is drawn from a probability density function $p(\y;\th)$ parameterized by $\th$, and our goal is to estimate $\th$ given the realizations of $\y$. If the regularity condition $$\E{\score}=0$$ is satisfied, we can define the Fisher information matrix (FIM) as
\[
\fim = -\E{\nabla_{\th}^2 \log p(\y;\th)},
\]
which can be shown to be equal to~\cite{kay1993fundamentals}
\[
\fim = \E{\score\score^T};
\]
then for any unbiased estimator $\hat{\th}$, i.e., $\E{\hat{\th}}=\th$, we have
\[
\cov{\hat{\th}} = \E{(\hat{\th}-\th)(\hat{\th}-\th)^T} \succeq \fim^{-1},
\]
or we can simply take
\[
\E{\|\hat{\th}-\th\|^2}\geq \trace{\fim^{-1}}.
\]

A simple way to prove the CRB is as follows. Let us look at the following covariance
\begin{align}
\E{\begin{bmatrix}
\hat{\th}-\th \\
\score
\end{bmatrix}
\begin{bmatrix}
\hat{\th}-\th \\
\score
\end{bmatrix}^T} \nonumber\\
= \begin{bmatrix}
\cov{\hat{\th}} & \G \\
\G^T & \fim \\
\end{bmatrix} \succeq 0 \label{eq:proof}
\end{align}
where $\G = \E{(\hat{\th}-\th)\score^T}$. According to Schur complement~\cite[Appendix~A.5.5]{boyd2004convex}, if $\fim\succ 0$, then (\ref{eq:proof}) holds iff
\begin{equation}\label{eq:crbproof}
\cov{\hat{\th}}-\G\fim^{-1}\G^T \succeq 0.
\end{equation}
Looking at
\begin{align*}
\E{\score} & = \int_{\cal Y} \left(\score\right) p(\y;\th) d\y \\
			&= \int_{\cal Y} \nabla_{\th} p(\y;\th) d\y ,			
\end{align*}
suppose the support of the random variable $\y$, denoted as $\mathcal{Y}$, is independent of $\th$, then we can reverse the order of derivative and integration, leading to
\[
\E{\score} = \nabla_{\th} \int_{\cal Y} p(\y;\th) d\y = 0,
\]
which gives us the regularity condition. Then for the matrix $\G$ we have that
\begin{align*}
\G &= \E{(\hat{\th}-\th)\score^T} \\
&= \E{\hat{\th}\score^T} \\
&= \int_{\cal Y} \hat{\th} \left(\score\right)^T p(\y;\th) d\y \\
&= \int_{\cal Y} \hat{\th} \nabla_{\th}p(\y;\th)^T d\y \\
&= \jacob_{\th}\E{\hat{\th}}^T = \I,
\end{align*}
where we again used the fact that the order of integral and derivative can be reversed, and that $\hat{\th}$ is unbiased $\E{\hat{\th}}=\th$. Thus, we plug it back into (\ref{eq:crbproof}) and obtain the \crb~bound
\[
\cov{\hat{\th}} \succeq \fim^{-1}.
\]

If the FIM $\fim$ is singular, we can use the generalized Schur complement result~\cite[\S A.5.5]{boyd2004convex} to conclude that
\[
\begin{bmatrix}
\cov{\hat{\th}} & \G \\
\G^T & \fim \\
\end{bmatrix} \succeq 0
\]
if and only if
\[
\fim\succeq 0,~~(\I-\fim\fim^\dagger)\G=0,~~
\cov{\hat{\th}}-\G\fim^\dagger\G^T \succeq 0.
\]
This means that:
\begin{enumerate}
\item $\E{\|\hat{\th}-\th\|^2} \geq \trace{\fim^\dagger}$ is still a valid bound;
\item this looser bound is in theory not attainable, because $\I-\fim\fim^\dagger\neq 0$.
\end{enumerate}

\noindent {\bf CRB and identifiability:} It is natural to suspect that the singularity of FIM is caused by the fact that the model is not identifiable (meaning the solution is not unique in the noiseless case). However, identifiability in general neither implies nor is implied by a non-singular FIM. A famous example is given in~\cite{basu2000stability}: consider the scalar signal model
\[
y = \theta^2 + \nu, \text{where } \nu\sim\mathcal{N}(0,\sigma^2),
\]
the FIM with respect to $\theta$ is
\[
\Phi = \frac{4}{\sigma^2}\theta^2;
\]
interestingly, $\Phi=0$ if and only if $\theta=0$, the only identifiable point. Experience shows that the rank deficiency of FIM is usually related to trivial ambiguities of the problem, but not the critical ones. Take phase retrieval as an example, it has been shown that the FIM corresponding to this problem is always rank one deficient~\cite{qian2016icassp,qian2016phase}, which, by identifying its null space, seems to be highly related to the global phase ambiguity inherent to this problem. For certain measurement systems, e.g., 1D Fourier measurement, the problem is not identifiable besides the trivial phase ambiguity, but the rank deficiency is still one~\cite{huang2016fourier,huang2016phase}, which means the critical non-uniqueness issue is not revealed by the singularity of FIM. The practical implication for us is that, for trivial ambiguities, for example the permutation and scaling ambiguity in the matrix and CP factorization considered in this paper, when we simulate we should fix these trivial ambiguities in a consistent way, and then compare the MSE with the generalized CRB obtained from the pseudo-inverse of the singular FIM.

\noindent {\bf CRB under constraints:} Suppose for estimating $\th$ we have the prior information that
\[
\bm{f}(\th)\leq 0,~~ \bm{g}(\th)=0.
\]
Roughly speaking, the inequality constraints do not affect the CRB. For equality constraints, denote the Jacobian matrix of the vector function $\bm{g}(\th)$ at point $\th$ as $\jacob_{\th}\bm{g}(\th)$, i.e.,
\[
[~\jacob_{\th}\bm{g}(\th)~]_{i,j} = \frac{\partial g_i(\th)}{\theta_j};
\]
we can find a matrix $\Q$ with ortho-normal columns that spans the null space of $\jacob_{\th}\bm{g}(\th)$, i.e.,
\[
\jacob_{\th}\bm{g}(\th)^T\Q = 0,~~ \Q^T\Q = \I.
\]
Then the constrained CRB is modified as follows
\[
\E{\|\hat{\th}-\th\|^2} \geq \trace{\Q\left(\Q^T\fim\Q\right)^\dagger\Q^T}.
\]

\noindent {\bf CRB under Gaussian noise:} Suppose the data model admits the form
\begin{equation}\label{eq:nls}
\bm{y}=\ph(\th)+\bm{\nu},
\end{equation}
where $\bm{\nu}$ are i.i.d. Gaussian noise with variance $\sigma^2$, the most commonly used noise model in practice. In this case, it can be shown that the Fisher information matrix admits a very simple form~\cite{basu2000stability}, as presented in the following.
\begin{proposition}\label{ppst:crb4gauss}
The Fisher information matrix for the data model (\ref{eq:nls}) is given by
\[
\fim = \frac{1}{\sigma^2}\jacob_{\th}\ph(\th)^T\jacob_{\th}\ph(\th).
\]
\end{proposition}
\begin{proof}
For i.i.d. Gaussian noise $\bm{\nu}$, the log-likelihood is simply given by
\[
\log p(\bm{y};\th)=-\frac{1}{2\sigma^2}\|\bm{y}-\ph(\th)\|^2,
\]
a non-linear least squares function. The Hessian matrix of it at point $\th$ has the form~\cite[\S1.5]{bertsekas1999nonlinear}
\begin{align*}
 &-\nabla_{\th}^2\log p(\bm{y};\th) \\
=& \frac{1}{\sigma^2}
\left( \jacob_{\th}\ph(\th)^T\jacob_{\th}\ph(\th) +
\sum_i (y_i-\varphi_i(\th))\nabla_{\th}^2\varphi_i(\th) \right).
\end{align*}
The FIM is taken as the expected value of $\nabla_{\th}^2\log p(\bm{y};\th)$ over $\bm{y}$. Notice that in the above equation, $\bm{y}$ only appears in the second term; furthermore, according to our data model, we have that $\E{y_i}=\varphi_i(\th)$, which means the second term becomes zero after we take expectation. Hence, we have that
\[
\fim = -\E{\nabla_{\th}^2\log p(\bm{y};\th)}
= \frac{1}{\sigma^2}\jacob_{\th}\ph(\th)^T\jacob_{\th}\ph(\th).
\]
\end{proof}

\noindent {\bf CRB under non-Gaussian noise:} In some cases we may observe that the noise is more heavy-tailed, thus we may wish to model noise different from the Gaussian distribution, for example the Laplacian or Cauchy. Luckily, it has been shown that the CRB for a family of non-Gaussian noise models are simply scaled versions of their Gaussian counter-parts~\cite{swami1996cramer,swami2002some}. Specifically, we have shown that the FIM under i.i.d. Gaussian noise with zero mean and variance $\sigma^2$ takes the form
\[
\fim = \frac{1}{\sigma^2}\ufim,
\]
which could be derived via an easier way as we discussed before; then for the same model, if the additive noise $\nu$ is changed to Laplacian noise
\[
p(\nu)=\frac{1}{2b}\exp\left(-\frac{|\nu|}{b}\right),
\]
the modified CRB is
\[
\fim = \frac{2}{b^2}\bm{\ufim};
\]
for Cauchy noise with distribution
\[
p(\nu) = \frac{1}{\pi\gamma}\left(\frac{\gamma^2}{\nu^2+\gamma^2}\right),
\]
the CRB becomes
\[
\fim = \frac{1}{2\gamma^2}\bm{\ufim}.
\]

\subsubsection{\crb~Bound for Matrix Factorization Models}
Consider the $m \times n$ matrix generated as
\begin{equation*}
    \Y = \W\H^T + \N,
\end{equation*}
where $\W$ is $m \times k$, $\H$ is $n \times k$, and the elements of $\N$ are drawn from an i.i.d. Gaussian distribution with zero-mean and variance $\sigma^2$. Then the log-likelihood of $\Y$ parameterized by $\W$ and $\H$ is
\begin{align*}
 & \log p(\Y;\W,\H) \\
=& -\frac{1}{2\sigma^2}\left\|\Y-\W\H^T\right\|_F^2 \\
=& -\frac{1}{2\sigma^2}\left\|\vec{\Y}-\ph(\th)\right\|^2,
\end{align*}
where the unknown parameters we want to estimate, $\W$ and $\H$, are stacked into a single long vector of size $(m+n)k$ as follows
\[
\th = \left[~\vec{\W}^T~\vec{\H}^T~\right]^T,
\]
and the non-linear function
\begin{align}
\ph(\th) & = \vec{\W\H^T} \nonumber\\
 & = (\H\otimes\I_m)\vec{\W} \label{eq:phW}\\
 & = \C_{n,m}\vec{\H\W^T} \nonumber\\
 & = \C_{n,m}(\W\otimes\I_n)\vec{\H}. \label{eq:phH}
\end{align}
Here we use $\C_{m,n}$ to represent the commutation matrix~\cite{magnus1979commutation} of size $mn \times mn$, which is a permutation matrix that has the following properties:
\begin{enumerate}
\item $\C_{m,n}\vec{\Sb}=\vec{\Sb^T}$, where $\Sb$ is $m \times n$;
\item $\C_{p,m}(\Sb\otimes\T)=(\T\otimes\Sb)\C_{q,n}$, where $\T$ is $p \times q$;
\item $\C_{p,m}(\Sb\odot\T)=\T\odot\Sb$;
\item $\C_{n,m}=\C_{m,n}^T=\C_{m,n}^{-1}$;
\item $\C_{mp,n}\C_{mn,p}=\C_{m,np}$.
\end{enumerate}

\noindent {\bf The Fisher Information Matrix:} Invoking Proposition~\ref{ppst:crb4gauss}, the FIM for matrix factorization model under Gaussian noise has the form
\[
\fim = \frac{1}{\sigma^2}\ufim = \frac{1}{\sigma^2}\jacob_{\th}\ph(\th)^T\jacob_{\th}\ph(\th).
\]
The Jacobian matrix of $\ph(\th)$ can be partitioned into two blocks
\[
\jacob_{\th}\ph(\th) = \begin{bmatrix}
\jacob_{\W}\ph(\th) & \jacob_{\H}\ph(\th)
\end{bmatrix},
\]
and according to \eqref{eq:phW} and \eqref{eq:phH}, we have that
\begin{align*}
\jacob_{\W}\ph(\th) &= \H\otimes\I_m, \\
\jacob_{\H}\ph(\th) &= \C_{n,m}(\W\otimes\I_n).
\end{align*}
Using properties of the commutation matrices, we have that
\begin{align*}
\jacob_{\W}\ph(\th)^T\jacob_{\W}\ph(\th) &= \H^T\H\otimes\I_m, \\
\jacob_{\H}\ph(\th)^T\jacob_{\H}\ph(\th) &= \W^T\W\otimes\I_n,
\end{align*}
and
\begin{align*}
 &\jacob_{\W}\ph(\th)^T\jacob_{\H}\ph(\th) \\
=& (\H^T\otimes\I_m)\C_{n,m}(\W\otimes\I_n) \\
=& (\H^T\otimes\I_m)(\I_n\otimes\W)\C_{n,k} \\
=& (\H^T\otimes\W)\C_{n,k} \\
=& (\I_k\otimes\W)(\H^T\otimes\I_k)\C_{n,k} \\
=& (\I_k\otimes\W)\C_k(\I_k\otimes\H^T).
\end{align*}

Hence, we can then express the $(m+n)k \times (m+n)k$ Fisher information matrix $\fim=\sigma^{-2}\ufim$ compactly as follows
\begin{align}\label{eq:fim4mf}
 & \ufim \nonumber\\
=&\begin{bmatrix}
\H^T\H \otimes \eye{m} & \!\!\! (\eye{k}\otimes\W)\C_k(\eye{k}\otimes\H)^T \\
(\eye{k}\otimes\H)\C_k(\eye{k}\otimes\W)^T \!\!\! & \W^T\W \otimes \eye{n}
\end{bmatrix} \\
= &\begin{bmatrix}
\H^T\H \otimes \eye{m} & 0 \\
0 & \W^T\W \otimes \eye{n}
\end{bmatrix} + \nonumber\\
&\begin{bmatrix}
\eye{k}\otimes\W & 0 \\
0 & \eye{k}\otimes\H
\end{bmatrix}\begin{bmatrix}
0 & \C_k \\
\C_k & 0
\end{bmatrix}\begin{bmatrix}
\eye{k}\otimes\W & 0 \\
0 & \eye{k}\otimes\H
\end{bmatrix}^T. \nonumber
\end{align}

The FIM for the matrix factorization model (\ref{eq:fim4mf}) is rank deficient, as shown in the following proposition.
\begin{proposition}\label{ppst:fimrank-mf}
If $\W$ and $\H$ both have full column rank, then the rank of the $(m+n)k \times (m+n)k$ FIM $\fim$ is at most $(m+n)k-k^2$.
\end{proposition}
\begin{proof}
This proposition is equivalent to the claim that the linear system $\ufim\z=0$ has at least $k^2$ linearly independent nonzero solutions. Let
\[
\z = [~\z_1^T~\z_2^T~]^T =
[~\vec{{\bf Z}_1}^T~\vec{{\bf Z}_2}^T~]^T,
\]
where ${\bf Z}_1$ is an $m\times k$ matrix, and ${\bf Z}_2$ is an $n\times k$ matrix. Then
\begin{align*}
 &\ufim\z \\
=& \begin{bmatrix}
\H^T\H\otimes\eye{m}\vec{{\bf Z}_1} + (\eye{k}\otimes\W)\C_k(\eye{k}\otimes\H)^T\vec{{\bf Z}_2} \\
(\eye{k}\otimes\H)\C_k(\eye{k}\otimes\W)^T\vec{{\bf Z}_1} + \W^T\W\otimes\eye{n}\vec{{\bf Z}_1}
\end{bmatrix} \\
=& \begin{bmatrix}
\vec{{\bf Z}_1\H^T\H + \W{\bf Z}_2^T\H} \\
\vec{\H{\bf Z}_1^T\W + {\bf Z}_2\W^T\W}
\end{bmatrix}.
\end{align*}
Now let ${\bf Z}_1 = \w_\kappa^{}\e_l^T$, ${\bf Z}_2 = -\h_l^{}\e_\kappa^T$, where $\kappa,l = 1,2,...,k$, then $\z\neq 0$ and
\[
\ufim\z = \begin{bmatrix}
\vec{\left(\w_\kappa^{}\h_l^T-\w_\kappa^{}\h_l^T\right)\H} \\
\vec{\left(\h_l^{}\w_\kappa^T-\h_l^{}\w_\kappa^T\right)\W}
\end{bmatrix} = 0.
\]
Thus, we have found $k^2$ solutions in that form, and indeed they are linearly independent, if $\W$ and $\H$ both have full column rank.
\end{proof}

\noindent {\bf Computing the \crb~Bound:} For classical CRB, once we have derived the FIM, the CRB is simply given by the inverse of FIM. As we have argued in Proposition~\ref{ppst:fimrank-mf}, the FIM for matrix factorization models is always rank deficient. Nevertheless, pseudo-inverse of the FIM can be used to compute a lowerbound, albeit not necessarily attainable in theory. In terms of identifiability, it is well-known that additional constraints are needed to insure uniqueness of the solution; however, as we have argued, simple constraints like non-negativity can provide identifiability under mild conditions, and in fact does not affect the CRB since it can be represented as inequality constraints. Therefore, we discuss how to efficiently compute the psuedo-inverse of the FIM without modifying it to accommodate any equality constraints.

Without exploiting any structure of the matrix, the usual way to calculate the pseudo-inverse is by using the singular value decomposition (SVD), which entails complexity approximately cubic in the matrix dimension. The FIM for matrix factorization is $(m+n)k \times (m+n)k$, and the complexity of brute-force pseudo-inversion via the SVD is problematic. At first glance, the FIM in (\ref{eq:fim4mf}) exhibits good structure: $\fim$ given in (\ref{eq:fim4mf}) is the summation of a non-singular matrix and a low rank term. However, we cannot directly apply the matrix inversion lemma (Woodbury's identity) or the blockwise inversion formula (cf. \cite{petersen2006matrix}), simply because they are both singular, as we have argued in Propositions \ref{ppst:fimrank-mf}.

There exist similar results for the pseudo-inverse, but the formulas are very complicated. In fact, $\fim$ also has Kronecker structure, and it is appealing to try using the following Property of the Kronecker product \cite{petersen2006matrix}
\begin{equation*}
    (\A \otimes \B)^{\dag} = \A^{\dag} \otimes \B^{\dag}
\end{equation*}
to greatly reduce the computation complexity. However, the formulas are so complicated that such structure would be destroyed. Therefore, in this subsection we seek specialized methods to compute the pseudo-inverse of $\fim$.

The basic idea of our method is based on the fact that we have not only identified the singularity but also bases for the null space of $\fim$. In this case, the basis of the null space helps us to calculate the pseudo-inverse by using the techniques for calculating the inverse, as described in the following lemma proven in the supplementary material of~\cite{huang2014spm}.
\begin{lemma}\label{lmm:pinv}
Let matrix $\M$ be symmetric and singular, and the matrix $\L$ satisfying $\range{\L}=\Null{\M}$, then
\begin{equation}\label{eq:pinv}
    \M^{\dag} = (\M + \L\L^T)^{-1}
                        - (\L^{\dag})^T \L^{\dag}.
\end{equation}
\end{lemma}
In addition, we will also use the matrix inversion lemma
\begin{equation*}
(\A+\B\C\D)^{-1} = \A^{-1} - \A^{-1}\B(\C^{-1}+\D\A^{-1}\B)^{-1}\D\A^{-1},
\end{equation*}
to ease the computation of matrix inverse.

Now we are ready to derive a computationally efficient way of calculating the pseudo-inverse of $\ufim$. Recall that we have fully identified the null space of $\ufim$ in Proposition~\ref{ppst:fimrank-mf}, and a basis of its null space is  of the form
\begin{align*}
\z =& \begin{bmatrix}
~~~\vec{\w_\kappa^{}\e_l^T} \\ -\vec{\h_l^{}\e_\kappa^T}
\end{bmatrix} = \begin{bmatrix}
\eye{k}\otimes\W & 0 \\ 0 & \eye{k}\otimes\H
\end{bmatrix}\begin{bmatrix}
~~\e_\kappa\otimes\e_l \\ -\e_l\otimes\e_\kappa
\end{bmatrix} \\
=&\begin{bmatrix}
\eye{k}\otimes\W & 0 \\ 0 & \eye{k}\otimes\H
\end{bmatrix}\begin{bmatrix}
\eye{k^2} & 0 \\ 0 & \C_k
\end{bmatrix}\begin{bmatrix}
~~\e_\kappa\otimes\e_l \\ -\e_\kappa\otimes\e_l
\end{bmatrix},
\end{align*}
Thus, we can stack all the vectors of this form and define the matrix $\L$ as
\begin{align*}
\L &= \begin{bmatrix}
\eye{k}\otimes\W & 0 \\ 0 & \eye{k}\otimes\H
\end{bmatrix}\begin{bmatrix}
\eye{k^2} & 0 \\ 0 & \C_k
\end{bmatrix}\begin{bmatrix}
~~\eye{k^2} \\ -\eye{k^2}
\end{bmatrix} 								\\
&= \begin{bmatrix}
\eye{k}\otimes\W & 0 \\ 0 & \eye{k}\otimes\H
\end{bmatrix}\begin{bmatrix}
~~\eye{k^2} \\ -\C_k
\end{bmatrix} = \begin{bmatrix}
\eye{k}\otimes\W \\ -\left(\eye{k}\otimes\H\right)\C_k
\end{bmatrix},
\end{align*}
whose columns are linearly independent, thus we can write its pseudo-inverse explicitly as
\begin{align*}
 &\L^\dagger \\
=& \inv{\L^T\L}\L^T \\
=& \inv{\eye{k}\otimes\W^T\W + \C_k\left(\eye{k}\otimes\H^T\H\right)\C_k}
\begin{bmatrix}
\eye{k}\otimes\W \\ -\left(\eye{k}\otimes\H\right)\C_k
\end{bmatrix}								\\
=& \inv{\eye{k}\otimes\W^T\W+\H^T\H\otimes\eye{k}}
\begin{bmatrix}
\eye{k}\otimes\W \\ -\left(\eye{k}\otimes\H\right)\C_k
\end{bmatrix}^T,
\end{align*}
and we have that $\range{\L}=\Null{\ufim}$. Then we can ``complete'' the range of $\ufim$ via defining
\begin{align*}
& \fullfim = \ufim + \L\L^T 				\\
=& \begin{bmatrix}
\H^T\H \otimes \eye{m} & 0 \\
0 & \W^T\W \otimes \eye{n}
\end{bmatrix}								\\
+&\begin{bmatrix}
\eye{k}\otimes\W & 0 \\
0 & \eye{k}\otimes\H
\end{bmatrix}\begin{bmatrix}
0 & \C_k \\
\C_k & 0
\end{bmatrix}\begin{bmatrix}
\eye{k}\otimes\W & 0 \\
0 & \eye{k}\otimes\H
\end{bmatrix}^T 							\\
+& \begin{bmatrix}
\eye{k}\otimes\W & 0 \\
0 & \eye{k}\otimes\H
\end{bmatrix}\begin{bmatrix}
\eye{k^2} & \!\!\!-\C_k \\
-\C_k\!\!\! & \eye{k^2}
\end{bmatrix}\begin{bmatrix}
\eye{k}\otimes\W & 0 \\
0 & \eye{k}\otimes\H
\end{bmatrix}^T								\\
=& \begin{bmatrix}
\fullfim_{\W} & 0 \\ 0 & \fullfim_{\H}
\end{bmatrix},
\end{align*}
where
\begin{align*}
\fullfim_{\W} &= \H^T\H\otimes\eye{m} + (\eye{k}\otimes\W)(\eye{k}\otimes\W)^T, \\
\fullfim_{\H} &= \W^T\W\otimes\eye{n} + (\eye{k}\otimes\H)(\eye{k}\otimes\H)^T,
\end{align*}
which is, surprisingly, block diagonal, and each diagonal block can be inverted easily using matrix inversion lemma as in \eqref{eq:invOmegaa}-\eqref{eq:invOmegac}, and finally $\ufim^\dagger$ given in \eqref{eq:pinvPsi}, thanks to Lemma~\ref{lmm:pinv}.

In a lot of cases we are only interested in evaluating how small $\|\W-\hat{\W}\|_F^2$ and $\|\H-\hat{\H}\|_F^2$ can be, on average. We can then define $\beta_{\W}$ and $\beta_{\H}$ as in \eqref{eq:betaWH}, and they are the \crb~bound for the matrix factorization model, i.e.,
\begin{align*}
\E{\|\W-\hat{\W}\|_F^2} &\geq \sigma^2\beta_{\W}, \\
\E{\|\H-\hat{\H}\|_F^2} &\geq \sigma^2\beta_{\H},
\end{align*}
for any unbiased estimators $\hat{\W}$ and $\hat{\H}$.

\begin{figure*}[t]
\begin{subequations}
\begin{align}
\fullfim^{-1} = & \begin{bmatrix}
\fullfim_{\W}^{-1} & 0 \\ 0 & \fullfim_{\H}^{-1}
\end{bmatrix}, 							\label{eq:invOmegaa}	\\
\fullfim_{\W}^{-1} = & \inv{\H^T\H}\otimes\eye{m} -
\left(\inv{\H^T\H}\otimes\W\right)
\inv{\eye{k^2}+\inv{\H^T\H}\otimes\W^T\W}
\left(\inv{\H^T\H}\otimes\W^T\right),	\label{eq:invOmegab}	\\
\fullfim_{\H}^{-1} = & \inv{\W^T\W}\otimes\eye{n} -
\left(\inv{\W^T\W}\otimes\H\right)
\inv{\eye{k^2}+\inv{\W^T\W}\otimes\H^T\H}
\left(\inv{\W^T\W}\otimes\H^T\right),	\label{eq:invOmegac}	\\
\ufim^\dag = & \begin{bmatrix}
\fullfim_{\W}^{-1} & 0 \\ 0 & \fullfim_{\H}^{-1}
\end{bmatrix} - \begin{bmatrix}
\eye{k}\otimes\W \\ -\left(\eye{k}\otimes\H\right)\C_k
\end{bmatrix}
\invtwo{\eye{k}\otimes\W^T\W+\H^T\H\otimes\eye{k}}
\begin{bmatrix}
\eye{k}\otimes\W \\ -\left(\eye{k}\otimes\H\right)\C_k
\end{bmatrix}^T,						\label{eq:pinvPsi}
\end{align}
\end{subequations}
\begin{subequations}\label{eq:betaWH}
\begin{align}
\beta_{\W} =&~ \trace{\inv{\H^T\H}\otimes\eye{m}} -
  \trace{\inv{\eye{k^2}+\inv{\H^T\H}\otimes\W^T\W}\left(\invtwo{\H^T\H}\otimes\W^T\W\right)} \nonumber\\
  &~~- \trace{\invtwo{\eye{k}\otimes\W^T\W+\H^T\H\otimes\eye{k}}\left(\eye{k}\otimes\W^T\W\right)}, \\
\beta_{\H} =&~ \trace{\inv{\W^T\W}\otimes\eye{n}} -
  \trace{\inv{\eye{k^2}+\inv{\W^T\W}\otimes\H^T\H}\left(\invtwo{\W^T\W}\otimes\H^T\H\right)} \nonumber\\
  &~~- \trace{\invtwo{\eye{k}\otimes\W^T\W+\H^T\H\otimes\eye{k}}\left(\H^T\H\otimes\eye{k}\right)},
\end{align}
\end{subequations}
\hrulefill
\end{figure*}

\subsubsection{\crb~Bound for CP Factorization Models}

The CRB for the CP factorization model exhibits a lot of similarities to the one for the matrix factorization model, but also a fair number of differences, thus it deserves to be derived from scratch and study its properties separately.

Consider the $N$-way tensor generated as
\[
\Yt = \ktensor{\H_d}_{d=1}^N + \Nt,
\]
where $\H_d$ is $n_d \times k$ and the elements of $\Nt$ are drawn from an i.i.d. Gaussian distribution with zero mean and variance $\sigma^2$. Then the log-likelihood of $\Yt$ parameterized by $\H_1, ..., \H_N$ is
\begin{align*}
 & \log p(\Yt;\H_1,...,\H_N) \\
=& -\frac{1}{\sigma^2}\left\|\vec{\Yt}-(\H_N\odot...\odot\H_1)\one\right\|^2 \\
=& -\frac{1}{\sigma^2}\|\vec{\Yt}-\ph(\th)\|^2,
\end{align*}
where the unknown parameters we want to estimate, $\H_1, ..., \H_N$, are stacked into one single long vector of size $(n_1+...+n_N)k$
\[
\th = \left[~\vec{\H_1}^T~...~\vec{\H_N}^T~\right]^T,
\]
and the nonlinear function $\ph(\th)$ is given in \eqref{eq:phHd}.

\noindent {\bf The Fisher Information Matrix:} Invoking Proposition~\ref{ppst:crb4gauss}, the FIM for the CP model under Gaussian noise has the form
\[
\fim = \frac{1}{\sigma^2}\ufim = \frac{1}{\sigma^2}\jacob_{\th}\ph(\th)^T\jacob_{\th}\ph(\th).
\]
The Jacobian matrix of $\ph(\th)$ can be partitioned into $N$ blocks
\[
\jacob_{\th}\ph(\th) = \begin{bmatrix}
\jacob_{\H_1}\ph(\th) & \cdots & \jacob_{\H_N}\ph(\th)
\end{bmatrix},
\]
and from \eqref{eq:phHd}, $\jacob_{\H_d}\ph(\th)$ can be written as in \eqref{eq:D_Hd}.
Using the properties of the commutation matrices, we have that
\begin{equation*}
\jacob_{\H_d}\ph(\th)^T\jacob_{\H_d}\ph(\th) =
\left(\hada{d}\H_j^T\H_j^{}\right)\otimes\I_{n_d}
\end{equation*}
and as for the off-diagonal blocks $\jacob_{\H_c}\ph(\th)^T\jacob_{\H_d}\ph(\th)$, consider multiplying this matrix with $\vec{\tilde{\H}_d}$ where $\tilde{\H}_d$ is a $n_d \times k$ matrix, we have \eqref{eq:Psy_cd*vec}, which holds for all possible $\tilde{\H}\in\Real^{n_d\times k}$, implying
\begin{align*}
&\jacob_{\H_c}\ph(\th)^T\jacob_{\H_d}\ph(\th) \\
=&(\I_k\otimes\H_c)\diag{\hada{d,c}\H_j^T\H_j^{}}\C_k(\I_k\otimes\H_d)^T.
\end{align*}

\begin{figure*}[t]
\begin{align}
\ph(\th) &= (\H_N\odot...\odot\H_1)\one \nonumber\\
&= \C_{n_{d-1}...n_1,n_N...n_d}(\H_{d-1}\odot\cdots\odot\H_1\odot\H_N\odot\cdots\odot\H_d)\one \nonumber\\
&= \C_{n_{d-1}...n_1,n_N...n_d}\vec{\H_d(\H_{d-1}\odot\cdots\odot\H_1\odot\H_N\odot\cdots\odot\H_{d+1})^T} \nonumber\\
&= \C_{n_{d-1}...n_1,n_N...n_d}
\left(\left(\H_{d-1}\odot\cdots\odot\H_1\odot\H_N\odot\cdots\odot\H_{d+1}\right)\otimes\I_{n_d}\right)\vec{\H_d}.
\label{eq:phHd}\\
\jacob_{\H_d}\ph(\th) &= \C_{n_{d-1}...n_1,n_N...n_d}
\left(\left(\H_{d-1}\odot\cdots\odot\H_1\odot\H_N\odot\cdots\odot\H_{d+1}\right)\otimes\I_{n_d}\right)\label{eq:D_Hd}
\end{align}
\begin{align}
\jacob_{\H_c}\ph(\th)^T\jacob_{\H_d}\ph(\th)\vec{\tilde{\H}_d}
=&\jacob_{\H_c}\ph(\th)^T\C_{n_{d-1}...n_1,n_N...n_d}
\vec{\tilde{\H}_d\left(\H_{d-1}\odot\cdots\odot\H_1\odot\H_N\odot\cdots\odot\H_{d+1}\right)^T}	\nonumber\\
=&\jacob_{\H_c}\ph(\th)^T\C_{n_{d-1}...n_1,n_N...n_d}
(\H_{d-1}\odot\cdots\odot\H_1\odot\H_N\odot\cdots\odot\H_{d+1}\odot\tilde{\H}_d)\one	\nonumber\\
=&\jacob_{\H_c}\ph(\th)^T(\H_N\odot\cdots\odot\H_{d+1}\odot\tilde{\H}_d\odot\H_{d-1}\odot\cdots\odot\H_1)\one	 \nonumber\\
=&\left(\left(\H_{c-1}\odot\cdots\odot\H_1\odot\H_N\odot\cdots\odot\H_{c+1}\right)^T\otimes\I_{n_c}\right)
\C_{n_N...n_c,n_{c-1}...n_1}	\nonumber\\
 &~~~~~~(\H_{d-1}\odot\cdots\odot\H_1\odot\H_N\odot\cdots\odot\H_{d+1}\odot\tilde{\H}_d)\one	\nonumber\\
=&\left(\left(\left(\hada{d,c}\H_j^T\H_j^{}\right)*\H_d^T\tilde{\H}_d^{}\right)\odot\H_c\right)\one	\nonumber\\
=&\vec{\H_c\left(\left(\hada{d,c}\H_j^T\H_j^{}\right)*\H_d^T\tilde{\H}_d^{}\right)^T} \label{eq:vec}\\
=&(\I_k\otimes\H_c)\diag{\hada{d,c}\H_j^T\H_j^{}}\C_{k,k}\vec{\H_d^T\tilde{\H}_d^{}} 	\nonumber\\
=&(\I_k\otimes\H_c)\diag{\hada{d,c}\H_j^T\H_j^{}}\C_{k,k}(\I_k\otimes\H_d^T)\vec{\tilde{\H}_c}. \label{eq:Psy_cd*vec}
\end{align}
\hrulefill
\end{figure*}

We can then express the $(n_1+...+n_N)k \times (n_1+...+n_N)k$ Fisher information matrix $\fim=\sigma^{-2}\ufim$ compactly as the following block form
\begin{equation}\label{eq:fim4cp1}
\ufim = \begin{bmatrix}
\ufim_{1,1} & \cdots & \ufim_{1,N} \\
\vdots		& \ddots & \vdots 	   \\
\ufim_{N,1} & \cdots & \ufim_{N,N}
\end{bmatrix},
\end{equation}
where
\begin{equation}\label{eq:fim4cp2}
\ufim_{d,c} = \left\{
\begin{aligned}
&\bm{\Gamma}_d \otimes \eye{n_d},~~~ d=c,\\
&\left(\eye{k}\otimes\H_d\right) \C_k\diag{\vec{\bm{\Gamma}_{d,c}}} \left(\eye{k}\otimes\H_c\right)^T,\\
&~~~~~~~~~~~~~,~~~ d\neq c,
\end{aligned}
\right.
\end{equation}
and
\begin{equation}\label{eq:fim4cp3}
\bm{\Gamma}_{d,c} = \hada{d,c}\H_j^T\H_j^{}.
\end{equation}
Alternatively, we can also write it in the form of ``block diagonal plus low rank'' as follows
\begin{align*}
\ufim & = \bm{\Delta} + \bm{\Upsilon}\mathbf{K}\bm{\Upsilon}^T,
\end{align*}
where $\bm{\Delta}$ and $\bm{\Upsilon}$ are both block diagonal
\begin{align*}
\bm{\Delta} &= \begin{bmatrix}
\bm{\Gamma}_1\otimes\eye{n_1} & 0 & \cdots & 0 \\
0 & \bm{\Gamma}_2\otimes\eye{n_2} &  & \vdots \\
\vdots &  & \ddots & 0 \\
0 & \cdots & 0 & \bm{\Gamma}_N\otimes\eye{n_N}
\end{bmatrix}, \\
\bm{\Upsilon} &= \begin{bmatrix}
~\eye{k}\otimes\H_1~ & 0 & \cdots & 0 \\
0 & ~\eye{k}\otimes\H_2~ &  & \vdots \\
\vdots &  & \ddots & 0 \\
0 & \cdots & 0 & ~~\eye{k}\otimes\H_N~~
\end{bmatrix},
\end{align*}
and the $Nk^2 \times Nk^2$ matrix $\mathbf{K}$ is partitioned into $N\times N$ blocks each of size $k^2 \times k^2$, and the $d,c$-th block equals to
\[
\mathbf{K}_{d,c} = \left\{
\begin{aligned}
& 0, ~~~ d=c, \\
& \C_k\diag{\vec{\bm{\Gamma}_{d,c}}}, ~~~ d\neq c.
\end{aligned}
\right.
\]

Formulae for the Jacobian matrix and FIM have appeared in~\cite{liu2001cramer,vorobyov2005robust,tomasi2006thesis,phan2013low,tichavsky2013cramer}, but the derivation is not as clear and straight-forward as the one given here.

\paragraph*{Remark}
The FIM for the CP model indeed looks very similar to the FIM for the matrix factorization model. In fact, for $N=2$, if we overload the definition of $\bm{\Gamma}_{d,c}$ to be $\bm{\Gamma}_{1,2}=\bm{\Gamma}_{2,1}=\one_{k\times k}$, then the FIM for CP defined in (\ref{eq:fim4cp1})-(\ref{eq:fim4cp3}) for $N=2$ becomes exactly equal to the FIM for matrix factorization defined in (\ref{eq:fim4mf}). Similar to the matrix factorization case, the FIM is also rank deficient. However, the null space result for the matrix factorization case does not simply generalize to the CP case, as will be seen in the following proposition.

\begin{proposition}\label{ppst:fimrank-cp}
If $\H_1, ..., \H_N$ all have full column rank, then the rank of the $(n_1+...+n_N)k \times (n_1+...+n_N)k$ FIM $\fim$ is at most $(n_1+...+n_N)k - (N-1)k$.
\end{proposition}
\begin{proof}
Again, it suffices to find $(N-1)k$ linearly independent solutions to the linear system $\ufim\z=0$. Consider a vector $\z$ of the following form
\[
\z = [~\underbrace{0~...~0}_{(n_1+...+n_{c-1})k}~\vec{\H_c\D}^T~
		\underbrace{0~...~0}_{(n_{c+1}+...+n_N)k}~]^T,
\]
where $\D$ is an arbitrary diagonal matrix, then
\[
\ufim\z = \begin{bmatrix}
\ufim_{1,c}\vec{\H_c\D} \\ \ufim_{2,c}\vec{\H_c\D} \\ \vdots \\ \ufim_{N,c}\vec{\H_c\D}
\end{bmatrix},
\]
where
\[
\ufim_{c,c}\vec{\H_c\D} = \vec{\H_c\D\bm{\Gamma}_c},
\]
and
\begin{align*}
 & \ufim_{d,c}\vec{\H_c\D} \\
=& \left(\eye{k}\otimes\H_d\right) \C_k\diag{\vec{\bm{\Gamma}_{d,c}}} \left(\eye{k}\otimes\H_c\right)^T\vec{\H_c\D} \\
=& \left(\eye{k}\otimes\H_d\right) \C_k\diag{\vec{\bm{\Gamma}_{d,c}}} \vec{\H_c^T\H_c^{}\D}\\
=& \left(\eye{k}\otimes\H_d\right) \C_k \vec{\bm{\Gamma}_d\D} \\
=& \left(\eye{k}\otimes\H_d\right) \vec{\D\bm{\Gamma}_d} \\
=& \vec{\H_d\D\bm{\Gamma}_d},
\end{align*}
for $d \neq c$. Notice that at the third step,
$$\diag{\vec{\bm{\Gamma}_{d,c}}} \vec{\H_c^T\H_c^{}\D} = \vec{\bm{\Gamma}_d\D}$$
holds if $\D$ is a diagonal matrix, but not in general. As we can see, for $\z$ of this form, the result of $\ufim\z$ is independent of $c$.

Next, consider $\z$ to be the difference of two vectors of the aforementioned form, the non-zero block of one of them being the first block
\[
~~\z = \begin{bmatrix}
\vec{\H_1\D} \\
\zero_{(n_2+...+n_{c-1})k,1} \\
-\vec{\H_c\D} \\
\zero_{(n_{c+1}+...+n_N)k,1}
\end{bmatrix},
\]
then
\[
\ufim\z = \begin{bmatrix}
\ufim_{1,1}\vec{\H_1\D} \\ \ufim_{2,1}\vec{\H_1\D} \\ \vdots \\ \ufim_{N,1}\vec{\H_1\D}
\end{bmatrix} - \begin{bmatrix}
\ufim_{1,c}\vec{\H_c\D} \\ \ufim_{2,c}\vec{\H_c\D} \\ \vdots \\ \ufim_{N,c}\vec{\H_c\D}
\end{bmatrix} = 0.
\]
Fixing $c$, a $k \times k$ diagonal matrix $\D$ has $k$ degrees of freedom, and $c$ can be chosen from $2,...,N$, so we have found in total $(N-1)k$ linearly independent solutions to the linear system $\ufim\z=0$.

Notice that we can also make the $d$-th block and $c$-th block of $\z$ being $\vec{\H_d\D}$ and $-\vec{\H_c\D}$, but it is equal to the first and $d$-th, minus first and $c$-th, so this does not introduce additional dimension to the null space of $\ufim$.
\end{proof}

\paragraph*{Remark}
In terms of the dimension of the null space of $\fim$, the FIM for the CP model behaves differently from the one for the matrix case, since the rank deficiency is $(N-1)k$, whereas the rank deficiency of the FIM for the MF model is $k^2 \neq (2-1)k$. In fact, let us pick a basis for the span of the diagonal matrices to be $\{\e_1^{}\e_1^T, ..., \e_k^{}\e_k^T\}$, then it becomes apparent that the null space is closely related to the inherent scaling ambiguity in the CP model (meaning if the $l$-th column of $\H_1$ and the $l$-th column of $\H_c$ move in the opposite direction, it does not affect the CRB), whereas in the two factor case, the $\D$ matrix is not restricted to be diagonal, which seems related to the fact that for matrix factorization $\Y=\W\H^T$, we can put a more general non-singular matrix in between $\Y=\W\A\A^{-1}\H^T$. Nevertheless, these are trivial ambiguities within these factor analysis models, and it is not obvious how, for example, simple non-negativity constraints can lead to essentially unique solutions as shown in~\cite{huang2014tsp}.

\begin{figure*}[t]
\begin{align}
\L & = \begin{bmatrix}
~~\vec{\H_1\e_1^{}\e_1^T} & \cdots & ~~\vec{\H_1\e_N^{}\e_N^T} & \cdots &
~~\vec{\H_1\e_1^{}\e_1^T} & \cdots & ~~\vec{\H_1\e_N^{}\e_N^T} \\
-\vec{\H_2\e_1^{}\e_1^T} & \cdots & -\vec{\H_2\e_N^{}\e_N^T} & 0 & \cdots &  & 0 \\
\vdots &  & \vdots & \ddots & & \ddots & \vdots \\
0 & \cdots &  & 0 & -\vec{\H_N\e_1^{}\e_1^T} & \cdots & -\vec{\H_N\e_N^{}\e_N^T}
\end{bmatrix} \nonumber\\
& = \begin{bmatrix}
\eye{k}\otimes\H_1 & 0 & \cdots & 0 \\
0 & \eye{k}\otimes\H_2 &  & \vdots \\
\vdots &  & \ddots & 0 \\
0 & \cdots & 0 & \eye{k}\otimes\H_N
\end{bmatrix}\begin{bmatrix}
~~\eye{k}\odot\eye{k} & \cdots & ~~\eye{k}\odot\eye{k} \\
 -\eye{k}\odot\eye{k} &  & 0 \\
  & \ddots & \vdots \\
 0 & \cdots & -\eye{k}\odot\eye{k}
\end{bmatrix} = \bm{\Upsilon}\mathbf{E},\label{eq:cpL}\\
\mathbf{E}^T\bm{\Upsilon}^T\bm{\Upsilon}\mathbf{E} &= \begin{bmatrix}
\eye{k}*(\H_1^T\H_1^{}+\H_2^T\H_2^{}) & \eye{k}*\H_1^T\H_1^{} & \cdots
												& \eye{k}*\H_1^T\H_1^{}		\\
\eye{k}*\H_1^T\H_1^{} & \eye{k}*(\H_1^T\H_1^{}+\H_3^T\H_3^{}) &
												& \eye{k}*\H_1^T\H_1^{}		\\
\vdots &  & \ddots & \vdots		\\
\eye{k}*\H_1^T\H_1^{} & \eye{k}*\H_1^T\H_1^{} & \cdots &
					\eye{k}*(\H_1^T\H_1^{}+\H_N^T\H_N^{})
\end{bmatrix} \label{eq:EUUE}\\
& = \begin{bmatrix}
\eye{k}*\H_2^T\H_2^{} & 0 & \cdots & 0 \\
 & \eye{k}*\H_3^T\H_3^{} & & \vdots \\
\vdots & & \ddots \\
0 & \cdots & & \eye{k}*\H_N^T\H_N^{}
\end{bmatrix} + \begin{bmatrix}
\eye{k}*\H_1^T\H_1^{} & \eye{k}*\H_1^T\H_1^{} & \cdots & \eye{k}*\H_1^T\H_1^{} \\
\eye{k}*\H_1^T\H_1^{} & \eye{k}*\H_1^T\H_1^{} & 	\\
\vdots &  & \ddots & \vdots \\
\eye{k}*\H_1^T\H_1^{}  & \cdots & & \eye{k}*\H_1^T\H_1^{}
\end{bmatrix},\nonumber
\end{align}
\hrulefill
\end{figure*}

\noindent {\bf Computing the \crb~Bound:} To compute the pseudo-inverse of the FIM we derived in (\ref{eq:fim4cp1})-(\ref{eq:fim4cp3}) to obtain the CRB for the CP factorization model, we use the similar idea used in CRB for the MF case, which is by invoking Lemma~\ref{lmm:pinv}, and the fact that we have identified the null space of $\fim$ in Proposition~\ref{ppst:fimrank-cp}.

First, define the matrix $\L$ whose columns span the null space of $\ufim$ as in \eqref{eq:cpL}, where $\mathbf{E}$ is $Nk^2 \times (N-1)k$, partitioned into $N\times(N-1)$ blocks, with $d,c$-th block defined as
\[
\mathbf{E}_{d,c} = \left\{\begin{aligned}
~\eye{k}\odot\eye{k}, &~~ d = 1, \\
-\eye{k}\odot\eye{k}, &~~ d = c+1, \\
0,~~~~~~~ & ~~\text{otherwise.}
\end{aligned}
\right.
\]
Since $\L$ has full column rank, its pseudo-inverse is
\begin{align*}
\L^\dagger & = \inv{\L^T\L}\L^T \\
& = \inv{\mathbf{E}^T\bm{\Upsilon}^T\bm{\Upsilon}\mathbf{E}}\mathbf{E}^T\bm{\Upsilon}^T,
\end{align*}
where the matrix we want to invert can be written as in \eqref{eq:EUUE},
which is ``diagonal plus low rank'', thus can be inverted efficiently.

Next, we define $\fullfim$ by completing the range space of $\ufim$
\begin{align*}
\fullfim & = \ufim + \L\L^T \\
&= \bm{\Delta} + \bm{\Upsilon}\mathbf{K}\bm{\Upsilon}^T + \bm{\Upsilon}\mathbf{EE}^T\bm{\Upsilon}^T \\
&= \bm{\Delta} + \bm{\Upsilon}\left(\mathbf{K}+\mathbf{EE}^T\right)\bm{\Upsilon}^T.
\end{align*}
In the CP factorization case, the ``completed'' matrix $\fullfim$ does not have the nice block diagonal structure as we did in the matrix case. Such a hope is arguably impossible, as we notice that each $k^2 \times k^2$ block not on the diagonal is full rank, and different, whereas the rank of $\mathbf{EE}^T$ has rank $(N-1)k$, therefore we cannot construct a matrix $\mathbf{E}$ such that each of its off-diagonal block have rank $k^2$, unless $k<N$. Nevertheless, if $\mathbf{K}+\mathbf{EE}^T$ is invertible, applying matrix inversion lemma on $\fullfim$ leads to
\begin{align*}
\fullfim^{-1} \!=\! \bm{\Delta}^{-1} \!\!-\! \bm{\Delta}^{-1}\bm{\Upsilon}
				\inv{\inv{\mathbf{K}\!+\!\mathbf{EE}^T}\!\!\!+\!\bm{\Upsilon\Delta}^{-1}\bm{\Upsilon}}
				\bm{\Upsilon}^T\bm{\Delta}^{-1}.
\end{align*}
Notice that $\bm{\Delta}$ is block diagonal, and each of its diagonal blocks is a Kronecker product, thus computing $\bm{\Delta}^{-1}$ only requires inverting $N$ number of $k \times k$ matrices. The most expensive step is to compute $\inv{\mathbf{K}+\mathbf{EE}^T}$ and $\inv{\inv{\mathbf{K}+\mathbf{EE}^T}+\bm{\Upsilon\Delta}^{-1}\bm{\Upsilon}}$, both of size $Nk^2\times Nk^2$. However, it is still a huge improvement, considering the size of $\ufim$ is $(n_1+...+n_N)k \times (n_1+...+n_N)k$, if $Nk<n_1+...+n_N$. Otherwise, the CP factorization is not considered ``low-rank'', thus directly (pseudo-)invert the FIM is not a bad idea, nonetheless.

Finally, by subtracting $(\L^{\dag})^T \L^{\dag}$ from $\fullfim^{-1}$ we obtain
\begin{align*}
\ufim^\dagger & = \fullfim^{-1} - (\L^{\dag})^T \L^{\dag} \\
& = \bm{\Delta}^{-1} \!\!-\! \bm{\Delta}^{-1}\bm{\Upsilon}
				\inv{\inv{\mathbf{K}\!+\!\mathbf{EE}^T}\!\!\!+\!\bm{\Upsilon\Delta}^{-1}\bm{\Upsilon}}
				\bm{\Upsilon}^T\bm{\Delta}^{-1} 		\\
& ~~~~~~~~~-	
\bm{\Upsilon}\mathbf{E}\invtwo{\mathbf{E}^T\bm{\Upsilon}^T\bm{\Upsilon}\mathbf{E}}\mathbf{E}^T\bm{\Upsilon}^T,
\end{align*}
and the CRB is simply
\[
\fim^\dagger = \sigma^2\ufim^\dag.
\]

\newpage
~
\newpage
~
\newpage

\begin{IEEEbiography}[{\includegraphics[width=1in,height=1.25in,clip,keepaspectratio]{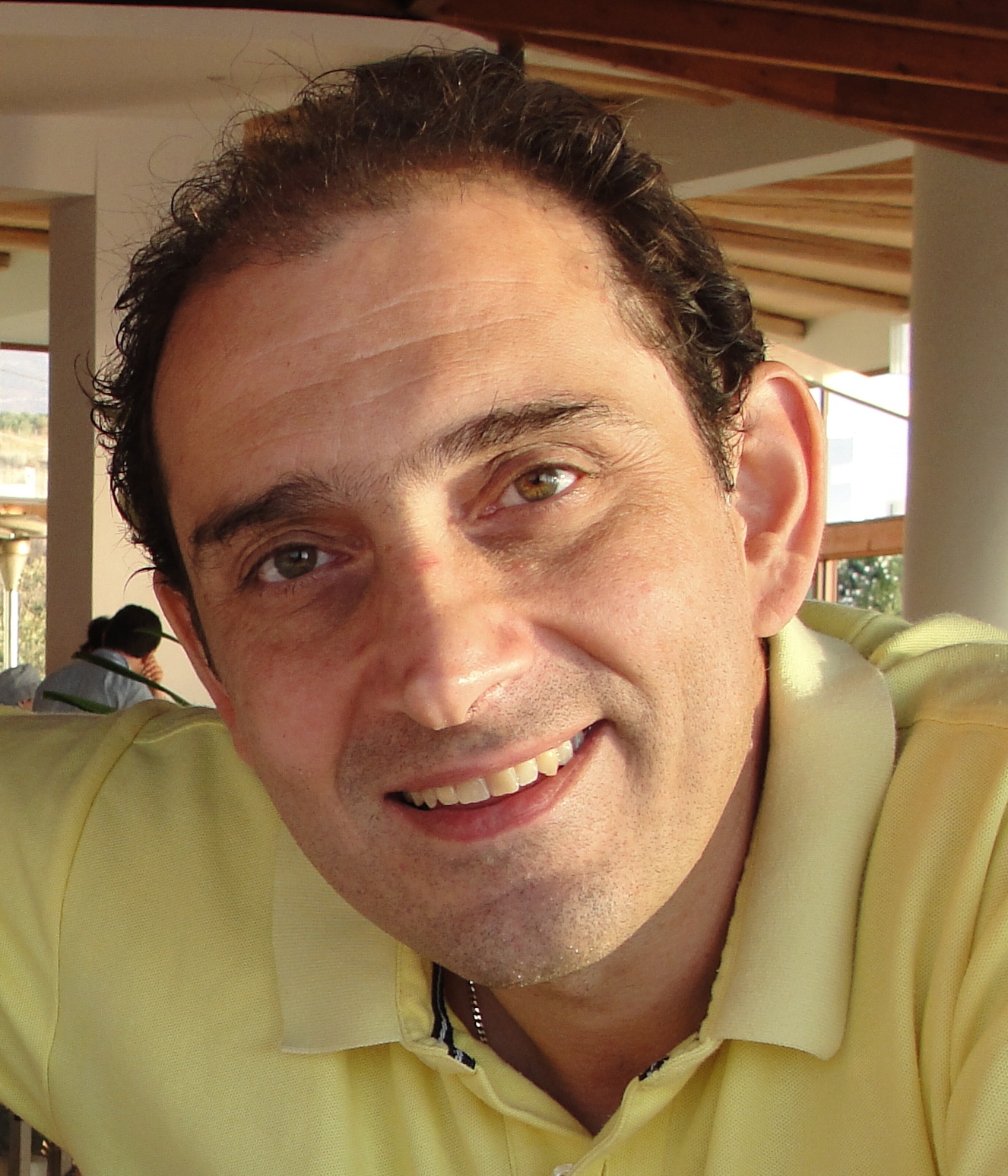}}]
{Nicholas D. Sidiropoulos} (F'09) received the Diploma in Electrical
Engineering from the Aristotelian University of Thessaloniki, Greece,
and M.S. and Ph.D. degrees in Electrical Engineering from the
University of Maryland at College Park, in 1988, 1990 and 1992,
respectively. He served as assistant professor at the University of
Virginia, associate professor at the University of Minnesota, and
professor at TU Crete, Greece. Since 2011, he has been at the
University of Minnesota, where he currently holds an ADC Chair in
digital technology. His research spans topics in signal processing
theory and algorithms, optimization, communications, and factor
analysis - with a long-term interest in tensor decomposition and its
applications. His current focus is primarily on signal and tensor
analytics for learning from big data. He received the NSF/CAREER award
in 1998, and the IEEE Signal Processing (SP) Society Best Paper Award
in 2001, 2007, and 2011. He served as IEEE SP Society Distinguished
Lecturer (2008-2009), and as Chair of the IEEE Signal Processing for
Communications and Networking Technical Committee (2007-2008). He
received the 2010 IEEE SP Society Meritorious Service Award, and the
2013 Distinguished Alumni Award from the Dept. of ECE, University of
Maryland. He is a Fellow of IEEE (2009) and a Fellow of EURASIP
(2014).
\end{IEEEbiography}

\begin{IEEEbiography}[{\includegraphics[width=1in,height=1.25in,clip,keepaspectratio]{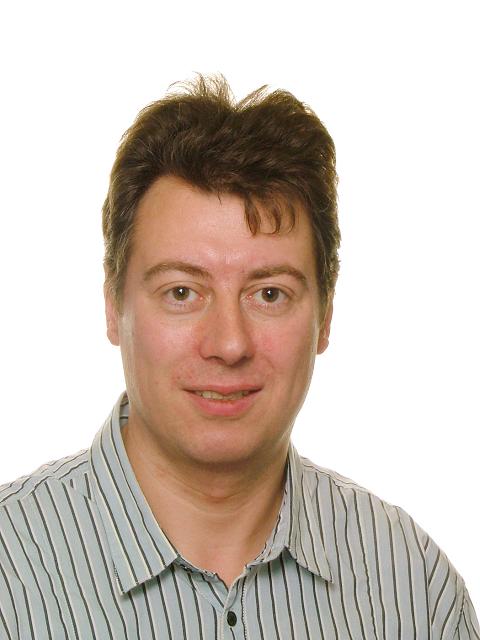}}]
{Lieven De Lathauwer} (F'15) received the Master's degree in electromechanical engineering
and the Ph.D. degree in applied sciences from KU Leuven, Belgium, in 1992 and 1997, respectively.
From 2000 to 2007 he was Research Associate of the French Centre National de la Recherche Scientifique, research group CNRS-ETIS. He is currently Professor at KU Leuven, affiliated with both the Group Science, Engineering and Technology of Kulak, and with the group STADIUS of the Electrical Engineering Department (ESAT). He is Associate Editor of the SIAM Journal on Matrix Analysis and Applications and he has served as Associate Editor for the IEEE Transactions on Signal Processing. His research concerns the development of tensor tools for mathematical engineering. It centers on the following axes: (i) algebraic foundations, (ii) numerical algorithms, (iii) generic methods for signal processing, data analysis and system modelling, and (iv) specific applications. Keywords are linear and multilinear algebra, numerical algorithms, statistical signal and array processing, higher-order statistics, independent component analysis and blind source separation, harmonic retrieval, factor analysis, blind identification and equalization, big data, data fusion. Algorithms have been made available as Tensorlab (www.tensorlab.net) (with N. Vervliet, O. Debals, L. Sorber and M. Van Barel).
\end{IEEEbiography}

\begin{IEEEbiography}[{\includegraphics[width=1in,height=1.25in,clip,keepaspectratio]{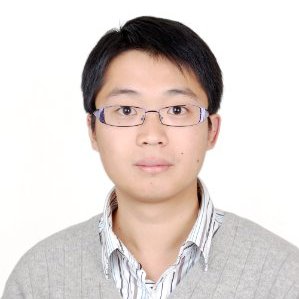}}]
{Xiao Fu} (S'12-M'15) received his B.Eng and M.Eng degrees in communication and information engineering from the University of Electronic Science and Technology of China, Chengdu, China, in 2005 and 2010, respectively. In 2014, he received his Ph.D. degree in electronic engineering from the Chinese University of Hong Kong (CUHK), Hong Kong. From 2005 to 2006, he was an assistant engineer at China Telecom Co. Ltd., Shenzhen, China. He is currently a Postdoctoral Associate at the Department of Electrical and Computer Engineering, University of Minnesota, Minneapolis, United States. His research interests include signal processing and machine learning, with a recent emphasis on factor analysis and its applications. Dr. Fu was an awardee of the Overseas Research Attachment Programme (ORAP) 2013 of the Engineering Faculty, CUHK, which sponsored his visit to the Department of Electrical and Computer Engineering, University of Minnesota, from September 2013 to February 2014. He received a Best Student Paper Award at ICASSP 2014, and co-authored a Best Student Paper Award at IEEE CAMSAP 2015.
\end{IEEEbiography}

\begin{IEEEbiography}[{\includegraphics[width=1in,height=1.25in,clip,keepaspectratio]{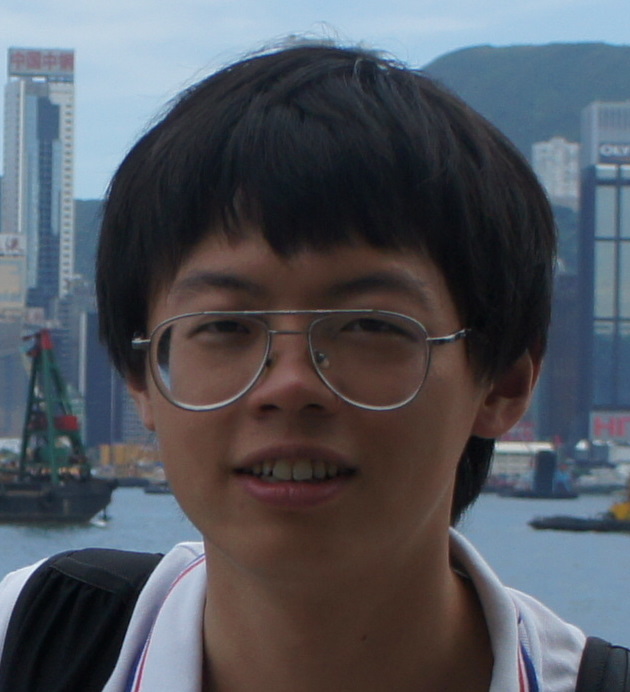}}]
{Kejun Huang} (S'13) received the B.Eng. in Communication Engineering from Nanjing University of Information Science and Technology, Nanjing, China in 2010. Since September 2010, he has been working towards his Ph.D. degree in the Department of Electrical and Computer Engineering, University of Minnesota. His research interests include signal processing, machine learning, and data analytics. His current research focuses on identifiability, algorithms, and performance analysis for factor analysis of big matrix and tensor data.
\end{IEEEbiography}

\begin{IEEEbiography}[{\includegraphics[width=1in,height=1.25in,clip,keepaspectratio]{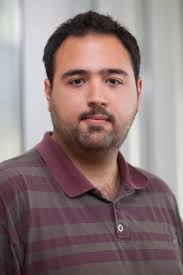}}]
{Evangelos Papalexakis} is an Assistant Professor in the Computer Science Department at the University of California, Riverside. He earned his Diploma and M.Sc. in Electronic and Computer Engineering at the Technical University of Crete, Chania, Greece, and his Ph.D. in Computer Science at Carnegie Mellon University, in 2010, 2011, and 2016 respectively. He has considerable experience in tensor decompositions, data mining, computing, and signal processing. He is currently interested in discovering how knowledge and information is expressed and stored in the brain, through analyzing brain scan data coupled with external information. He is also interested in anomaly detection on very large graphs, especially when temporal or multi-view information is present. He has spent two summers as an intern at  Microsoft Research Silicon Valley, working at the Search Labs and the Interaction \& Intent group.
\end{IEEEbiography}

\begin{IEEEbiography}[{\includegraphics[width=1in,height=1.25in,clip,keepaspectratio]{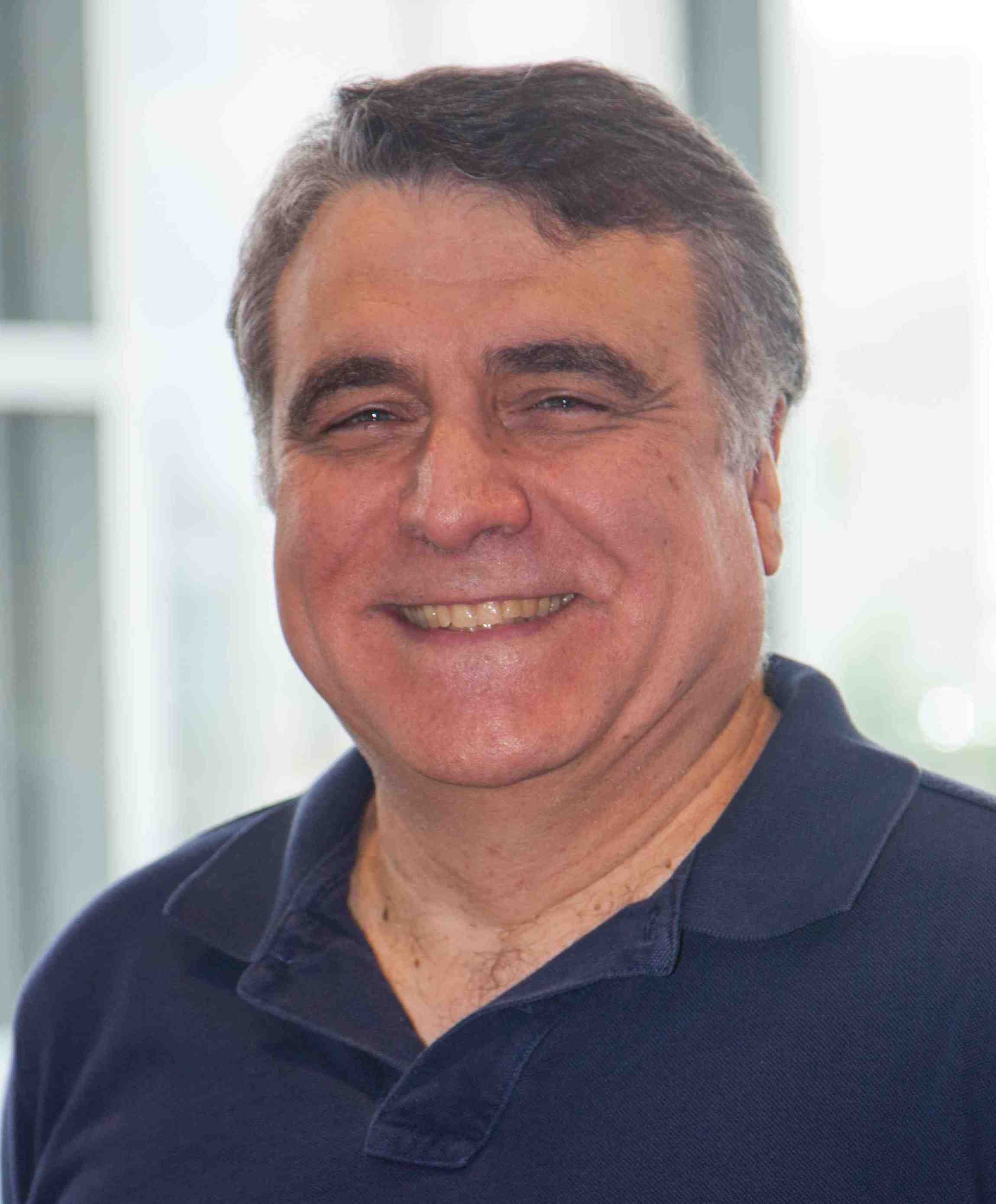}}]
{Christos Faloutos} is a Professor at Carnegie Mellon University. He has received the Presidential Young Investigator Award by the National Science Foundation (1989), the Research Contributions Award in ICDM 2006,
the SIGKDD Innovations Award (2010), 21 ``best paper'' awards (including 3 ``test of time'' awards),
and four teaching awards. Five of his advisees have won KDD or SCS dissertation awards. He is an ACM Fellow,
he has served as a member of the executive committee of SIGKDD; he has published over 300 refereed articles, 17 book chapters and two monographs.  He holds nine patents and he has given over 40 tutorials and over 20 invited distinguished lectures. He has a long-term interest in tensor decompositions and their practical applications in data mining, having published many well-appreciated papers in the area. His broad research interests include large-scale data mining with emphasis on graphs and time sequences; anomaly detection, tensors, and fractals.
\end{IEEEbiography}

\end{document}